\PassOptionsToPackage{table}{xcolor}
\PassOptionsToPackage{section}{placeins}
\documentclass{fairmeta}

\usepackage{amsmath,amssymb}
\usepackage{mathtools}
\usepackage{algorithm}
\usepackage{algpseudocode}
\usepackage{enumitem}

\DeclareMathOperator*{\argmin}{arg\,min\,}
\DeclareMathOperator{\tr}{tr}
\DeclareMathOperator{\softmax}{softmax}

\usepackage{tikz}
\usetikzlibrary{positioning, arrows.meta, fit, backgrounds, calc,
                decorations.pathreplacing}

\definecolor{brickred}{HTML}{C1121F}
\definecolor{molten}{HTML}{780000}
\definecolor{papaya}{HTML}{FDF0D5}
\definecolor{deepspace}{HTML}{003049}
\definecolor{steel}{HTML}{669BBC}
\definecolor{brickfill}{HTML}{F3E4E4}
\definecolor{oursred}{RGB}{252,228,228}
\colorlet{slate}{deepspace}
\colorlet{brick}{brickred}
\colorlet{grayline}{deepspace}
\colorlet{grayfill}{steel!18!white}
\colorlet{oursrow}{brickred!12}

\colorlet{metabg}{brickred!8}

\AtBeginDocument{\let\cite\citep}

\hypersetup{
  colorlinks,
  linkcolor=deepspace,
  citecolor=deepspace,
  urlcolor=deepspace,
  pdfborder={0 0 0},
}

\usepackage{titletoc}

\newtheorem{theorem}{Theorem}

\newtheorem{proposition}{Proposition}
\newtheorem{corollary}{Corollary}
\newtheorem{remark}{Remark}

\newtheorem{propositionS}{Proposition}

\newtheorem{lemmaS}{Lemma}

\newtheorem*{thmIrestated}{Theorem 1 (restated)}
\newtheorem*{thmIIrestated}{Theorem 2 (restated)}
\newtheorem*{propIrestated}{Proposition 1 (restated)}
\newtheorem*{corIrestated}{Corollary 1 (restated)}

\newcommand{\method}{NOVA-KV}
\newcommand{\ci}[2]{#1\,{\scriptsize$\pm$#2}}
\newcommand{\cib}[2]{\textbf{#1}\,{\scriptsize$\pm$#2}}

\renewcommand{\titlelogo}{\includegraphics[width=1.6cm]{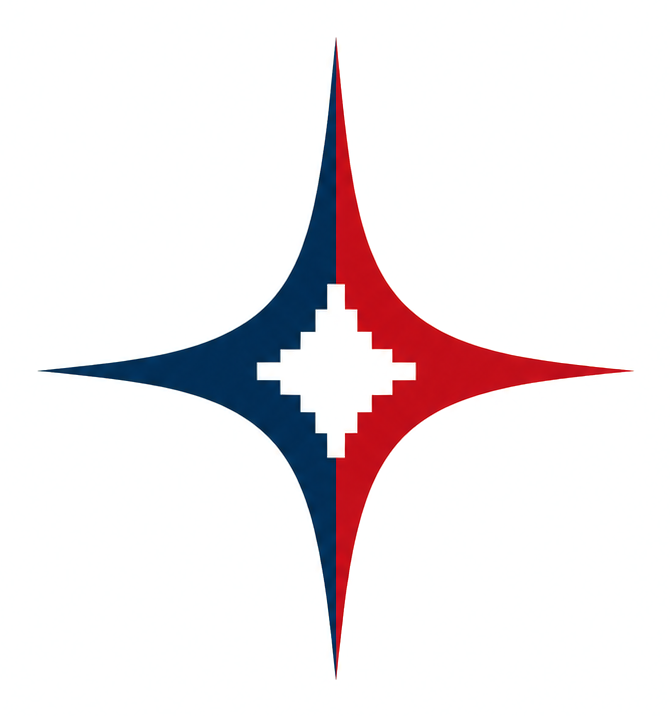}}

\title{Spend Bits Where Queries Look: KV Cache Vector Quantization with
Attention-Preserving Transforms}

\author[*]{Samuel Fern\'andez-Mendui\~na}
\author[*]{Amir Ziashahabi}
\author{Eduardo Pavez}
\author{Antonio Ortega}
\author{Salman Avestimehr}

\affiliation{Department of Electrical and Computer Engineering,
University of Southern California}

\contribution[*]{Equal contribution}

\abstract{Long-context LLM decoding reads the key-value (KV) cache at every step. Loading it takes longer than computing attention over it, so throughput is bandwidth-bound. Hence, reducing the cache size can raise both decoding speed and serving capacity. 
The challenge is to reduce cache size while preserving the attention products, keeping reconstruction cheap, and using a fixed per-token bit count. 
At two bits per element, the most competitive methods rely on orthogonal transforms. However, existing techniques are either data-oblivious or use the query statistics without deriving the transform from a distortion criterion. 
Moreover, they rely on transforms built on top of random or Hadamard rotations, which equalize variances across entries rather than compacting energy, and fixed-width scalar quantizers, which are suboptimal at low rates. In this paper, we formulate KV cache quantization as a transform coding problem in which distortion is the error in the attention products. 
We derive closed-form optimal transforms for keys and values from calibration statistics, under a high-resolution model. We show that the optimal key transform is not orthogonal and satisfies a generalized Parseval relation: the attention-aware distortion becomes mean-squared error (MSE) in the transform domain. Thus, we can use  MSE-optimal vector quantizers applied directly to the transformed key coefficients. 
To meet the fixed-width layout requirement, we show that grouping coefficients into equal-volume partitions makes equal-size codebooks attain the variable-rate optimum under the same high-resolution model. At two bits per element, our method, termed \method{}, recovers most of the long-context retrieval accuracy lost by scalar quantization methods at comparable throughput. The margin is widest on hybrid-attention mixture-of-experts models, an increasingly common design: on GPT-OSS-20B, prior two-bit transforms collapse at every context length, while \method{} remains effective.}

\website{https://amir-zsh.github.io/nova-kv}
\code{https://github.com/Amir-zsh/nova-kv}
\correspondence{\email{\{samuelf9, ziashaha\}@usc.edu}}

\begin{document}

\maketitle

\section{Introduction}

The key-value (KV) cache of a large language model (LLM) stores, for each past token and attention head, a key vector and a value vector, to
avoid recomputation during decoding
\cite{pope2023efficiently}. The size of the cache is a central bottleneck in inference \cite{shazeer2019fast}: it grows with context length and batch size, yet every decoding step reads it from memory in full \cite{sadhukhan2025magicdec}
(Fig.~\ref{fig:memory}). Since kernels compute attention products faster than the KV cache can be loaded from memory \cite{dao2022flashattention}, decoding throughput is bound by memory bandwidth rather than by compute \cite{kwon2023efficient}. Thus, \emph{KV cache compression} can reduce both per-step memory traffic and footprint, raising decoding speed and serving capacity.

\begin{figure}
\centering\includegraphics[width=0.665\linewidth]{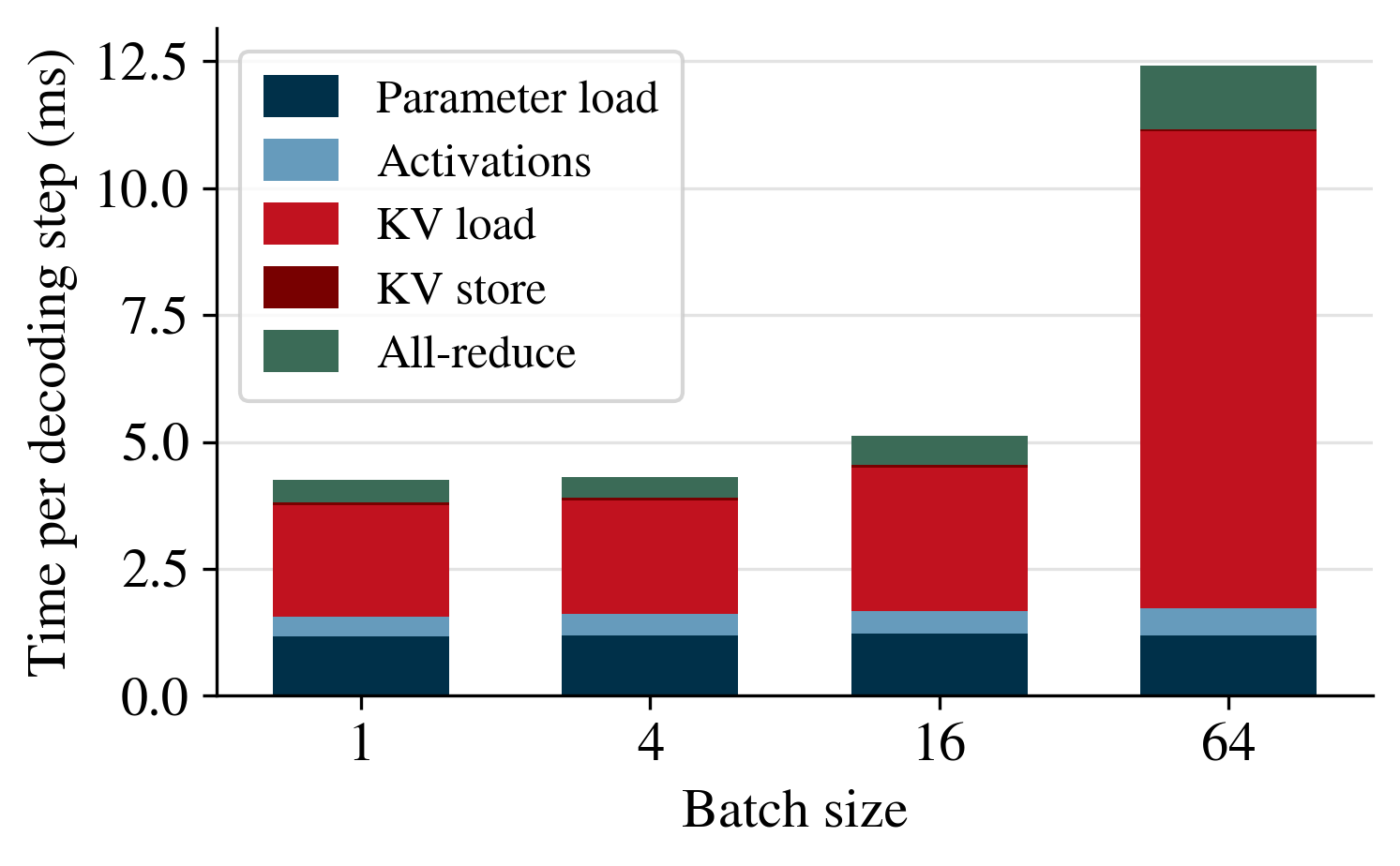}
    \caption{Measured per-step decode time for Qwen3-8B  ($8\times$H100, BF16, $S{=}16384$). At large batch sizes, KV cache loading dominates over parameter loading and compute.}    \label{fig:memory}
\end{figure}

Three requirements constrain the design of KV cache compressors. First, task accuracy must be preserved. Since capturing the effect of quantization on task accuracy is not straightforward,  we can model the problem using classical rate-distortion (RD) theory \cite{berger1971rate}, with the mean squared error (MSE) of the quantized values as the distortion metric. 
However, the cache contents are used to compute attention products, so we should minimize errors in the products rather than MSE in the factors  
\cite{zandieh2025turboquant, zhou2026oscar}.
Second, complexity on the read path \cite{williams2009roofline} has to be kept low, because the cache is read at every decoding step, so any per-element cost is incurred across the whole cache at every step. 
Third, a fixed-width layout must be used, so that tokens occupy a fixed number of bits and can be efficiently retrieved by serving engines \cite{kwon2023efficient}.

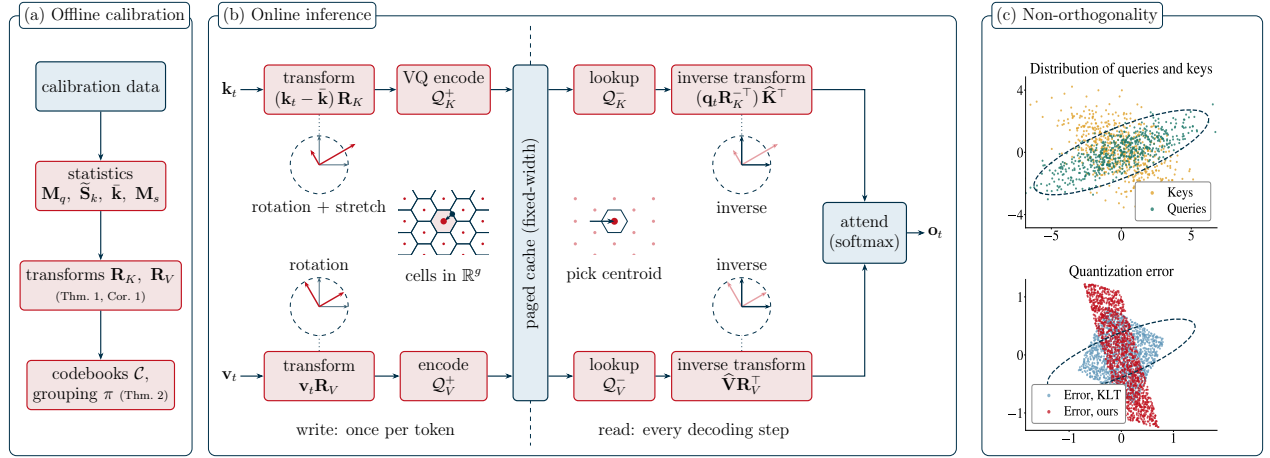
\begin{figure}[!htbp]
\centering
\resizebox{\textwidth}{!}{%
\begin{tikzpicture}[
  font=\large,
  std/.style={draw=grayline, fill=grayfill, rounded corners=3pt,
              minimum height=3.4em, inner sep=4pt, align=center, thick},
  new/.style={draw=brick, fill=brickfill, rounded corners=3pt,
              minimum height=3.4em, inner sep=4pt, align=center, thick},
  wide/.style={minimum width=7.6em},
  colT/.style={minimum width=76pt}, colQ/.style={minimum width=60pt},
  colL/.style={minimum width=56pt}, colI/.style={minimum width=88pt},
  plain/.style={align=center, inner sep=2pt},
  arr/.style={-{Stealth[length=5pt]}, thick, draw=slate},
  boxlbl/.style={font=\large, text=black, fill=white, draw=grayline,
                 rounded corners=3pt, inner sep=3pt},
  costlbl/.style={font=\large, text=black, align=center},
  sep/.style={dashed, semithick, draw=grayline},
  cell/.style={draw=grayline, thin}, pt/.style={fill=slate, draw=none},
  ctr/.style={fill=brick, draw=none},
  refc/.style={draw=grayline, dashed, semithick},
  gax/.style={-{Stealth[length=3pt]}, draw=grayline, semithick, opacity=0.7},
  tarr/.style={-{Stealth[length=4pt]}, draw=brick, line width=1pt},
  rarr/.style={-{Stealth[length=4pt]}, draw=slate, line width=1pt},
  inslbl/.style={font=\footnotesize, text=black, align=center},  
]
\def\gp{16pt}\def\gpa{34pt}
\def\ytop{30pt}\def\ybot{-260pt}

\coordinate (atop) at (0,\ytop); \coordinate (abot) at (0,\ybot);
\node[std, wide] (data) at (0,-11.8pt) {calibration data};
\node[new, wide, below=\gpa of data]  (stats)
  {statistics\\ $\mathbf{M}_q,\ \widetilde{\mathbf{S}}_k,\ \bar{\mathbf{k}},\ \mathbf{M}_s$};
\node[new, wide, below=\gpa of stats] (tfm)
  {transforms $\mathbf{R}_K,\ \mathbf{R}_V$\\ {\footnotesize(Thm.~1, Cor.~1)}};
\node[new, wide, below=\gpa of tfm]   (books)
  {codebooks $\mathcal{C}$,\\ grouping $\pi$ {\footnotesize(Thm.~2)}};
\draw[arr] (data) -- (stats); \draw[arr] (stats) -- (tfm); \draw[arr] (tfm) -- (books);
\begin{scope}[on background layer]
  \node[draw=grayline, rounded corners=5pt, inner sep=7pt, semithick,
        fit=(atop)(abot)(data)(stats)(tfm)(books)] (abox) {};
\end{scope}
\node[boxlbl, anchor=west] at ($(abox.north west)+(7pt,0)$) {(a) Offline calibration};

\begin{scope}[shift={(80pt,0)}]
\coordinate (btop) at (250pt,\ytop); \coordinate (bbot) at (250pt,\ybot);
\def\yK{-14pt}\def\yV{-212pt}\def\yC{-113pt}

\node[plain] (kin) at (9pt,\yK) {$\mathbf{k}_t$};
\node[plain] (vin) at (9pt,\yV) {$\mathbf{v}_t$};
\node[new, colT] (kt) at (71pt,\yK) {transform\\ $(\mathbf{k}_t-\bar{\mathbf{k}})\,\mathbf{R}_K$};
\node[new, colT] (vt) at (71pt,\yV) {transform\\ $\mathbf{v}_t\mathbf{R}_V$};
\node[new, colQ] (kq) at (157pt,\yK) {VQ encode\\ $\mathcal{Q}_K^{+}$};
\node[new, colQ] (vq) at (157pt,\yV) {encode\\ $\mathcal{Q}_V^{+}$};

\draw[sep] (217pt,\ytop) -- (217pt,\ybot);
\node[std, rotate=90, minimum width=232pt, minimum height=2.4em]
  (cache) at (217pt,\yC) {paged cache (fixed-width)};
\node[new, colL] (klu) at (275pt,\yK) {lookup\\ $\mathcal{Q}_K^{-}$};
\node[new, colL] (vlu) at (275pt,\yV) {lookup\\ $\mathcal{Q}_V^{-}$};
\node[new, colI] (kinv) at (363pt,\yK)
  {inverse transform\\ $(\mathbf{q}_t\mathbf{R}_K^{-\top})\,\widehat{\mathbf{K}}^{\top}$};
\node[new, colI] (vinv) at (363pt,\yV)
  {inverse transform\\ $\widehat{\mathbf{V}}\mathbf{R}_V^{\top}$};
\node[std, minimum height=4.2em] (att) at (448pt,\yC) {attend\\ (softmax)};
\node[plain] (outp) at (497pt,\yC) {$\mathbf{o}_t$};

\begin{scope}[shift={(71pt,-66pt)}]
  \draw[refc] (0,0) circle (20pt);
  \draw[gax] (0,0) -- (20pt,0);  \draw[gax] (0,0) -- (0,20pt);
  \draw[tarr] (0,0) -- (30:29pt); \draw[tarr] (0,0) -- (120:12pt);
\end{scope}
\node[costlbl] at (71pt,-95pt) {rotation $+$ stretch};
\draw[grayline, dotted, semithick] (kt.south) -- (71pt,-46pt);

\begin{scope}[shift={(71pt,-164pt)}]
  \draw[refc] (0,0) circle (20pt);
  \draw[gax] (0,0) -- (20pt,0);  \draw[gax] (0,0) -- (0,20pt);
  \draw[tarr] (0,0) -- (30:20pt); \draw[tarr] (0,0) -- (120:20pt);
\end{scope}
\node[costlbl] at (71pt,-135pt) {rotation};
\draw[grayline, dotted, semithick] (vt.north) -- (71pt,-184pt);

\begin{scope}[shift={(363pt,-66pt)}]
  \draw[refc] (0,0) circle (20pt);
  \draw[tarr, opacity=0.4] (0,0) -- (30:29pt);
  \draw[tarr, opacity=0.4] (0,0) -- (120:12pt);
  \draw[rarr] (0,0) -- (0:20pt); \draw[rarr] (0,0) -- (90:20pt);
\end{scope}
\node[costlbl] at (363pt,-95pt) {inverse};
\draw[grayline, dotted, semithick] (kinv.south) -- (363pt,-46pt);

\begin{scope}[shift={(363pt,-164pt)}]
  \draw[refc] (0,0) circle (20pt);
  \draw[tarr, opacity=0.4] (0,0) -- (30:20pt);
  \draw[tarr, opacity=0.4] (0,0) -- (120:20pt);
  \draw[rarr] (0,0) -- (0:20pt); \draw[rarr] (0,0) -- (90:20pt);
\end{scope}
\node[costlbl] at (363pt,-135pt) {inverse};
\draw[grayline, dotted, semithick] (vinv.north) -- (363pt,-184pt);

\begin{scope}[yshift=-86pt]
\begin{scope}
  \clip (126pt,2pt) rectangle (188pt,-42pt);
  \begin{scope}[shift={(157pt,-19pt)}]
    \fill[brickfill] (0:9pt) -- (60:9pt) -- (120:9pt) -- (180:9pt)
                   -- (240:9pt) -- (300:9pt) -- cycle;
  \end{scope}
  \foreach \i in {-2,...,2}{\foreach \j in {-2,...,2}{
    \pgfmathsetmacro{\cx}{157 + \i*13.5}
    \pgfmathsetmacro{\cy}{-19 + \j*15.6 + mod(abs(\i),2)*7.8}
    \begin{scope}[shift={(\cx pt,\cy pt)}]
      \draw[cell] (0:9pt) -- (60:9pt) -- (120:9pt) -- (180:9pt)
                -- (240:9pt) -- (300:9pt) -- cycle;
      \fill[ctr] (0,0) circle (0.9pt);
    \end{scope}}}
  \foreach \p in {(150,-13),(154,-24),(163,-25),(147,-20),(159,-10),
                  (167,-16),(151,-28),(169,-24)}{
    \fill[pt, opacity=0.7] \p circle (1.1pt);}
  \fill[slate] (163pt,-14pt) circle (1.9pt);
  \draw[-{Stealth[length=3pt]}, slate, line width=0.9pt]
        (163pt,-14pt) -- (158.2pt,-18.2pt);
  \fill[brick] (157pt,-19pt) circle (2.2pt);
\end{scope}
\node[costlbl] at (157pt,-56pt) {cells in $\mathbb{R}^{g}$};

\begin{scope}
  \clip (244pt,2pt) rectangle (306pt,-42pt);
  \foreach \i in {-2,...,2}{\foreach \j in {-2,...,2}{
    \pgfmathsetmacro{\dx}{275 + \i*13.5}
    \pgfmathsetmacro{\dy}{-19 + \j*15.6 + mod(abs(\i),2)*7.8}
    \fill[ctr, opacity=0.45] (\dx pt,\dy pt) circle (1.1pt);}}
  \begin{scope}[shift={(275pt,-19pt)}]
    \draw[cell] (0:9pt) -- (60:9pt) -- (120:9pt) -- (180:9pt)
              -- (240:9pt) -- (300:9pt) -- cycle;
    \fill[brick] (0,0) circle (2.4pt);
  \end{scope}
  \draw[-{Stealth[length=5pt]}, slate, line width=0.9pt] (258pt,-19pt) -- (275pt,-19pt);
\end{scope}
\node[costlbl] at (275pt,-56pt) {pick centroid};
\end{scope}

\draw[arr] (kin) -- (kt);  \draw[arr] (vin) -- (vt);
\draw[arr] (kt) -- (kq);   \draw[arr] (vt) -- (vq);
\draw[arr] (kq.east) -- (kq.east -| cache.north);
\draw[arr] (vq.east) -- (vq.east -| cache.north);
\draw[arr] (cache.south |- klu.west) -- (klu.west);
\draw[arr] (cache.south |- vlu.west) -- (vlu.west);
\draw[arr] (klu) -- (kinv); \draw[arr] (vlu) -- (vinv);
\draw[arr] (kinv.east) -| (att.north);
\draw[arr] (vinv.east) -| (att.south);
\draw[arr] (att) -- (outp);

\node[costlbl] (wcost) at (110pt,-250pt) {write: once per token};
\node[costlbl] (rcost) at (330pt,-250pt) {read: every decoding step};

\begin{scope}[on background layer]
  \node[draw=grayline, rounded corners=5pt, inner sep=7pt, semithick,
        fit=(btop)(bbot)(kin)(vin)(kq)(vq)(cache)(klu)(vlu)(kinv)(vinv)
            (att)(outp)(wcost)(rcost)] (bbox) {};
\end{scope}
\node[boxlbl, anchor=west] at ($(bbox.north west)+(7pt,0)$) {(b) Online inference};
\end{scope}
\begin{scope}[shift={(706pt,0)}]
\coordinate (ctop) at (0,\ytop); \coordinate (cbot) at (0,\ybot);
\coordinate (cl) at (-90pt,\ytop);
\coordinate (cr) at ( 90pt,\ytop);
\node[inner sep=0pt] (sc1) at (-8pt,-60pt)
  {\includegraphics[width=150pt,height=132pt]{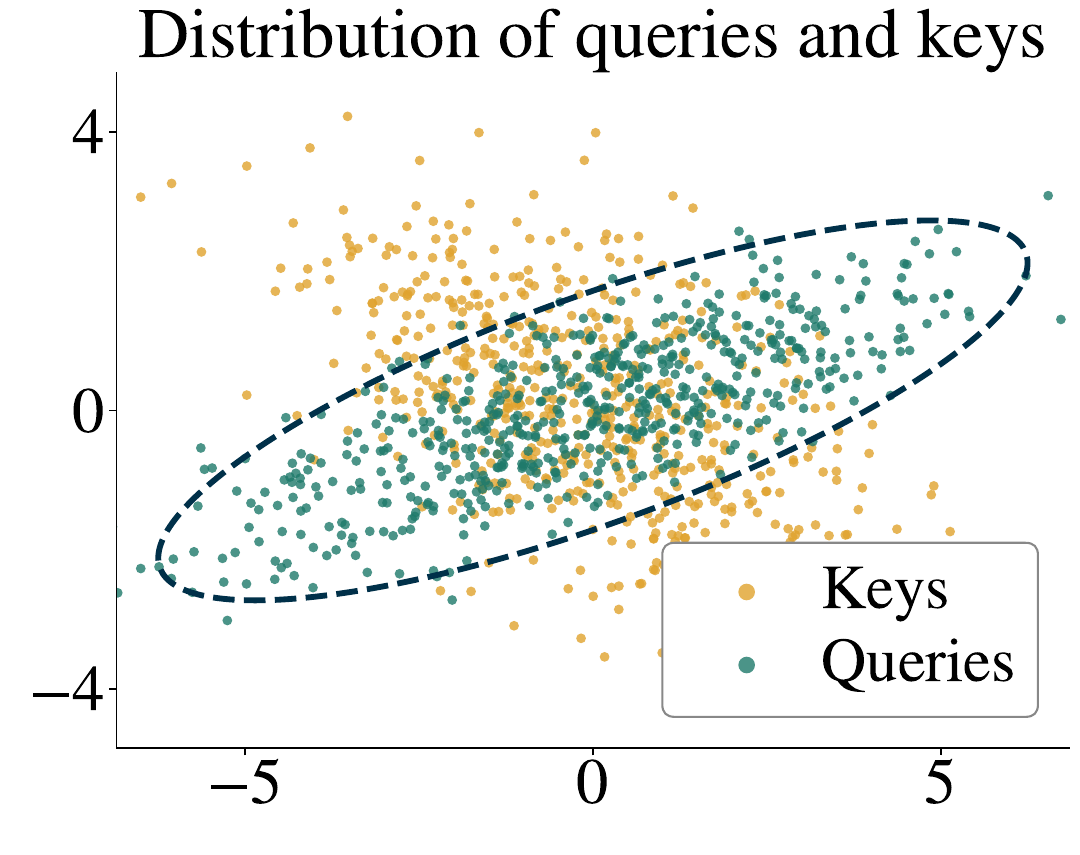}};
\node[inner sep=0pt] (sc2) at (-8pt,-200pt)
  {\includegraphics[width=150pt,height=132pt]{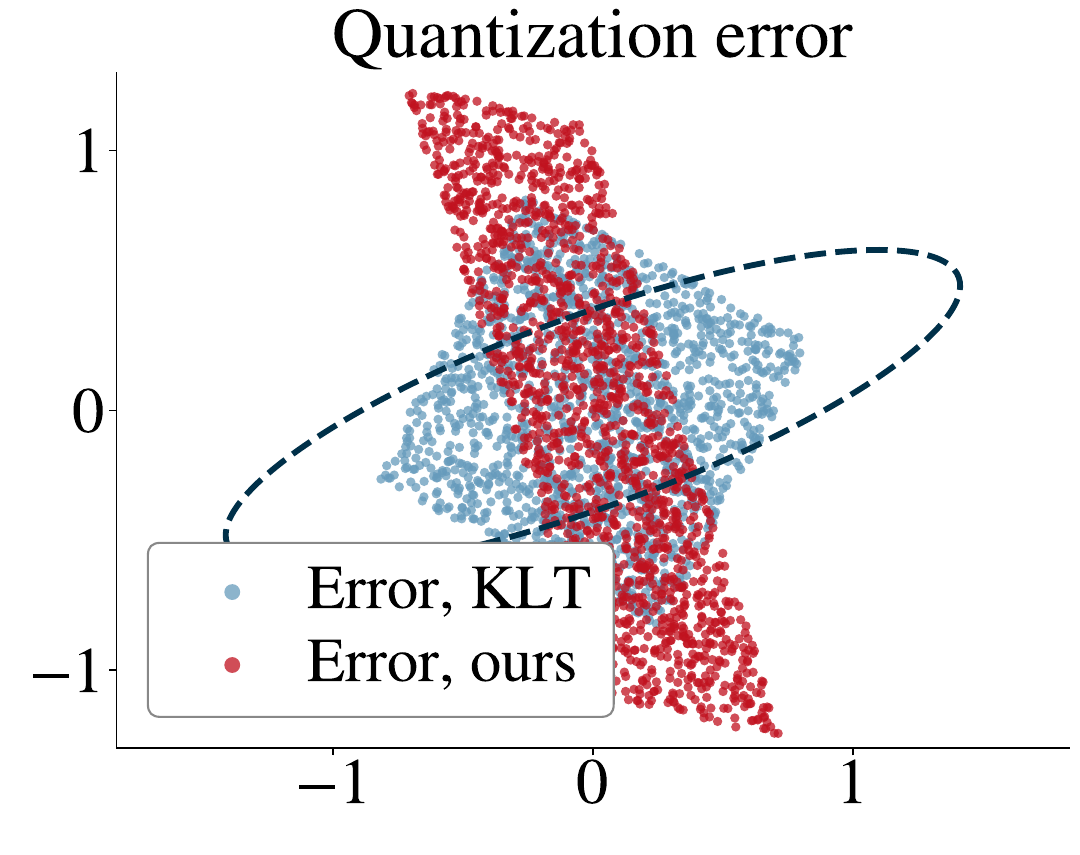}};
\node[draw=grayline, rounded corners=5pt, inner sep=7pt, semithick,
      fit=(ctop)(cbot)(cl)(cr)] (cbox) {};
\node[boxlbl, anchor=west] at ($(cbox.north west)+(7pt,0)$) {(c) Non-orthogonality};
\end{scope}
\end{tikzpicture}}%
\caption{\method{}. (a)~Calibration yields the transforms, the volume-equalized grouping~$\pi$, and the codebooks. (b)~Online, each new key and value is
transformed and encoded once on the write path; the read path, paid at every decoding step, is a table lookup. 
The dashed line separates the two paths. 
(c)~Top: keys, queries, and
query covariance (dashed). Bottom: quantization error at equal rate; our
transform steers the error away from the high-energy query directions, preserving
the attention logits.}
\label{fig:overview}
\end{figure}
Fixed-width quantization meets the last two requirements: decoding is a scaling or table lookup \cite{gray2002quantization}, and each token occupies a fixed number of bits. 
As the context grows, the KV cache dominates the memory footprint, so the reduction in inference time approaches the KV compression ratio itself. However, to preserve task accuracy, direct quantization of KV entries requires at least 4 bits per element \cite{liu2024kivi, sheng2023flexgen}; at 2 bits, accuracy degrades \cite{zhou2026oscar}. 
Transform coding \cite{goyal2001theoretical} has been proposed as an alternative to address this problem \cite{zandieh2025turboquant, zhou2026oscar}.

In classical transform coding, a \textit{data-dependent} linear operator (the transform) decorrelates the inputs;  scalar quantization is applied to the outputs, followed by entropy coding. 
Because these transforms achieve energy compaction, bits can be allocated proportionally to the coefficient variances, achieving lower overall distortion for a given average rate compared to directly encoding without a transform \cite{goyal2001theoretical}. 
Existing transform-based KV cache methods depart from this classical scheme. 
In some cases, they use transforms, such as random or Hadamard rotations \cite{ashkboos2024quarot, zandieh2025turboquant}, that are not optimized for the data. 
While data-driven transforms have been proposed, these are assembled heuristically from second-order statistics rather than as solutions to a specific rate-distortion optimization problem  \cite{zhou2026oscar}. 
Finally, existing designs use fixed-rate scalar quantization, with transforms that  \emph{flatten} the variances so that the same rate can be applied for every vector entry; flattening removes the variance spread (energy compaction) that bit allocation exploits, so the fixed-width layout can be achieved but at the expense of reduced coding efficiency. Scalar quantization with variable-length entropy coding is effective (even at low rates) \cite{goyal2001theoretical}, but requires sequential decoding,
breaking the fixed-width layout and the parallel read path requirements.

In this paper, we formulate KV cache compression as a transform coding
problem \cite{goyal2001theoretical} where distortion is set to be the attention product error. This error separates into two terms, one for
the keys and one for the values, recovering the two criteria that prior
work adopts separately \cite{zandieh2025turboquant, zhou2026oscar}. 
We derive, in closed form, the optimal transform for each distortion term in the high-resolution regime. 
For \textit{values}, the optimal transform is orthogonal and is obtained from the same covariance that prior work introduced as a heuristic \cite{zhou2026oscar}. 
For \textit{keys}, our focus in this work, we show that the optimal transform is not orthogonal.  
All prior designs used orthogonal transforms, which we show are suboptimal in practice when the goal is to minimize the key-query inner product.

The transform we derive for the keys satisfies a generalized Parseval relation: the MSE in the transform domain equals the key-query inner-product distortion in the original domain. This relation reduces the design of the quantizer to a standard MSE problem, so any MSE-optimal quantizer can act on the transform coefficients. 
Based on this reduction, we propose vector quantization of the key coefficients with an optimized grouping of entries into vectors. 
Transform coding would normally allocate more bits to high-variance entries, resulting in codebooks of different sizes, which breaks the fixed-width layout that serving engines require. 
Instead, our grouping approximately equalizes group volumes, allowing equal-size codebooks to approach the variable-rate distortion bound (cf.~Thm.~\ref{thm:grouping}). This combines the advantage of 
flattening (fixed bits per quantized input), while benefiting from energy compaction (fewer bits per input to achieve same distortion).

We integrate our method, termed \method{} (Non-Orthogonal Vector-quantized Attention for KV cache), into a production serving stack \cite{zheng2024sglang} with a fused decoding kernel \cite{tillet2019triton} (Fig.~\ref{fig:overview}). 
At 2 bits per element, \method{} recovers most of the accuracy lost by scalar methods in long-context retrieval. In terms of speed, decoding is on par with the scalar baseline \cite{zhou2026oscar} at long
contexts.

Our contributions are: 
1) data-driven transforms that minimize a high-resolution bound on the attention product distortion (Sec.~\ref{sec:transform}), 
2) an analysis of vector quantization for the keys in the KV cache, including the optimal fixed-rate grouping of transform coefficients under an independent-Gaussian high-resolution model (Sec.~\ref{sec:vq}), 
and 3) a serving-compatible implementation with
a fused decoding kernel, which we test against state-of-the-art 2-bit
methods (Sec.~\ref{sec:experiments}).

\section{Preliminaries}
\label{sec:preliminaries}

\textbf{Notation.} Uppercase and lowercase bold letters, such as $\mathbf A$ and $\mathbf a$, denote matrices and vectors, respectively. The $n$th entry of $\mathbf a$ is $a_n$, and the $(i, j)$th entry of $\mathbf A$ is $A_{ij}$. Regular letters denote scalar values. We use row-vector notation and assume the eigenvalues are sorted in decreasing order. We use $(\cdot)^+$ to denote the pseudoinverse.

\subsection{Attention and the KV Cache}
\label{sec:attn-cache}
Consider a single attention head \cite{vaswani2017attention} acting on a
sequence of hidden states $\{\mathbf{x}_t\}_{t=1}^{T}$, $\mathbf{x}_t \in
\mathbb{R}^{1\times d}$, with projection weights $\mathbf{W}_Q,
\mathbf{W}_K, \mathbf{W}_V \in \mathbb{R}^{d\times d}$, where $d$ denotes
the head dimension. The \textit{query}, \textit{key}, and \textit{value} vectors are
$\tilde{\mathbf{q}}_t = \mathbf{x}_t \mathbf{W}_Q$,
$\tilde{\mathbf{k}}_t = \mathbf{x}_t \mathbf{W}_K$,
$\mathbf{v}_t = \mathbf{x}_t \mathbf{W}_V$. We use rotary position
embeddings (RoPE) \cite{su2024roformer}: queries and keys are rotated by a
position-dependent block-diagonal matrix $\boldsymbol{\Phi}_t \in
\mathbb{R}^{d\times d}$ before the inner product,
$\mathbf{q}_t = \tilde{\mathbf{q}}_t \boldsymbol{\Phi}_t$ and
$\mathbf{k}_t = \tilde{\mathbf{k}}_t \boldsymbol{\Phi}_t$, while values
carry no positional encoding. Symbols without a tilde denote post-RoPE vectors. 
These are stacked into
$\mathbf{Q} = [\mathbf{q}_1; \dots; \mathbf{q}_T]$,
$\mathbf{K} = [\mathbf{k}_1; \dots; \mathbf{k}_T]$,
$\mathbf{V} = [\mathbf{v}_1; \dots; \mathbf{v}_T] \in \mathbb{R}^{T\times d}$.
The scores and outputs are:
\begin{equation}
\label{eq:scores}
  \mathbf{S} = \operatorname{softmax}_{\mathrm{row}}\!\left(
      {\mathbf{Q}\mathbf{K}^\top}/{\sqrt{d}} \right)
      \in \mathbb{R}^{T\times T}, \quad 
  \mathbf{O} = \mathbf{S}\mathbf{V}.
\end{equation}
During autoregressive decoding, generating token $t+1$ requires the query
$\mathbf{q}_{t+1}$ to attend over all previous positions. To avoid
recomputing past projections, the \emph{KV cache} stores in memory
$\mathbf{K}_{1:t} = [\mathbf{k}_1; \dots; \mathbf{k}_t]$ and
$\mathbf{V}_{1:t} = [\mathbf{v}_1; \dots; \mathbf{v}_t]$.

\subsection{Transform coding}
\label{ssec:tc}
Classical transform coding \cite{goyal2001theoretical} consists of three
stages: a transform that produces uncorrelated coefficients, scalar quantization of each coefficient, and entropy coding of the indices. 
Entropy decoding is sequential and cannot be easily parallelized, so it is not well-suited for our problem, where the cache is read at every decoding step, and parallel decoding is needed to maintain high throughput. 

\textbf{Transform.} By Parseval's relation, an orthogonal $\mathbf R$ preserves the MSE,
$\|\mathbf{x} - \widehat{\mathbf{x}}\|_2 = \|\mathbf{x}\mathbf{R} -
\mathcal{Q}(\mathbf{x}\mathbf{R})\|_2$, so the quantizer can be designed in
the transform domain; for Gaussian sources in the high-resolution regime
the Karhunen-Lo\`{e}ve transform (KLT, the eigenbasis of the covariance), is optimal
\cite{goyal2001theoretical,gersho1992vector}. 
A transform achieves \emph{energy compaction} when only a few coefficients in the transform have high variance. The KLT is optimal in the sense that its leading $p$ coefficients capture the most
variance. Thus, we can find the KLT, $\mathbf{R}^\star$, by optimizing the rank-$p$ reconstruction error: 
given centered $\mathbf{x}_j$, $j = 1,\dots,M$, let
\begin{equation}
\label{eq:klt_optimal}
  \mathbf{R}^\star_p
  = \operatorname*{arg\,min}_{\mathbf{R}\,\in\,\mathbb{R}^{d\times p}} \, 
  \sum_{j=1}^{M} \, 
  \bigl\|\,\mathbf{x}_j - \mathbf{x}_j\,\mathbf{R}\mathbf{R}^{+}\bigr\|_2^{2}.
\end{equation}
The KLT can be computed via eigendecomposition of the covariance matrix, and $\mathbf{R}^\star_p$ is the KLT truncated to the $p$
eigenvectors corresponding to the dominant eigenvalues of the data covariance. We will show in Sec.~\ref{sec:key_trans} that this is not optimal in our setting.

\textbf{Quantization.} A vector quantizer operates on vectors of dimension $g$ and with an average rate $b$ bits per entry, spending $gb$ bits per
vector. An encoder $\mathcal{Q}^{+}:\mathbb{R}^{1\times g}\to
\{1,\dots,2^{gb}\}$ assigns each subvector the index of its nearest
codeword, and a decoder $\mathcal{Q}^{-}$ returns that codeword from a
codebook $\mathcal{C}\subset\mathbb{R}^{1\times g}$ of $2^{gb}$ entries;
their composition is $\mathcal{Q}$ 
\cite{lloyd1982least,gersho1992vector}. \emph{Scalar quantization} (SQ), 
$g=1$, acts on each coordinate independently. For a given rate, 
\emph{vector quantization} (VQ), $g>1$, achieves distortion no larger than
SQ, since the products of scalar codebooks are valid vector
codebooks \cite{lookabaugh1989high}. 
Classical transform coding pairs SQ with variable-length entropy coding; however, as discussed earlier, entropy coding prevents parallel decoding, reducing the efficiency of the read path and impacting throughput. 

As a summary, given an invertible $\mathbf{R} \in
\mathbb{R}^{d\times d}$ and a vector $\mathbf x$, transform coding encodes it as
$\mathcal{Q}^{+}(\mathbf{x}\mathbf{R})$ and reconstructs it as
$\widehat{\mathbf{x}} =
\mathcal{Q}(\mathbf{x}\mathbf{R})\,\mathbf{R}^{-1}$.

\subsection{Related work}

\begin{table}[!htbp]
\small
\centering
\caption{KV cache compressors. \emph{Transform}: linear map applied before quantization. \emph{Criterion}: the objective the transform
is derived from.
\emph{Target}: the resource reduced; bandwidth-oriented methods (bw) keep the read
path to a scaling or a lookup, whereas storage-oriented methods (sto) pay a penalty in the read path. Had. stands for Hadamard, orth. stands for orthogonal.}
\label{tab:related}
\setlength{\tabcolsep}{2.5pt}
\renewcommand{\arraystretch}{0.9}
\begin{tabular}{lcccc}
\toprule
\textbf{Method} & \textbf{Transform} & \textbf{Criterion} & \textbf{Quant.}
  & \textbf{Target} \\
\midrule
KIVI        & None              & Heuristic                 & SQ & sto \\
KVQuant     & None             & Heuristic                 & SQ & sto \\
Kitty       & None              & Heuristic                 & SQ & sto \\
CommVQ      & None   & key MSE            & VQ & sto \\
\midrule
QuaRot      & rand. Had. orth.    & Heuristic                 & SQ & bw \\
TurboQuant  & rand. orth.    & inner-prod   & SQ & bw \\
RotateKV    & Had.+calib.\ orth.   & Heuristic                 & SQ & bw \\
OSCAR       & calib.+Had.\ orth.   & Heuristic                & SQ & bw \\
\rowcolor{oursrow}
\textbf{\method{}}   & \textbf{calib.\ non-orth.} & attn.\ products & VQ & bw \\
\bottomrule
\end{tabular}
\end{table}
\textbf{KV cache reduction} can be achieved by combining several complementary approaches:  
(i) token eviction, where cached tokens are dropped based on their
estimated importance to future attention \cite{zhang2023h2o, li2024snapkv,
xiao2024efficient, devoto2025expected}; (ii) low-rank designs, which shrink
the head dimension \cite{liu2024deepseek}; and (iii) quantization (our approach). 

\textbf{Scalar KV quantization} methods quantize entries independently and do not consider the attention products. 
Some designs perform per-channel quantization of keys and per-token quantization of values without accounting for input data statistics \cite{liu2024kivi}. Others fit the quantizers to the data, typically learning per-channel codebooks from the cache statistics \cite{hooper2024kvquant, cai2025nqkv}, or boosting sensitive channels at the cost of per-channel metadata and non-uniform layouts \cite{xia2025kitty}. 

\textbf{Transform coding.} QuaRot \cite{ashkboos2024quarot} uses random or Hadamard
rotations, and TurboQuant \cite{zandieh2025turboquant} debiases inner products. SpinQuant \cite{liu2025spinquant}
and RotateKV \cite{su2025rotatekv} learn the rotation on
calibration data. \citet{zhou2026oscar} reweight the key error by
query statistics. All are orthogonal and paired with SQ. We derive the transform from the RD problem using the attention-product distortion and show that the optimal key transform is
non-orthogonal.

\textbf{Vector quantization.} Unlike model weights \cite{chee2023quip}, KV caches are generated online and read at every decoding step, which changes the VQ requirements. \citet{zhang2024kv} explores joint codebooks across channels. 
CommVQ \cite{li2025commvq} targets memory \emph{storage} using VQ on the keys with large codebooks whose reconstruction requires a matrix product paid over the whole cache at every step. 
We target bandwidth-bound serving by applying VQ to small groups of transform-domain entries and restricting reconstruction to fixed-width codebook lookups that can be fused into the attention kernel. Thus, our design jointly provides an
attention-derived transform, fixed-width VQ, and a lookup-based reconstruction. 
In summary, NOVA-KV is the only method that targets attention products and optimizes both the transform and quantization for this purpose (see Table~\ref{tab:related}).

\section{Attention-preserving transforms}
\label{sec:transform}
We aim to design transforms ($\mathbf{R}_K, \mathbf{R}_V$)  and quantizers ($\mathcal{Q}_K(\cdot), \mathcal{Q}_V(\cdot)$) to compress keys and values
\begin{equation}
  \widehat{\mathbf{K}} = \mathcal{Q}_K(\mathbf{K}\mathbf{R}_K)\,\mathbf{R}_K^{-1},
  \qquad
  \widehat{\mathbf{V}} = \mathcal{Q}_V(\mathbf{V}\mathbf{R}_V)\,\mathbf{R}_V^{-1}, 
\end{equation}
while  
minimizing the attention output error 
$\|\mathbf{S}\mathbf{V} - \widehat{\mathbf{S}}\widehat{\mathbf{V}}\|_F^2$,
where $\widehat{\mathbf{S}}$ is obtained by replacing $\mathbf{K}$ by 
$\widehat{\mathbf{K}}$ in \eqref{eq:scores}. 
In the high-resolution regime, second order terms on the quantization errors vanish, and the target error can be upper-bounded by the sum of two terms (Appendix~\ref{sec:decoupling}):
\begin{equation}
\label{eq:separation}
    \|\mathbf{S}\mathbf{V} - \widehat{\mathbf{S}}\widehat{\mathbf{V}}\|_F \lesssim   
  \|\mathbf{S} - \widehat{\mathbf{S}}\|_F\|\mathbf{V}\|_2
  + \|\mathbf{S}(\mathbf{V} - \widehat{\mathbf{V}})\|_F.
\end{equation}
We use the score error $\|\mathbf{S} - \widehat{\mathbf{S}}\|_F$ as a criterion to derive $\mathbf{R}_K$  and the value error weighted by the scores $\|\mathbf{S}(\mathbf{V} - \widehat{\mathbf{V}})\|_F$ as a criterion to derive $\mathbf{R}_V$. 
$\mathcal{Q}_K$ consists of VQs 
with fixed codebook sizes applied to subvectors (Sec.~\ref{sec:vq}), while 
$\mathcal{Q}_V$ applies SQ to each transformed entry \cite{zhou2026oscar}.

\subsection{Key transform}
\label{sec:key_trans}
Since row-wise
softmax is $1/2$-Lipschitz \cite{gao2017properties}, the key-dependent term in \eqref{eq:separation}, 
$\|\mathbf{S}-\widehat{\mathbf{S}}\|_F$, 
can be replaced by the logit error: 
$\|\mathbf{S}-\widehat{\mathbf{S}}\|_F \le {1}/(2\sqrt{d})
\|\mathbf{Q}\mathbf{K}^\top - \mathbf{Q}\widehat{\mathbf{K}}^\top\|_F$. 
Then, defining $\mathbf{M}_q = \mathbf{Q}^\top\mathbf{Q}$, our goal is to design a transform and a quantizer that minimize:
\begin{equation}
\bigl\| \mathbf{Q}\mathbf{K}^\top -
      \mathbf{Q}\widehat{\mathbf{K}}^{\top} \bigr\|_F^2
    \\ = \sum_{j=1}^{M} \, (\mathbf{k}_j - \widehat{\mathbf{k}}_j)\,
      \mathbf{M}_q\, (\mathbf{k}_j - \widehat{\mathbf{k}}_j)^\top, 
  \label{eq:qmse}
\end{equation}
for a given rate. 
Denoting 
$\|\mathbf{a}\|_{\mathbf{M}}^2 = \mathbf{a}\,\mathbf{M}\,\mathbf{a}^\top$, the right side of 
 \eqref{eq:qmse} can be rewritten as 
  $\sum_j \|\mathbf{k}_j-\hat{\mathbf{k}}_j\|^2_{\mathbf{M}_q}$, i.e.,    
      an input-weighted quadratic
distortion with \emph{constant} sensitivity matrix $\mathbf{M}_q$
\cite{linder1999high}. 

To understand how a $\mathbf{M}_q$-based distortion affects the design, consider quantization based on companding. In the SQ case, this involves defining an invertible mapping $h(x)$, so that given $\mathcal{Q}_u$, a uniform SQ, $\hat{x} = h^{-1}(\mathcal{Q}_u(h(x)))$. For vectors, this can be generalized by selecting an invertible matrix $\mathbf U$, so that $h_U(\mathbf x) = h(\mathbf{x}\mathbf{U})$. Then, companding can be applied directly by using $\mathcal{Q}_u$ for each entry (e.g., entry-wise SQ) and defining $\hat{\mathbf{x}} = h_U^{-1}(\mathcal{Q}_u(h_U(\mathbf{x})))$. High-resolution quantization theory shows that we can minimize the $\mathbf{M}_q$ distortion in \eqref{eq:qmse} via companding and an entry-wise SQ by selecting the optimal $h_U(\cdot)$, i.e., the one that minimizes \eqref{eq:qmse}, satisfying
$h_U'(\mathbf x)h_U'(\mathbf x)^{\top} =  c\,\mathbf{M}_q$, which, for linear companding characterized by a matrix $\mathbf U$ implies $ \mathbf U\mathbf U^\top = c\,\mathbf{M}_q$ \cite{linder1999high}. 
Orthogonal transforms are
suboptimal for the high-resolution regime unless $\mathbf{M}_q \propto \mathbf{I}$. 

The condition $\mathbf{U}\mathbf{U}^\top = c\,\mathbf{M}_q$ determines the
transform only up to an orthogonal factor: if $\mathbf{U}$ satisfies it, so does
$\mathbf{U}\mathbf{W}$ for any orthogonal $\mathbf{W}$, since
$\mathbf{U}\mathbf{W}\mathbf{W}^\top\mathbf{U}^\top =
\mathbf{U}\mathbf{U}^\top$. We fix this factor by optimizing energy compaction,
as the KLT does in \eqref{eq:klt_optimal}, but under the $\mathbf{M}_q$-weighted
cost of \eqref{eq:qmse} rather than the Euclidean one. Assume $M$ tokens in the calibration set and
let $\bar{\mathbf{k}} = {1}/{M}\sum_j \mathbf{k}_j$; centering the keys has
no cost to attention (Appendix~\ref{sec:lemma-offset}). We write $\mathbf{R}^{\dagger} = (\mathbf{R}^\top\mathbf{M}_q^{-1}\mathbf{R})^{-1}
\mathbf{R}^\top\mathbf{M}_q^{-1}$ for the $\mathbf{M}_q$-weighted pseudoinverse.

\begin{theorem}
\label{thm:transform}
Let $\widetilde{\mathbf{k}}_j = \mathbf{k}_j - \bar{\mathbf{k}}$,
$\widetilde{\mathbf{S}}_k = \sum_j \widetilde{\mathbf{k}}_j^\top\widetilde{\mathbf{k}}_j$,
and let
\begin{equation}
\label{eq:key_objective}
\mathbf{R}^\star_{K,p}
  = \operatorname*{arg\,min}_{\mathbf{R}\,\in\,\mathbb{R}^{d\times p}}\, 
  \sum_{j=1}^{M} \, 
  \bigl\|\,\widetilde{\mathbf{k}}_j - \widetilde{\mathbf{k}}_j\,\mathbf{R}\mathbf{R}^{\dagger}\bigr\|_{\mathbf{M}_q}^{2}.\end{equation}
Then $\mathbf{R}^\star_{K,p} = \mathbf{M}_q^{1/2}\,\mathbf{E}_{1:p}$, where
$\mathbf{E}\boldsymbol{\Lambda}\mathbf{E}^\top$ is the eigendecomposition of
$\mathbf{M}_q^{1/2}\,\widetilde{\mathbf{S}}_k\,\mathbf{M}_q^{1/2}$.
\end{theorem}
Proof: Appendix~\ref{sec:proof-thm1}.   We construct $\mathbf R_K$ by computing the eigendecomposition of $\mathbf{M}_q^{1/2}\,\widetilde{\mathbf{S}}_k\,\mathbf{M}_q^{1/2}$ and then multiplying by $\mathbf{M}_q^{1/2}$; $\mathbf{R}^\star_{K,p}$ corresponds to the matrix truncated by keeping the eigenvectors corresponding to the dominant eigenvalues of $\mathbf{M}_q^{1/2}\,\widetilde{\mathbf{S}}_k\,\mathbf{M}_q^{1/2}$, i.e., $\mathbf R_K = \mathbf{M}_q^{1/2}\,\mathbf{E}$. The centered key is
encoded as $\mathbf{r}_j = (\mathbf{k}_j - \bar{\mathbf{k}})\,\mathbf{R}_K$
and reconstructed as $\widehat{\mathbf{k}}_j = \mathbf{r}_j\,
\mathbf{R}^{-1}_K + \bar{\mathbf{k}}$. We have
$\mathbf{R}_K\mathbf{R}_K^\top = \mathbf{M}_q$, as demanded by \cite{linder1999high}, and the transform satisfies a generalized Parseval relation \cite{girault2018irregularity}: 
\begin{proposition}
\label{prop:parseval}
Let $\mathbf r$ and $\hat{\mathbf{r}}$ be any two vectors, and $\mathbf k = \mathbf{r}\mathbf{R}_K^{-1}$ and $\hat{\mathbf k} = \hat{\mathbf{r}}\mathbf{R}_K^{-1}$. Then, $\bigl\| \mathbf{r} - \hat{\mathbf{r}} \bigr\|_2^2
  = \bigl\| \mathbf k - \hat{\mathbf{k}} \bigr\|_{\mathbf M_q}^2$.
\end{proposition}

Prop.~\ref{prop:parseval} is analogous to the Parseval relation in
Sec.~\ref{ssec:tc} for the $\mathbf{M}_q$-MSE:
the transform converts the weighted objective into ordinary MSE. Thus, we can run any off-the-shelf MSE-optimal quantizer on the transform coefficients. Since $\mathbf{R}_K$ satisfies the linear compander condition, it is optimal for the associated high-resolution attention-weighted companding problem. Its orthogonal factor is chosen to optimize the weighted low-rank approximation in Theorem 1.

\subsection{Value transform}
We derive the value transform from
$\| \mathbf{S}\mathbf{V} - \mathbf{S}\widehat{\mathbf{V}} \|_F^2
= \sum_{j,j'} (\mathbf{M}_s)_{jj'} \, (\mathbf{v}_j - \widehat{\mathbf{v}}_j)
(\mathbf{v}_{j'} - \widehat{\mathbf{v}}_{j'})^\top$, and
$\mathbf{M}_s = \mathbf{S}^\top\mathbf{S}$, the second moment of the attention scores. 
\begin{corollary}
\label{cor:values}
Let $\mathbf{o}_i = \mathbf{s}_i \mathbf{V}$ denote the attention
outputs on calibration data, and
$\mathbf{E}\boldsymbol{\Lambda}\mathbf{E}^\top$ the eigendecomposition
of $\mathbf{M}_o =
\mathbf{V}^\top \mathbf{M}_s \mathbf{V}$. The minimizer of $\| \mathbf{S}\mathbf{V} - \mathbf{S}\widehat{\mathbf{V}} \|_F^2$ over low-rank approximations is
$\mathbf{R}_V = \mathbf{E}_{1:p}$, with codes
$\mathbf{s}_j = \mathbf{v}_j \mathbf{R}_V$ and reconstruction
$\widehat{\mathbf{v}}_j = \mathbf{s}_j \mathbf{R}_V^\top$.
\end{corollary}
The proof is in Appendix~\ref{sec:proof-cor1}. Unlike $\mathbf{R}_K$, the value transform \emph{is}
orthogonal. The term
$\mathbf{V}^\top\mathbf{M}_s\mathbf{V}$ is the covariance used in \cite{zhou2026oscar}.

\section{KV cache vector quantization}
\label{sec:vq}
Our goal is to design quantizers to minimize distortion at a fixed rate, because serving imposes a fixed-width constraint: tokens must occupy a fixed number of bits, and rate cannot be reallocated across entries. 
We achieve this via VQ, in which groups of entries are jointly quantized. 
This leaves the choice of partition, i.e., how to group the entries,  as the only degree of freedom. We show that under standard information-theoretic assumptions, there is a criterion that identifies which groupings are rate-distortion optimal when per-group rate is fixed, and we show how to find them. 

Let $\pi$ denote a partition of $\{1,\dots,d\}$ into $L = d/g$ groups
$G_1,\dots,G_L$ of equal size $g$, and let
$\mathbf{r}_{G_\ell}$ be the transform
coefficients indexed by $G_\ell$. 
We first assume that each group has its own
codebook $\mathcal{C}_\ell$ with $b_\ell$ bits, subject to 
$\sum_{\ell} g b_\ell = d b$. Later, we enforce the constraint that $b_\ell=b$.  
By Proposition~\ref{prop:parseval}, we can work with the MSE, which satisfies $D(\pi, \{b_\ell\}) = \sum_\ell D_\ell(b_\ell)$. We model the coefficients as independent zero-mean Gaussians with
variances $\sigma_1^2 \geq \cdots \geq \sigma_d^2$ (independence
is assumed for tractability.) Under the high-resolution regime,
the 
distortion-rate function of the optimal $g$-dimensional quantizer for
group $G_\ell$ is \cite{zador1982asymptotic}:
\begin{equation}
  D_\ell(b_\ell)
  =C_g \,
    2^{-2 b_\ell}
    \Bigl(\prod_{i\in G_\ell}\sigma_i^2\Bigr)^{1/g}
    \bigl(1+o(1)\bigr),
  \label{eq:zador}
\end{equation}
where $C_g$ is common to all groups. The $o(1)$ term in \eqref{eq:zador} vanishes for high rates, and we assume continuous rates $b_\ell \in \mathbb{R}$. The partition enters \eqref{eq:zador} only
through the \emph{volume}
$v_\ell(\pi) \doteq \prod_{i\in G_\ell}\sigma_i^{2}$. We seek the partition and per-group rates minimizing  $D(\pi, \{b_\ell\})$, assuming: 
\emph{variable rates}, where $\{b_\ell\}$ can take any value subject to the
budget $\sum_\ell b_\ell = Lb$, and \emph{fixed rates}, where
$b_\ell \equiv b$ as required by fixed-width packing. Note that $\prod_\ell v_\ell(\pi) = \prod_{i=1}^d \sigma_i^2$ for every $\pi$.

\begin{theorem}
\label{thm:grouping}
Under the model \eqref{eq:zador}, for any partition $\pi$ the allocation
minimizing $D(\pi,\{b_\ell\})$ subject to $\sum_\ell b_\ell = Lb$ is
\begin{equation}
\label{eq:alloc}
  b^{*}_\ell(\pi) \;=\; b \;+\; \frac{1}{2g}\,
  \log_2\Big({v_\ell(\pi)}\big/{\prod_m v_m(\pi)^{1/L}}\Big),
\end{equation}
which spends more bits on groups of larger volume. The optimal distortion, 
$D^{*}(b) = C_gL\,2^{-2b}\bigl(\prod_{i=1}^d\sigma_i^2\bigr)^{1/d}$, is the same
for every $\pi$,  assuming that rates can be the arbitrary real values from \eqref{eq:alloc}.
\end{theorem}

The proof is in Appendix~\ref{sec:proof-thm2}. The allocation \eqref{eq:alloc} depends on the partition only through the
volumes, and it is uniform, $b^{*}_\ell = b$ for all $\ell$, only when
$v_1(\pi) = \cdots = v_L(\pi)$. 
Hence, for a volume-equalizing partition, the optimal allocation is already fixed-width. Instead, if the volumes are not equal and we apply the same number of bits,  $b_\ell \equiv b$, the distortion will be suboptimal, 
$2^{-2b}\sum_\ell v_\ell(\pi)^{1/g} \ge D^{*}(b)$, with equality iff
$v_1(\pi) = \cdots = v_L(\pi)$. We therefore group entries into equal-volume groups and spend the same number of
bits per group; within each group, the codebook places more resolution on the
entries of higher variance in the query-weighted transform domain.

The distortion of a partition is proportional to the average of the $v_\ell(\pi)^{1/g}$,
while $D^*(b)$ is proportional to their geometric mean, so the penalty grows with the spread of the
group volumes. Sorting the entries and grouping them consecutively maximizes spread, while a random partition
averages log-variances per group (closer to balanced). Balancing $\sum_{i\in G_\ell}\log\sigma_i^2$ across groups 
is NP-hard (reduces to $3$-Partition \cite{garey2002computers}); we use a heuristic. 
Since each vector has $g$ entries, we will have $L=d/g$ vectors to quantize. Then, we sort the entries by decreasing order of $\log\sigma_i^2$ and, to approximately equalize their volume, group $G_\ell$ contains  entries with index $i \equiv \ell-1 \mod L$. The partition is folded into the
transform, so it costs nothing at inference.

We apply VQ to the keys and keep
SQ for the values since they are well
approximated at this rate: replacing the value quantizer with VQ
leaves accuracy almost unchanged (Appendix~\ref{app:experiments}). \textbf{Complexity.} Write, once per token: centering, and a nearest-neighbor search over centroids per group. Read: a lookup into a codebook of $2^{gb}g$ scalars per group, and one multiply to restore per-token scale.

\section{Experiments}
\label{sec:experiments}
We use $12$ A100 GPUs (40GB); for throughput, we use a single H100 (details in Appendix~\ref{app:experiments}). The high-resolution model might not hold at $2$
bits, we use the theory to guide our design rather than as a guarantee. As models, we use Llama-3.1-8B \cite{grattafiori2024llama}, Qwen3, both 8B and 4B-Thinking \cite{yang2025qwen3}, and GPT-OSS \cite{agarwal2025gpt}. For GPT-OSS, we quantize the KV cache only in its full-attention layers, leaving sliding-window layers (128-token window) unchanged.

\paragraph{Transform analysis.}
\label{ssec:transform-analysis} We fix the quantizer: uniform SQ
per coordinate with per-coordinate entropy coding, fit on the calibration
split and evaluated on held-out data. 
Entropy coding removes rate allocation as a confound, so the comparison reflects the basis alone; OSCAR's transforms target fixed-rate SQ and gain no compaction here. We compare four
bases: the KLT of the keys (MSE-optimal), the eigenbasis of the query second
moment used by OSCAR, the full OSCAR transform, and ours.
Fig.~\ref{fig:transform-analysis} shows the MSE, the query-weighted
distortion \eqref{eq:qmse}, and the top-1 attention agreement, i.e., the
fraction of queries whose highest-scoring key survives quantization. 
MSE does not determine attention fidelity: the KLT attains the lowest MSE at
every rate (a), yet ours, with a higher MSE, preserves the top-1 key more often
(c). 

\begin{figure}
    \centering
    \includegraphics[width=0.7\linewidth]{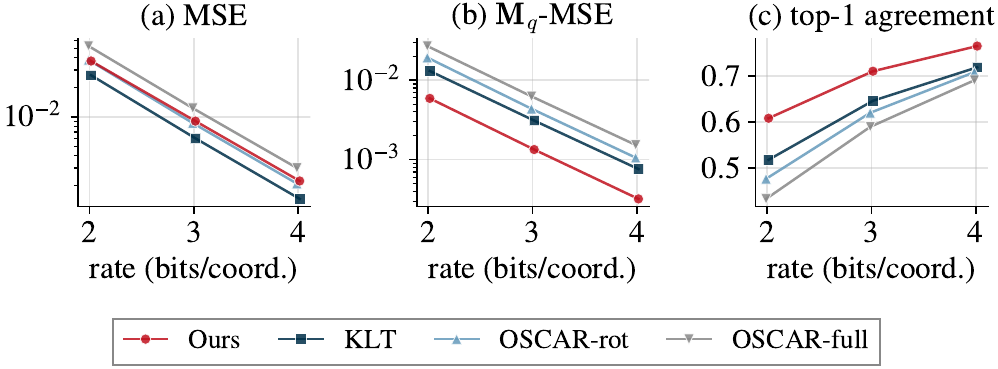}
\caption{Bases under a fixed quantizer. The KLT minimizes MSE (a) but not top-1
agreement (c); the $\mathbf{M}_q$-MSE (b) orders the bases like (c); our transform is best in (b-c).}
\label{fig:transform-analysis}
\end{figure}
\begin{table}[!htbp]
\small
\centering
\caption{RULER NIAH mean accuracy (\%) across context lengths, mean\,$\pm$\,95\% CI over 3 evaluation rollouts. BPE is the effective bits per KV element at 128K context. Best compressed method per column in \textbf{bold}.}
\label{tab:niah}
\setlength{\tabcolsep}{6pt}
\renewcommand{\arraystretch}{0.9}
\begin{tabular}{llcccccc}
\toprule
\textbf{Model} & \textbf{Method} & \textbf{BPE} & \textbf{8K} & \textbf{16K} & \textbf{32K} & \textbf{64K} & \textbf{128K} \\
\midrule
\multirow{4}{*}{\shortstack[l]{Qwen3-4B\\-Thinking-2507}}
 & BF16 & 16 & \ci{99.5}{0.4} & \ci{98.4}{0.4} & \ci{98.7}{0.2} & \ci{93.1}{0.8} & \ci{89.9}{2.1} \\
 & QuaRot & 2.25 & \ci{0.0}{0.0} & \ci{0.0}{0.0} & \ci{13.4}{0.8} & \ci{0.0}{0.0} & \ci{0.0}{0.0} \\
 & OSCAR & 2.28 & \cib{99.7}{0.0} & \ci{94.0}{0.7} & \ci{81.6}{0.8} & \ci{63.0}{0.8} & \ci{34.7}{0.5} \\
 & \cellcolor{oursrow}\textbf{\method{}} (ours) & \cellcolor{oursrow}2.22 & \cellcolor{oursrow}\ci{99.2}{0.4} & \cellcolor{oursrow}\cib{97.1}{0.0} & \cellcolor{oursrow}\cib{97.8}{0.1} & \cellcolor{oursrow}\cib{82.9}{0.8} & \cellcolor{oursrow}\cib{79.2}{0.5} \\
\midrule
\multirow{4}{*}{Qwen3-8B}
 & BF16 & 16 & \ci{99.9}{0.0} & \ci{99.4}{0.0} & \ci{98.5}{0.1} & \ci{84.3}{0.6} & \ci{83.4}{0.6} \\
 & QuaRot & 2.25 & \ci{54.8}{0.2} & \ci{31.3}{0.2} & \ci{23.5}{0.5} & \ci{0.0}{0.0} & \ci{0.0}{0.0} \\
 & OSCAR & 2.28 & \ci{96.7}{0.3} & \ci{94.6}{0.2} & \ci{86.9}{0.2} & \ci{60.6}{0.6} & \ci{25.3}{0.8} \\
 & \cellcolor{oursrow}\textbf{\method{}} (ours) & \cellcolor{oursrow}2.22 & \cellcolor{oursrow}\cib{99.4}{0.1} & \cellcolor{oursrow}\cib{98.7}{0.0} & \cellcolor{oursrow}\cib{96.0}{0.2} & \cellcolor{oursrow}\cib{76.4}{0.4} & \cellcolor{oursrow}\cib{75.4}{0.1} \\
\midrule
\multirow{4}{*}{Llama-3.1-8B}
 & BF16 & 16 & \ci{97.9}{3.0} & \ci{98.3}{1.8} & \ci{97.3}{3.3} & \ci{97.3}{2.4} & \ci{84.2}{10.7} \\
 & QuaRot & 2.25 & \ci{37.2}{25.2} & \ci{33.8}{26.5} & \ci{35.2}{26.6} & \ci{34.7}{27.1} & \ci{4.8}{5.5} \\
 & OSCAR & 2.28 & \ci{86.1}{12.2} & \ci{81.6}{17.7} & \ci{82.4}{19.4} & \ci{76.9}{21.9} & \ci{36.7}{25.2} \\
 & \cellcolor{oursrow}\textbf{\method{}} (ours) & \cellcolor{oursrow}2.22 & \cellcolor{oursrow}\cib{94.5}{3.4} & \cellcolor{oursrow}\cib{91.9}{5.9} & \cellcolor{oursrow}\cib{92.1}{7.2} & \cellcolor{oursrow}\cib{92.4}{5.6} & \cellcolor{oursrow}\cib{63.3}{24.6} \\
\midrule
\multirow{4}{*}{GPT-OSS-20B}
 & BF16 & 16 & \ci{95.8}{4.4} & \ci{95.7}{3.4} & \ci{94.6}{3.2} & \ci{92.2}{5.5} & \ci{80.4}{2.2} \\
& QuaRot & 2.50 & \ci{0.0}{0.0} & \ci{0.0}{0.0} & \ci{0.0}{0.0} & \ci{0.0}{0.0} & \ci{0.0}{0.0} \\
& OSCAR & 2.53 & \ci{0.5}{0.2} & \ci{0.0}{0.0} & \ci{0.0}{0.0} & \ci{0.0}{0.0} & \ci{0.0}{0.0} \\
 & \cellcolor{oursrow}\textbf{\method{}} (ours) & \cellcolor{oursrow}2.41 & \cellcolor{oursrow}\cib{89.6}{1.1} & \cellcolor{oursrow}\cib{81.0}{2.6} & \cellcolor{oursrow}\cib{79.2}{3.8} & \cellcolor{oursrow}\cib{70.6}{4.1} & \cellcolor{oursrow}\cib{54.0}{4.4} \\
\bottomrule
\end{tabular}
\end{table}

\begin{table}[!htbp]
\small
\centering
\caption{KV-cache quantization compared across models and benchmarks.
Entries are mean\,$\pm$\,95\% CI with 5 samples per prompt. BPE denotes effective bits per KV element at 128K
context. ``Drop'' is the gap in the Mean column to the BF16 reference; less negative is better.
Best compressed method per column in \textbf{bold}. Naive INT2 yields $0.0$ for all models and tasks.}
\label{tab:main}
\setlength{\tabcolsep}{3pt}
\renewcommand{\arraystretch}{0.8}
\begin{tabular}{llccccccrr}
\toprule
\textbf{Model} & \textbf{Method} & \textbf{BPE}
 & \textbf{GPQA} & \textbf{HumanE} & \textbf{LCB v6} & \textbf{AIME25} & \textbf{MATH500}
 & \textbf{Mean} & \textbf{Drop} \\
\midrule
\multirow{6}{*}{\shortstack[l]{Qwen3-4B\\-Thinking-2507}}
 & BF16 & 16.00 & \ci{64.9}{0.8} & \ci{84.8}{0.5} & \ci{59.2}{0.9} & \ci{70.0}{1.5} & \ci{97.5}{0.2} & 75.3 & -- \\
 
 & TurboQuant & 3.25 & \ci{44.3}{1.5} & \ci{31.2}{0.4} & \ci{1.3}{1.5} & \ci{2.2}{3.8} & \ci{56.6}{0.7} & 27.1 & $-$48.2 \\
 & QuaRot & 2.25 & \ci{12.6}{1.7} & \ci{6.1}{1.3} &\ci{3.0}{1.2} & \ci{0.0}{0.0} & \ci{32.5}{1.3} & 10.8 & $-$64.5 \\
 & OSCAR & 2.28 & \ci{62.4}{0.9} & \ci{83.8}{0.2} & \cib{57.4}{0.8} & \ci{63.3}{2.1} & \ci{96.8}{0.2} & 72.7 & $-$2.6 \\ 
 & \cellcolor{oursrow}\textbf{\method{}} (ours) & \cellcolor{oursrow}2.22 & \cellcolor{oursrow}\cib{62.9}{1.3} & \cellcolor{oursrow}\cib{84.3}{0.4} & \cellcolor{oursrow}\ci{56.6}{0.6} & \cellcolor{oursrow}\cib{67.3}{1.2} & \cellcolor{oursrow}\cib{97.4}{0.2} & \cellcolor{oursrow}\textbf{73.7} & \cellcolor{oursrow}\textbf{$-$1.6} \\ 
\midrule
\multirow{6}{*}{Qwen3-8B}
 & BF16 & 16.00 & \ci{57.2}{0.9} & \ci{91.7}{0.5} & \ci{57.6}{0.7} & \ci{70.0}{1.7} & \ci{97.0}{0.1} & 74.7 & -- \\
 & TurboQuant & 3.25 & \ci{45.6}{6.2} & \ci{76.0}{1.9} & \ci{32.7}{2.5} & \ci{45.6}{5.1} & \ci{94.3}{0.1} & 58.8 & $-$16.9 \\
 & QuaRot & 2.25 & \ci{44.2}{1.3} & \ci{39.0}{3.7} & \ci{9.0}{2.0} & \ci{20.0}{3.3} & \ci{75.2}{0.4} & 37.5 & $-$37.2 \\
 & OSCAR & 2.28 & \ci{53.5}{1.2} & \ci{90.9}{0.7} & \ci{53.6}{0.2} & \cib{66.2}{1.7} & \ci{96.7}{0.2} & 72.2 & $-$2.5 \\
 & \cellcolor{oursrow}\textbf{\method{}} (ours) & \cellcolor{oursrow}2.22 & \cellcolor{oursrow}\cib{55.1}{1.5} & \cellcolor{oursrow}\cib{92.1}{0.2} & \cellcolor{oursrow}\cib{55.6}{0.6} & \cellcolor{oursrow}\ci{64.6}{2.5} & \cellcolor{oursrow}\cib{96.9}{0.2} & \cellcolor{oursrow}\textbf{72.9} & \cellcolor{oursrow}\textbf{$-$1.8} \\
\midrule
\multirow{6}{*}{Llama-3.1-8B}
 & BF16 & 16.00 & \ci{25.5}{2.7} & \ci{63.5}{2.3} & \ci{10.5}{1.7} & \ci{0.0}{0.0} & \ci{44.2}{1.9} & 28.7 & -- \\
 & TurboQuant & 3.25 & \ci{22.7}{2.3} & \ci{56.7}{2.8} & \cib{11.2}{1.7} & \ci{0.0}{0.0} & \ci{37.3}{1.2} & 25.6 & $-$3.2 \\
 & QuaRot & 2.25 & \ci{2.8}{1.0} & \ci{6.7}{1.4} & \ci{0.7}{0.5} & \ci{0.7}{1.3} & \ci{6.0}{0.8} & 3.4 & $-$25.4 \\
 & OSCAR & 2.28 & \cib{23.1}{2.6} & \ci{64.4}{2.5} & \ci{9.4}{1.4} & \cib{3.3}{4.1} & \ci{44.8}{1.1} & 29.0 & $+$0.3 \\
 & \cellcolor{oursrow}\textbf{\method{}} (ours) & \cellcolor{oursrow}2.22 & \cellcolor{oursrow}\ci{21.3}{2.6} & \cellcolor{oursrow}\cib{64.5}{1.6} & \cellcolor{oursrow}\ci{10.2}{1.7} & \cellcolor{oursrow}\cib{3.3}{2.9} & \cellcolor{oursrow}\cib{45.8}{1.3} & \cellcolor{oursrow}\textbf{29.0} & \cellcolor{oursrow}\textbf{$+$0.3} \\
 \midrule
 \multirow{6}{*}{GPT-OSS-20B}
 & BF16 & 16.00 & \ci{50.4}{3.1} & \ci{90.2}{2.9} & \ci{76.2}{2.3} & \ci{74.0}{7.0} & \ci{91.8}{1.1} & 76.5 & -- \\
 & TurboQuant & 3.50 & \ci{35.2}{3.0} & \ci{33.4}{3.8} & \ci{30.2}{4.1} & \ci{17.3}{3.5} & \ci{70.9}{1.8} & 37.4 & $-$39.1 \\
& QuaRot & 2.50 & \ci{0.0}{0.0} & \ci{0.0}{0.0} & \ci{0.0}{0.0} & \ci{0.0}{0.0} & \ci{0.0}{0.0} & 0.0 & $-$76.5 \\
& OSCAR & 2.53 & \ci{25.1}{2.7} & \ci{5.4}{1.6} & \ci{0.0}{0.0} & \ci{0.0}{0.0} & \ci{39.0}{1.9} & 13.9 & $-$62.6 \\
 & \cellcolor{oursrow}\textbf{\method{}} (ours) & \cellcolor{oursrow}2.41 & \cellcolor{oursrow}\cib{46.5}{3.1} & \cellcolor{oursrow}\cib{84.6}{3.6} & \cellcolor{oursrow}\cib{71.9}{2.5} & \cellcolor{oursrow}\cib{67.3}{7.5} & \cellcolor{oursrow}\cib{91.8}{1.1} & \cellcolor{oursrow}\textbf{72.4} & \cellcolor{oursrow}\textbf{$-$4.1} \\

\bottomrule
\end{tabular}
\end{table}
\paragraph{Calibration.} As in  \cite{zhou2026oscar}, the key transform and codebooks are fit offline from $198$ GPQA-Diamond prompts, run in prefill
only, with per-(layer, KV-head) statistics. For GQA \cite{ainslie2023gqa}, we calibrate across heads sharing the same KV cache. For the keys, we fit codebooks with $256$ centroids (vector size $4$, average $2$ bits/coordinate) by $k$-means. We concatenate
these prompts into long-context sequences of $128$K so the data spans the RoPE positions at which the transform and codebooks are evaluated. No label, gold answer, model generation, or task metric enters the fit. For values, we use the SQ from \cite{zhou2026oscar}. Refitting the transform, grouping, and codebooks on MMLU instead, a disjoint domain, leaves RULER NIAH within $2.0$ points and does not lower GPQA, so calibrating on the evaluation domain confers no advantage (Appendix~\ref{app:calib}).

\paragraph{Baselines.}  As in \cite{xiao2024efficient, zhou2026oscar}, we keep the first
$64$ tokens and the most recent $256$ in BF16 and quantize the
rest. We run QuaRot, TurboQuant, and OSCAR under the configuration of \cite{zhou2026oscar}: matching group sizes, per-channel/per-token axes, and BF16 sink-plus-recent band. For OSCAR we evaluated both the released transform and one refit on our calibration split; the two are within $1.5$ points on RULER-NIAH, and we report the better of the two. TurboQuant is included at 3.25 BPE, following \cite{zhou2026oscar}; \citet{zandieh2025turboquant} report quality neutrality at 3.5 bits per channel and marginal quality degradation at 2.5.

\paragraph{Long context.}
We use RULER NIAH \cite{hsieh2024ruler}, a benchmark for long context robustness, and sweep the context
from $8$K to $128$K. We run under chunked prefill, as serving engines do to bound iteration
latency: the prompt is processed in fixed-size chunks and each chunk is
quantized on write, so every subsequent chunk attends to an already
quantized history. Deferring quantization until the whole prompt is
prefilled avoids this accumulation, but holds the cache in BF16 at its
peak. All methods run in the same
harness with the same chunk size. Numbers for prior methods might be lower than those reported in the original papers, which do not specify a chunked prefill. Yet, our implementations remain comparable to their reported behavior on the reasoning and coding benchmarks (cf.~ Table \ref{tab:main}). We report the mean over 3 rollouts with a 95\% CI in Table~\ref{tab:niah}. The baselines degrade with
context while \method{} tracks BF16 more closely: on Qwen3-8B at $128$K,
QuaRot and OSCAR reach $0.0$ and $25.3$ against $83.4$ for BF16,
where \method{} reaches $75.4$. At $128$K, \method{} exceeds OSCAR by $50.1$ points (McNemar $p < 10^{-37}$). The same ordering holds on Llama-3.1-8B, where \method{} reaches $92.4$ at $64$K and $63.3$ at $128$K against $76.9$ and $36.7$ for OSCAR.

The separation is starkest on GPT-OSS-20B, where both baselines fail rather than degrading with context: QuaRot scores $0.0$ at every length, and OSCAR reaches $0.5$ at $8$K and $0.0$ from $16$K on, so neither retrieves anything at any context we test. \method{} instead retains $89.6$ at $8$K and $54.0$ at $128$K, against $95.8$ and $80.4$ for BF16. This is the only model on which the orthogonal-transform baselines lose the task entirely, and it is also the one whose architecture is closest to current practice: GPT-OSS-20B interleaves sliding-window and full-attention layers and routes its feed-forward blocks through a mixture of experts, a combination that recent open-weight models have increasingly adopted. 

\paragraph{Accuracy comparisons.}
We test on GPQA-Diamond \cite{rein2023gpqa}, HumanEval
\cite{chen2021evaluating}, LiveCodeBench v6 \cite{jain2025livecodebench},
AIME25 \cite{maa2025aime}, and MATH500 \cite{lightman2024let, hendrycks2021measuring}. We draw $5$ samples per prompt (Table~\ref{tab:main}). Against BF16, \method{} shows no detectable degradation on either Qwen model
(McNemar $p = 0.17$, $0.24$), while OSCAR is significantly below BF16 on
Qwen3-4B ($p = 0.002$). The two compressed methods are not separable there ($p = 0.46$), and on these two models the significant separation appears only on long-context retrieval. The remaining two models separate them directly. On Llama-3.1-8B every arm tracks BF16 within the confidence intervals, so the methods are hard to tell apart: the uncompressed model already scores low on these tasks. On GPT-OSS-20B, by contrast, \method{} preserves most of the BF16 accuracy on every benchmark (mean $72.4$ against $76.5$), whereas OSCAR collapses, scoring $5.4$ on HumanEval and $0.0$ on both LiveCodeBench v6 and AIME25 for a mean of $13.9$, and TurboQuant falls to $37.4$. \method{} therefore stays close to BF16 on all four models, while the robustness of the orthogonal-transform baselines varies considerably across architectures.

\paragraph{Bit rate.} At $128$K context, the K cache stores one $8$ bit index per group of $g=4$ entries ($2$ BPE, as
$|\mathcal{C}|{=}256$) together with a per-token BF16 scale  ($16/128 = 0.125$). The V cache is INT2 with a BF16 scale and offset (zero-point), giving $2.000 + 32/128 = 2.250$.
Averaging over the $256$ elements gives $2.1875$, and the BF16 sink-plus-recent band ($64$, $256$) adds $ 0.034$, for $2.22$ BPE. OSCAR stores a zero-point on both caches, giving $2.28$ BPE; the $0.06$ gap is the K-side zero-point. Transforms and codebooks are stored per
model.

\paragraph{Throughput.} We measure decode throughput in SGLang
\citep{zheng2024sglang} at input lengths of 30K, 60K and 90K (Fig.~\ref{fig:throughput_bars}). We include a warm-up
request, and throughput is computed over the decode window alone, excluding
time to first token. Cross-request prefix sharing is disabled. We report
aggregate tokens per second across the batch. At batch 1 and batch 4 the
comparison is bandwidth-dominated: NOVA-KV reaches $1.6$--$3.1\times$ BF16 on
Qwen3-8B and $1.7$--$3.4\times$ on Qwen3-4B, with the larger factors at the
longer inputs. The factor is smaller on GPT-OSS-20B, $1.1$--$1.5\times$, since
sliding-window reduces the share of the cache being compressed. Against OSCAR at the same rate, the two arms are within a few percent of
each other throughout. At 90K the BF16 KV pool admits at most $\mathtt{b4}$ on either Qwen model, so
it has no batch-16 configuration, whereas both 2-bit arms serve one.

\begin{figure}[!htbp]
    \centering
    \includegraphics[width=0.7\linewidth]{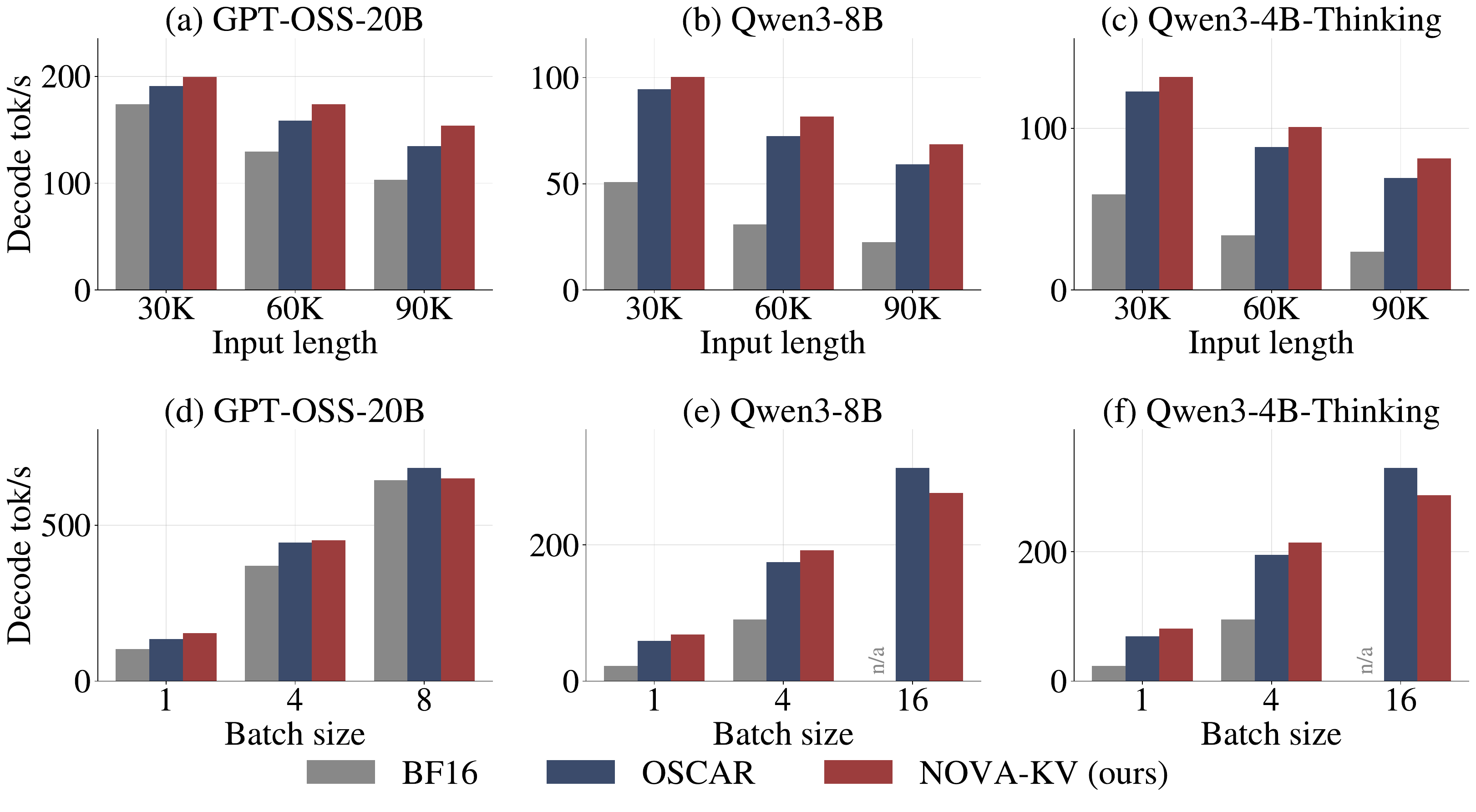}
\caption{Decode throughput, prefill excluded. Top row: batch 1 against input
length. Bottom row: batch scaling at 90K input. BF16+Qwen does not admit the largest batch at 90K.}
    \label{fig:throughput_bars}
\end{figure}

\paragraph{Transform-quantizer.} Table~\ref{tab:ablation} crosses transform (OSCAR, \method{}) and quantizer. Under SQ (same number of bits to each coordinate), our transform collapses: it compacts energy rather than flattening it, so a single rate cannot serve
every coordinate. Replacing SQ by VQ is worth $7.5$ and $37.4$ points under OSCAR's basis; replacing OSCAR's basis by ours, under VQ, is worth a further $8.3$ and $12.7$. Neither substitution alone reaches the full result.

\begin{table}[!htbp]
\small
\centering
\caption{Transform-quantizer ablation on RULER NIAH,
Qwen3-8B. Mean accuracy (\%) $\pm$ 95\% CI over 3 rollouts.}
\label{tab:ablation}
\setlength{\tabcolsep}{6pt}
\renewcommand{\arraystretch}{0.7}
\begin{tabular}{llccc}
\toprule
\textbf{Transform} & \textbf{Quantizer} & \textbf{BPE} & \textbf{64K} & \textbf{128K} \\
\midrule
OSCAR & SQ & 2.28 & \ci{60.6}{0.6} & \ci{25.3}{0.8}\\
OSCAR & VQ & 2.22 & \ci{68.1}{0.6} & \ci{62.7}{0.4} \\
\method{}  & SQ & 2.28 & \ci{0.0}{0.0} & \ci{0.0}{0.0}\\
\rowcolor{oursrow} \method{} & VQ & 2.22 & \cib{76.4}{0.4} & \cib{75.4}{0.1} \\
\bottomrule
\end{tabular}
\end{table}

\paragraph{Volume equalization.}
\label{ssec:grouping}
We fix the
transform, the codebook size, and the rate (2.22 BPE), varying only the
partition: \emph{variance-sorted} groups consecutive entries in decreasing
variance, the configuration with the most unbalanced volumes; \emph{random}
draws one fixed partition at random; and \emph{equalizing} is the
partition of Sec.~\ref{sec:vq}. Over 3 evaluation rollouts for Qwen3-8B, variance-sorted is the worst (\ci{53.2}{0.7} at 64K, \ci{37.2}{0.1} at 128K),
random recovers most of the gap (\ci{74.2}{0.5}, \ci{74.4}{0.2}), and the
equalizing partition is the highest (\ci{76.4}{0.4},
\ci{75.4}{0.1}), matching Theorem~\ref{thm:grouping}.

\section{Conclusion}
We formulated KV cache quantization as a transform coding problem whose
distortion is the attention product error. Under a high-resolution
model, the optimal key transform is not orthogonal, and satisfies a
generalized Parseval relation, turning the attention-aware criterion into MSE
in the transform domain.
Under the independent-Gaussian high-resolution model, equal-volume grouping makes fixed-rate VQ attain the variable-rate optimum, allowing a fixed-width cache without an asymptotic distortion penalty. At two bits per element, \method{} reduces the gap to BF16 relative to the
2-bit state-of-the-art at comparable decoding speed. Joint quantization of
keys and values is left for future work.

\clearpage
\beginappendix
\startcontents[appendix]
\printcontents[appendix]{}{1}{\setcounter{tocdepth}{2}}
\clearpage

\section{Notation and Definitions}
\label{app:notation}

This appendix collects the notation and definitions used throughout the paper.

\subsection{Conventions}

Uppercase bold letters ($\mathbf{A}$) denote matrices, lowercase bold letters
($\mathbf{a}$) denote row vectors, and regular letters denote scalars. The
$n$th entry of $\mathbf{a}$ is $a_n$ and the $(i,j)$th entry of $\mathbf{A}$
is $A_{ij}$. All vectors are row vectors, so a linear map $\mathbf R$ acts on the right:
$\mathbf{x}\mathbf{R}$. Eigenvalues are indexed in decreasing order,
$\lambda_1 \ge \cdots \ge \lambda_d$, and $\mathbf{E}_{1:p}$ collects the
eigenvectors associated with the $p$ largest eigenvalues. For a symmetric
positive definite (p.d.) $\mathbf{M}$ we let
$\|\mathbf{a}\|_{\mathbf{M}}^2 = \mathbf{a}\mathbf{M}\mathbf{a}^\top$, and
$\mathbf{M}^{1/2}$ denotes the symmetric square root, i.e.,\ the unique
symmetric p.d. $\mathbf{B}$ with $\mathbf{B}\mathbf{B} =
\mathbf{M}$. We write $\|\cdot\|_F$ for the Frobenius norm and
$\|\cdot\|_2$ for the spectral norm of a matrix and the $\ell_2$ norm of a vector.

\subsection{Attention and the cache}
Notation for attention, RoPE, and the cache follows Sec.~\ref{sec:attn-cache} of the main text. We let $\hat{\mathbf{K}}$, $\hat{\mathbf{V}}$ be the reconstructed keys and values after quantization, and $\hat{\mathbf{S}}$ be the scores computed from $\hat{\mathbf{K}}$. Under grouped-query attention
(GQA)~\cite{ainslie2023gqa}, the query heads are partitioned into groups, and each group shares a single key-value projection pair $(\mathbf{W}_K, \mathbf{W}_V)$; we call the shared projection pair, together with the cache entries it produces, a \emph{KV head}. Quantities are then formed per (layer, KV head); the query statistic $\mathbf{M}_q$ is accumulated over the query heads sharing the KV head (Appendix~\ref{sec:gqa}), and the transform, grouping, and codebooks are fit per KV head.

\subsection{Calibration statistics}
\label{app:calib-stats}

Let $M$ be the number of cached keys collected on calibration. The calibration statistics are the key mean
$\bar{\mathbf{k}} = {1}/{M}\, \sum_{j} \mathbf{k}_j$, the key covariance
$\widetilde{\mathbf{S}}_k = \sum_j
(\mathbf{k}_j-\bar{\mathbf{k}})^\top(\mathbf{k}_j-\bar{\mathbf{k}})$, the
query second moment $\mathbf{M}_q = \mathbf{Q}^\top\mathbf{Q}$, the score
second moment $\mathbf{M}_s = \mathbf{S}^\top\mathbf{S}$, and the output
second moment $\mathbf{M}_o = \mathbf{V}^\top\mathbf{M}_s\mathbf{V}$.
$\mathbf{M}_q$ is positive semidefinite by construction; we assume it is
positive definite for the analysis. Numerically, $\mathbf{M}_q^{1/2}$ and
$\mathbf{M}_q^{-1/2}$ are formed from its symmetric eigendecomposition
after flooring the eigenvalues at $10^{-30}$; no additive ridge is applied. For values, we use the default rotation matrix provided in \cite{zhou2026oscar}.

\subsection{Transforms, quantization, and grouping}
For the keys, the transform is $\mathbf{R}_K = \mathbf{M}_q^{1/2}\mathbf{E}$, with
$\mathbf{E}$ the eigenvectors of $\mathbf{M}_q^{1/2}\widetilde{\mathbf{S}}_k
\mathbf{M}_q^{1/2}$ (Theorem~1); the transform coefficients of key $j$ are
$\mathbf{r}_j = (\mathbf{k}_j-\bar{\mathbf{k}})\mathbf{R}_K$, and
$\sigma_1^2 \ge \cdots \ge \sigma_d^2$ denote their per-coordinate variances.
The query-weighted distortion, or \emph{$\mathbf{M}_q$-MSE}, of a reconstruction
is $\|\mathbf{k}-\hat{\mathbf{k}}\|_{\mathbf{M}_q}^2$; summed over cached keys it
equals the key-query inner-product error (cf.~Eq.~\eqref{eq:supp_kq}). 

A quantizer of group size $g$ maps subvectors of $g$
entries jointly; a partition $\pi$ of $\{1,\dots,d\}$ into $L = d/g$ groups
$G_1,\dots,G_L$ of size $g$ has group \emph{volume}
$v_\ell(\pi) = \prod_{i\in G_\ell}\sigma_i^2$, and is \emph{volume-equalizing}
when $v_1(\pi)=\cdots=v_L(\pi)$. Under \emph{fixed rates} every group gets the
same rate $b_\ell \equiv b$; under
\emph{variable rates} $\{b_\ell\}$ is free subject to $\sum_\ell b_\ell = Lb$.

In \eqref{eq:zador}, we let $C_g = g\,Z_g\,\gamma_g$, with $Z_g$ the normalized moment of inertia of the optimal
$g$-dimensional cell shape and $\gamma_g$ the source-density functional
of Zador's theorem; under
the independent-Gaussian model $C_g$ is common to all groups \cite{zador1982asymptotic} (no effect on the choice of partition).

The groupings compared in this paper are: \emph{variance-sorted}, $G_\ell = \{(\ell-1)g+1,\dots,\ell g\}$ after sorting entries by decreasing $\sigma_i^2$; \emph{random}, one partition drawn uniformly at random and then held fixed; and \emph{equalizing}, the partition of Sec.~\ref{sec:vq}, $G_\ell = \{\,i \in \{1,\dots,d\} : i-1 \equiv \ell-1 \ (\mathrm{mod}\ L)\,\}$ after the same sort. Fig.~\ref{fig:groupings} illustrates the three rules.

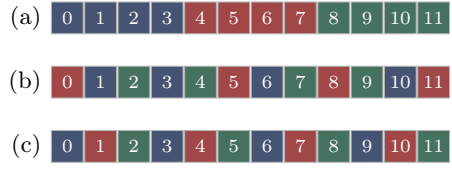
\begin{figure}[!htbp]
\centering
\definecolor{slate}{HTML}{3B4B6B}
\definecolor{brick}{HTML}{9C3D3D}
\definecolor{paergreen}{HTML}{3B6B57}
\begin{tikzpicture}[
  cell/.style={draw=gray!60, minimum width=0.42cm, minimum height=0.42cm,
               inner sep=0pt, font=\scriptsize, text=white},
  lab/.style={font=\small, anchor=east},
  g1/.style={cell, fill=slate!95},
  g2/.style={cell, fill=brick!95},
  g3/.style={cell, fill=paergreen!95},
  x=0.44cm, y=0.85cm]

\node[lab, opacity=0] at (12.0, 1) {(a)};

\node[lab] at (-0.5, 2) {(a)};
\foreach \i in {0,...,3}  \node[g1] at (\i, 2) {\i};
\foreach \i in {4,...,7}  \node[g2] at (\i, 2) {\i};
\foreach \i in {8,...,11} \node[g3] at (\i, 2) {\i};

\node[lab] at (-0.5, 1) {(b)};
\foreach \i/\s in {0/g2,1/g1,2/g3,3/g1,4/g3,5/g2,
                   6/g1,7/g3,8/g2,9/g3,10/g1,11/g2}
  \node[\s] at (\i, 1) {\i};

\node[lab] at (-0.5, 0) {(c)};
\foreach \i in {0,...,11}{
  \pgfmathtruncatemacro{\m}{mod(\i,3)+1}
  \node[g\m] at (\i, 0) {\i};
}
\end{tikzpicture}
\caption{The three partitions in a zero-indexed toy example ($d=12$, $g=4$, $L=3$);
color denotes group membership, entries ordered by decreasing
variance $\sigma_i^2$ from left to right. (a)~\emph{Variance-sorted}
groups consecutive spectrum entries. (b)~\emph{Random} mixes levels by chance. (c)~\emph{Equalizing}
deals coordinate $i$ to group $i \bmod L$, so each group receives one
coordinate per level (Theorem~2).}
\label{fig:groupings}
\end{figure}

\subsection{Symbol table}
We summarize important notations in Table~\ref{tab:notation}.

\begin{table}[!htbp]
\small
\setlength{\tabcolsep}{5pt}
\centering
\renewcommand{\arraystretch}{1.15}
\caption{Summary of notations used in this work.}
\label{tab:notation}
\begin{tabular}{ll}
\toprule
Symbol & Meaning \\
\midrule
$d$ & head dimension \\
$T$ & sequence length \\
$M$ & number of calibration keys \\
$\mathbf{q}_t,\mathbf{k}_t$ & post-RoPE query and key at position $t$ \\
$\mathbf{v}_t$ & value at position $t$ \\
$\mathbf{S}$, $\hat{\mathbf{S}}$ & attention scores from $\mathbf{K}$, from $\hat{\mathbf{K}}$ \\
$\bar{\mathbf{k}}$ & calibration key mean \\
$\mathbf{M}_q,\widetilde{\mathbf{S}}_k,\mathbf{M}_s,\mathbf{M}_o$ & calibration second-order statistics \\
$\mathbf{R}_K,\mathbf{R}_V$ & key and value transforms \\
$\mathbf{r}_j$ & transform coefficients of key $j$ \\
$\sigma_i^2$ & variance of transform coordinate $i$ \\
$\mathcal{Q}^{+},\mathcal{Q}^{-},\mathcal{Q}$ & encoder, decoder, their composition \\
$\mathcal{C}_\ell$ & codebook of group $\ell$ \\
$g$, $L$ & group size, number of groups \\
$b$, $b_\ell$ & rate, per-group rate \\
$\pi$, $G_\ell$, $v_\ell(\pi)$ & partition, group, group volume \\
$n_{\mathrm{sink}}$, $n_{\mathrm{rec}}$ & sink and recent band sizes \\
SQ, VQ & scalar quantization, vector quantization\\
\bottomrule
\end{tabular}
\end{table}

\subsection{High-level algorithm}

In Algorithm~\ref{alg:pipeline}, we detail our pipeline both during calibration and inference.

\begin{algorithm}[t]
\caption{Query-aware key transform and quantizer}
\label{alg:pipeline}
\begin{algorithmic}[1]
\item[] \hspace{-2em} \colorbox{deepspace!8}{\makebox[\dimexpr\linewidth-\algorithmicindent][l]{\textbf{Calibration (offline):}}}
\State Collect queries $\mathbf{Q}$ and keys $\mathbf{K}$ on calibration data
\State $\bar{\mathbf{k}} \gets \sum_{j=1}^M \mathbf{k}_j/{M}$;\quad
       $\mathbf{M}_q \gets \mathbf{Q}^\top\mathbf{Q}$;\quad
       $\widetilde{\mathbf{S}}_k \gets \sum_{j=1}^M
       (\mathbf{k}_j-\bar{\mathbf{k}})^\top(\mathbf{k}_j-\bar{\mathbf{k}})$
\State Eigendecompose
       $\mathbf{M}_q^{1/2}\widetilde{\mathbf{S}}_k\mathbf{M}_q^{1/2}
       = \mathbf{E}\boldsymbol{\Lambda}\mathbf{E}^\top$
\State $\mathbf{R}_K \gets \mathbf{M}_q^{1/2}\mathbf{E}$;\quad
       $\mathbf{R}_K^{-\top} = \mathbf{M}_q^{-1/2}\mathbf{E}$
\State Find volume-equalized groups $\pi$ (Thm.~2), fold into $\mathbf{R}_K$
\State Train per-group codebooks $\{\mathcal{C}_\ell\}$ on
       $(\mathbf{K}-\mathbf{1}\bar{\mathbf{k}})\mathbf{R}_K$
\item[] \hspace{-2em} \colorbox{brickred!8}{\makebox[\dimexpr\linewidth-\algorithmicindent][l]{\textbf{Inference (online):}}}
\State Write: cache
       $\hat{\mathbf{K}} = \mathcal{Q}^{+}\bigl((\mathbf{k}_j-\bar{\mathbf{k}})\mathbf{R}_K\bigr)$
\State Attention: $\mathbf{S} = \operatorname{softmax}\bigl(
       \mathbf{Q}\,\mathbf{R}_K^{-\top}\,
       \mathcal{Q}^{-}(\hat{\mathbf{K}})^{\top}
       / \sqrt{d} \bigr)$
\end{algorithmic}
\end{algorithm}

\section{Experimental Details}
\label{app:experiments}

In this section, we detail our experimental setup, including metrics, decoding configuration, calibration, serving setup, benchmarks, and datasets. We also ablate the VQ and SQ in the value arm to justify our deployed method.

\subsection{Reported quantities}

\paragraph{Bits per element (BPE).} The average number of stored bits per
cached scalar entry, over keys and values and accounting for the full context, including
per-token scales, zero-points, and the BF16 sink-plus-recent band, and
excluding the per-model transforms and codebooks, whose contribution is
reported separately. The cache-rate accounting ($2.22$ for ours, $2.28$ for
OSCAR) is given in the main text.

The per-model transforms and codebooks are stored once and are independent of
context and batch. Per (layer, KV head) we store the three $d\times d$ transforms
$\mathbf{R}_K,\mathbf{R}_K^{-\top},\mathbf{R}_V$, the mean $\bar{\mathbf{k}}$
($d$ scalars), and $L=d/g$ key codebooks of $2^{gb}$ centroids in
$\mathbb{R}^{g}$, that is $3d^2 + d + L\,2^{gb}g$ scalars. For instance, with Qwen3-8B, for $d=128$,
$g=4$, $2^{gb}=256$, $L=32$ this is $49152 + 128 + 32768 \approx
8.2\times10^{4}$ scalars per (layer, KV head); at $16$-bit transforms and $8$-bit
codebooks it totals $\approx 0.30$\,Gbit for the full $36$-layer, $8$-KV-head
model. Divided by the $2\,n_\text{layer}\,n_\text{kv}\,d\,L_\text{ctx}\,B$
quantized cache scalars, a  $L_\text{ctx}=128$K sequence carries
$\approx 0.03$ bits per KV element; since the same transforms and codebooks serve
the entire batch and all requests, the amortized contribution falls below $0.01$
bits per KV element at the batch sizes we serve, negligible against the
$\approx 2.2$ cache rate.

\paragraph{Top-1 attention agreement.} The fraction of queries
whose highest-scoring key under $\hat{\mathbf{K}}$ is the same as under
$\mathbf{K}$. It measures whether quantization preserves the argmax of the
logits, which the MSE does not constrain directly.

\paragraph{Sink-plus-recent band.} The first $n_{\mathrm{sink}}$ and the most recent $n_{\mathrm{rec}}$
tokens are held in BF16 and excluded from quantization; $n_{\mathrm{sink}}$ and $n_{\mathrm{rec}}$ are fixed
constants, so the band's share of the cache decreases as the context grows (see Fig.~\ref{fig:band}).

\begin{figure}[!htbp]
\centering
\definecolor{slate}{HTML}{3B4B6B}
\definecolor{brick}{HTML}{9C3D3D}
\definecolor{paergreen}{HTML}{3B6B57}
\begin{tikzpicture}[
  band/.style={draw=gray!60, minimum height=0.65cm, inner sep=0pt,
               font=\small, text=white, anchor=west},
  note/.style={font=\small, text=gray!40!black},
  tick/.style={font=\normalsize, text=gray!40!black, anchor=north},
  x=1cm]

\node[band, fill=slate!95,     minimum width=1.2cm] (sink) at (0,0)   {BF16};
\node[band, fill=brick!95,     minimum width=5.4cm] (hist) at (1.2,0) {quantized (\method{})};
\node[band, fill=paergreen!95, minimum width=1.6cm] (rec)  at (6.6,0) {BF16};

\node[tick] at (0,   -0.4) {$1\vphantom{n_{\mathrm{sink}}}$};
\node[tick] at (1.2, -0.4) {$n_{\mathrm{sink}}$};
\node[tick] at (6.6, -0.4) {$t-n_{\mathrm{rec}}\vphantom{n_{\mathrm{sink}}}$};
\node[tick] at (8.2, -0.4) {$t\vphantom{n_{\mathrm{sink}}}$};

\draw[decorate, decoration={brace, amplitude=4pt}]
  (0,0.4) -- (1.2,0.4)
  node[midway, above=4pt, note] {sink};
\draw[decorate, decoration={brace, amplitude=4pt}]
  (1.2,0.4) -- (6.6,0.4)
  node[midway, above=4pt, note] {history (grows with $t$)};
\draw[decorate, decoration={brace, amplitude=4pt}]
  (6.6,0.4) -- (8.2,0.4)
  node[midway, above=4pt, note] {recent};
\end{tikzpicture}
\caption{Cache layout at decoding position $t$. The first
$n_{\mathrm{sink}}$ tokens (attention sinks) and the most recent
$n_{\mathrm{rec}}$ tokens are held in BF16 and excluded from
quantization; the tokens in between are stored in the \method{} format.
$n_{\mathrm{sink}}$ and $n_{\mathrm{rec}}$ are fixed
($n_{\mathrm{sink}}{=}64$, $n_{\mathrm{rec}}{=}256$), so the BF16
band's share of the cache vanishes as the context grows.}
\label{fig:band}
\end{figure}
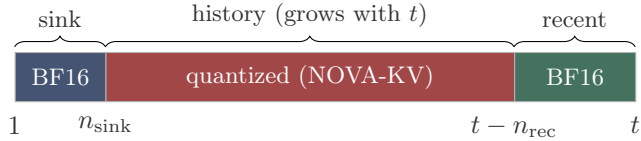
\paragraph{Benchmarks and scoring.} Table~\ref{tab:datasets} lists, for each
benchmark, the number of items scored (\#\,items), the number of samples drawn
per prompt (Smp.), the generation cap in tokens (Gen), and the per-item metric.
In the main paper, the five generation benchmarks draw $5$ samples per prompt under a $32$K cap,
while RULER NIAH uses $3$ samples with a $128$-token answer cap. The per-item
metrics are as follows. \emph{RULER NIAH}: case-insensitive substring match,
with partial credit equal to the fraction of gold needles found. \emph{GPQA-Diamond}:
the answer letter is extracted with the simple-evals multiple-choice template and
scored by exact match against the gold option (A--D). \emph{HumanEval} and
\emph{LiveCodeBench v6}: pass@1, i.e., a sample is accepted only if it passes all
reference unit tests under sandboxed execution. \emph{AIME25} and \emph{MATH500}:
symbolic-equivalence checking of the final answer with the \texttt{math-verify}
library. Benchmark accuracy is the mean of this per-item score over samples,
reported with a bootstrap $95\%$ confidence interval.

\subsection{Decode-time breakdown (Fig.~\ref{fig:memory})}
Figure~\ref{fig:memory} reports measured GPU kernel time per decoding step for Qwen3-8B served with SGLang on $8\times$H100 80\,GB SXM GPUs using tensor parallelism of degree~8. We use a fixed context length of $16384$ and a BF16 KV cache. Profiling covers decode steps 6--8, after the complete $16384$-token prompt has been written to the cache, so every profiled
step reads the full cache. Decode runs under CUDA graphs. Each bar is the median over the three steps of the summed CUDA-profiler kernel durations; CPU launch overhead and gaps between kernels are excluded.

We assign kernels to the five categories shown in Fig.~\ref{fig:memory} according to their names. \emph{Parameter load} contains the linear and MLP GEMM kernels, whose time is dominated by accessing the model parameters in this regime. \emph{Activations} contains the remaining normalization, elementwise, and activation-processing kernels. \emph{KV load} contains the fused decode-attention kernels, which read the cached keys and values and compute the query-key products, softmax, and score-value products. \emph{KV store} contains the cache-write kernels, and \emph{all-reduce} contains the NCCL collectives.

The model parameter footprint is independent of batch size: the model has $15.1$\,GB of parameters in aggregate, and the same parameters are accessed once per decoding step regardless of the number of sequences in flight. Hence, the parameter-load bar remains nearly constant across batch sizes. KV-cache traffic, in contrast, scales with both batch size and context length. Each token contributes $512$ bytes per KV head in BF16, including its key and value. Under tensor parallelism of degree~8, the per-GPU cache footprint therefore grows from $0.30$\,GB at batch~1 to $19.33$\,GB at batch~64. Correspondingly, the KV-load bar in Fig.~\ref{fig:memory} increases from $2.2$ to $9.4$\,ms, or from $51\%$ to $76\%$ of the kernel time.

The KV-load category measures the duration of the fused attention kernels rather than memory traffic in isolation. These kernels stream the KV cache while performing the attention arithmetic. Accounting for reuse across grouped-query heads, decode attention has an arithmetic intensity of approximately $4.0$\,FLOP/byte, far below the H100 BF16 roofline ridge point of approximately $295$\,FLOP/byte. The fused attention kernel is therefore bandwidth-bound once sufficient
parallelism is available. At batch~1, its measured KV traffic
corresponds to only approximately $4\%$ of the $3.35$\,TB/s peak HBM bandwidth because a single $16$K-token sequence provides insufficient parallelism to occupy all 132 SMs. At batch~64, the achieved bandwidth rises to approximately $61\%$ of peak. Thus, at large batch sizes, streaming the KV cache dominates the measured GPU kernel time.
\subsection{Calibration}

Statistics are accumulated over the calibration prompts, per (layer,
KV head), in prefill only. The matrices $\mathbf{M}_q^{\pm 1/2}$ are formed
from the symmetric eigendecomposition of $\mathbf{M}_q$ with eigenvalues floored at
$10^{-30}$ before the $\pm1/2$ power; no additive ridge is applied. The
bases $\mathbf{R}_K$, $\mathbf{R}_K^{-\top}$, $\mathbf{R}_V$ are stored in float64 (the mean $\bar{\mathbf{k}}$ in float32) and cast to
bfloat16 when loaded into the serving engine.

Key codebooks are trained by $k$-means with random initialization, for at most $25$
iterations, on the $131072$ transformed key subvectors per (layer, KV head)
($8$ calibration sequences $\times\,16384$ averaged positions) after centering,
transformation, and grouping; empty clusters are reseeded each iteration from the
points farthest from their centroid. There is one codebook per group per (layer,
KV head), $256$ centroids at $g=4$, stored in float16 and snapped to fp8 offline;
each cached subvector is an $8$-bit index. Encoding is an exact nearest-neighbor
search in the Euclidean metric of the transform domain. The grouping sorts
entries by $\log \sigma_i^2$, taking $\sigma_i^2$ from the calibration
spectrum, and deals them to the $L$ groups in turn; the procedure is
deterministic. We detail the algorithmic training steps in Algorithms~\ref{alg:transform} and \ref{alg:codebooks}.

\begin{algorithm}[t]
\caption{Query-aware transform and grouping}
\label{alg:transform}
\begin{algorithmic}[1]
\Require $M$ token positions; at each position $j$,
         the post-RoPE key $\mathbf{k}_j$ of this KV head and the
         post-RoPE queries, shared across $H$ heads, $\mathbf{q}_j^{(1)},\dots,\mathbf{q}_j^{(H)}$
\Ensure  $\mathbf{R}_K$, $\mathbf{R}_K^{-\top}$, $\bar{\mathbf{k}}$
\State $\mathbf{M}_q \gets \sum_{j=1}^{M}\sum_{h=1}^{H}
       \mathbf{q}_j^{(h)\top}\mathbf{q}_j^{(h)}$
       \Comment{summed over the $H$ shared query heads (Appendix~\ref{sec:gqa})}
\State $\bar{\mathbf{k}} \gets {1}/{M}\sum_j \mathbf{k}_j$;\quad
       $\widetilde{\mathbf{S}}_k \gets \sum_j (\mathbf{k}_j-\bar{\mathbf{k}})^{\top}
       (\mathbf{k}_j-\bar{\mathbf{k}})$
\State $\mathbf{E}\boldsymbol{\Lambda}\mathbf{E}^{\top} \gets
       \operatorname{eig}(\mathbf{M}_q^{1/2}\widetilde{\mathbf{S}}_k\mathbf{M}_q^{1/2})$,
       $\lambda_1 \ge \cdots \ge \lambda_d$
\State $\mathbf{R}_K \gets \mathbf{M}_q^{1/2}\mathbf{E}$;\quad
       $\mathbf{R}_K^{-\top} \gets \mathbf{M}_q^{-1/2}\mathbf{E}$
\State $\pi(r) \gets r \bmod L$, $r = 0,\dots,d-1$ by decreasing $\lambda$
\State $\mathbf{R}_K \gets \mathbf{R}_K[:,\pi]$;\quad
       $\mathbf{R}_K^{-\top} \gets \mathbf{R}_K^{-\top}[:,\pi]$
       \Comment{folded in}
\end{algorithmic}
\end{algorithm}
\begin{algorithm}[t]
\caption{Codebook training (per layer, KV head)}
\label{alg:codebooks}
\begin{algorithmic}[1]
\Require keys $\{\mathbf{k}_j\}$, $\mathbf{R}_K$, $\bar{\mathbf{k}}$; rate $b$
\Ensure  codebooks $\{\mathcal{C}_\ell\}_{\ell=1}^{L}$
\For{$j = 1,\dots,M$}
  \State $\mathbf{r}_j \gets (\mathbf{k}_j-\bar{\mathbf{k}})\mathbf{R}_K$
  \State $\rho_j \gets \lVert \mathbf{r}_j \rVert_2/\sqrt{d}$;\quad
         $\mathbf{r}_j \gets \rho_j^{-1}\mathbf{r}_j$
         \Comment{Normalize}
\EndFor
\For{$\ell = 1,\dots,L$}
  \State $\mathcal{C}_\ell \gets k\text{-means}\bigl(\{\mathbf{r}_j[G_\ell]\}_j,\,
         2^{gb}\bigr)$
         \Comment{equal size for every group}
\EndFor
\end{algorithmic}
\end{algorithm}

The per-token scale $\rho_j$ in Algorithm~\ref{alg:codebooks} normalizes each
transformed key before the codebooks are trained. It is not part of the model of
Theorem~2, which treats the entries as independent Gaussians with fixed
variances $\sigma_i^2$.

\subsection{Serving and evaluation}
We run in a research fork based on SGLang v0.5.10 with custom decoding kernels.
Software versions for the two hardware setups are listed in
Table~\ref{tab:hparams}. The accuracy evaluations use the Triton attention
backend for both prefill and decode. Chunked prefill uses $4096$-token
chunks on the dense models and $8192$ on GPT-OSS. RULER NIAH uses $8$ subtasks (\texttt{niah\_single\_1/2/3},
\texttt{niah\_multikey\_1/2/3}, \texttt{niah\_multivalue},
\texttt{niah\_multiquery}); item counts, samples per prompt, generation
caps, and per-item metrics are collected in Table~\ref{tab:datasets}.
Baselines use their own reference calibration procedure, under the same
sink-plus-recent band and the same chunk size as ours. \emph{Naive INT2} denotes per-token INT2 applied directly to
keys and values with no transform and no sink/recent band; it collapses to $0.0$
on every model and task (main-text Table~\ref{tab:main}). For the OSCAR-basis rows of the
transform--quantizer ablation (main-text Table~\ref{tab:ablation}), codebooks are refit on that
basis with the same group size, codebook size, and training budget as ours; the
KLT comparison (main-text Fig.~\ref{fig:transform-analysis}) instead uses entropy-coded scalar quantization.

\subsection{Decoding kernel and read path}
\label{sec:decode-kernel}
The read path is a fused Triton kernel that never multiplies by the transform.
The synthesis $\mathbf{R}_K^{-\top}$ is folded into the query once per step: the
engine forms the modified query $\tilde{\mathbf{q}} = \mathbf{q}\,\mathbf{R}_K^{-\top}$
(a single $d\times d$ product, amortized over the whole cache), so scoring a
cached key reduces to the inner product of $\tilde{\mathbf{q}}$ with the
dequantized code $\mathcal{Q}^{-}(\hat{\mathbf{k}})$. For each of the $L$ groups
the stored $8$-bit index selects a $g$-dimensional centroid from that group's
codebook; the kernel accumulates the partial logit across groups and applies the
per-token scale $\rho_j$ recovered from metadata (one multiply). Because a shared
mean shifts every logit of a query by a constant, $\bar{\mathbf{k}}$ cancels in
the softmax (Sec.~\ref{sec:lemma-offset}) and is not re-added on the read path.
On the value side the INT2 code is dequantized in the OSCAR rotation basis, whose
inverse is absorbed into the output projection, so it adds no online matmul
either. Thus the per-element read cost is a table lookup and a scale, matching
scalar dequantization. We implement the read path as a two-stage split-KV decode kernel, with gather and dequantize fused with the attention
math in stage~1 and stage-2 split-combine; the split count follows the
engine's occupancy heuristic.

\paragraph{Value quantizer.} We keep SQ for the
values: replacing the value SQ with VQ barely changes downstream
accuracy. We compared two configurations, the decorrelating transform
$\mathbf{U}_S$ alone (as suggested by Cor.~1) and $\mathbf{U}_S$
followed by a Hadamard transform (closest to the deployed SQ value
path). Both stay within $2$ points of the deployed SQ
(Table~\ref{tab:vqv-ablation}). While VQ is slightly more accurate at
a slightly lower rate ($2.16$ vs.\ $2.22$ BPE), we keep SQ for
read-path simplicity.

\begin{table}[!htbp]
\small
\centering
\caption{VQ on the values under two transforms, on
RULER NIAH with Qwen3-8B, at fixed group size, codebook size, and rate.
Mean accuracy (\%) with CIs drawn from $3$ evaluation rollouts. SQ corresponds to our deployed method. BPE=$2.22$ for the SQ method, $2.16$ for the VQ method.}
\label{tab:vqv-ablation}
\setlength{\tabcolsep}{6pt}
\renewcommand{\arraystretch}{0.9}
\begin{tabular}{lcc}
\toprule
\textbf{Configuration} & \textbf{64K} & \textbf{128K} \\
\midrule
SQ (deployed) & \ci{76.4}{0.4} & \ci{75.4}{0.1} \\
VQ+$\mathbf{U}_S$ & \ci{77.8}{0.2} & \ci{76.1}{0.5}  \\
VQ+$\mathbf{U}_S$+Hadamard & \ci{76.4}{0.3} & \ci{76.7}{0.7}\\
\bottomrule
\end{tabular}
\end{table}

\paragraph{Implementation}
\label{sec:kernel-profiling}

We consider two implementations of the read path of Sec.~\ref{sec:decode-kernel}.
Both realize the same two-stage split-KV decode attention, in which
stage~1 fuses the codebook lookup, the per-token scale, and the attention
arithmetic, and stage~2 combines the splits; they share the cache layout
and the fp8 codebooks. The \emph{Triton} kernel is the reference
implementation, with launch configurations selected offline per batch size
and geometry. The \emph{CUDA} kernel is a hand-written stage~1 with a
warp-tiled codebook gather and fp32 accumulation; it reduces the
instruction count of the gather-and-dequantize loop, which lowers the
attention read at batch 128 on Qwen3-8B from $48.9$ to $36.4$\,ms per step
(Table~\ref{tab:kernel-profile-qwen}). The CUDA kernel is the one
measured in the main-text throughput comparison (main-text Fig.~\ref{fig:throughput_bars}).

\paragraph{Per-step kernel profiling}

\begin{table}[!htbp]
\small
\centering
\caption{Per-decoding-step kernel time on Qwen3-8B, one H100, 8192-token
contexts. Entries are ms per step with the share of the step in
parentheses (\%). The BF16 pool admits at most 38 concurrent requests at
this length.}
\label{tab:kernel-profile-qwen}
\setlength{\tabcolsep}{3pt}
\renewcommand{\arraystretch}{0.9}
\begin{tabular}{lrccccc}
\toprule
\textbf{Method} & \textbf{B} & \textbf{GEMM} & \textbf{Attn} & \textbf{Quant} & \textbf{Other} & \textbf{Total} \\
\midrule
\multirow{6}{*}{BF16}
 & 1   & 5.8 (57.1) & 3.8 (37.6)  & --        & 0.5 (5.2)  & 10.1 \\
 & 8   & 5.8 (54.0) & 4.4 (40.8)  & --        & 0.6 (5.2)  & 10.8 \\
 & 16  & 5.9 (37.9) & 9.0 (58.3)  & --        & 0.6 (3.8)  & 15.4 \\
 & 32  & 5.9 (24.5) & 17.5 (73.1) & --        & 0.6 (2.4)  & 24.0 \\
 & 64  & \multicolumn{5}{c}{\emph{capacity limited}} \\
 & 128 & \multicolumn{5}{c}{\emph{capacity limited}} \\
\midrule
\multirow{6}{*}{OSCAR}
 & 1   & 6.0 (72.7) & 1.3 (16.4)  & 0.0 (0.5) & 0.9 (10.4) & 8.2 \\
 & 8   & 6.1 (62.3) & 2.6 (26.9)  & 0.0 (0.5) & 1.0 (10.3) & 9.8 \\
 & 16  & 6.2 (52.8) & 4.5 (38.2)  & 0.1 (0.5) & 1.0 (8.5)  & 11.7 \\
 & 32  & 6.2 (39.9) & 8.3 (53.2)  & 0.1 (0.4) & 1.0 (6.5)  & 15.5 \\
 & 64  & 6.2 (27.1) & 15.7 (68.1) & 0.1 (0.4) & 1.0 (4.4)  & 23.1 \\
 & 128 & 6.4 (16.2) & 31.6 (80.5) & 0.1 (0.4) & 1.1 (2.9)  & 39.3 \\
\midrule
\multirow{6}{*}{\shortstack[l]{\method{}\\(Triton)}}
 & 1   & 5.8 (67.0) & 1.7 (19.5)  & 0.3 (3.4) & 0.9 (10.2) & 8.7 \\
 & 8   & 5.9 (50.5) & 4.1 (35.3)  & 0.5 (4.5) & 1.1 (9.7)  & 11.7 \\
 & 16  & 6.0 (40.6) & 6.8 (45.8)  & 0.8 (5.8) & 1.2 (7.9)  & 14.8 \\
 & 32  & 6.0 (28.1) & 12.6 (59.5) & 1.4 (6.8) & 1.2 (5.6)  & 21.2 \\
 & 64  & 6.1 (17.4) & 24.7 (70.7) & 2.8 (8.0) & 1.4 (4.0)  & 35.0 \\
 & 128 & 6.1 (9.9)  & 48.9 (78.7) & 5.3 (8.6) & 1.7 (2.8)  & 62.1 \\
\midrule
\multirow{6}{*}{\shortstack[l]{\method{}\\(CUDA)}}
 & 1   & 5.9 (64.8) & 2.0 (21.7)  & 0.3 (3.4) & 0.9 (10.0) & 9.0 \\
 & 8   & 5.9 (55.6) & 3.0 (28.1)  & 0.6 (5.2) & 1.2 (11.1) & 10.7 \\
 & 16  & 6.0 (45.2) & 5.2 (39.3)  & 0.9 (6.6) & 1.2 (8.9)  & 13.3 \\
 & 32  & 6.0 (32.8) & 9.5 (52.1)  & 1.5 (8.2) & 1.3 (6.9)  & 18.3 \\
 & 64  & 6.1 (21.1) & 18.6 (64.4) & 2.8 (9.7) & 1.4 (4.9)  & 28.8 \\
 & 128 & 6.1 (12.4) & 36.4 (73.3) & 5.3 (10.8) & 1.8 (3.6) & 49.7 \\
\bottomrule
\end{tabular}
\end{table}

\begin{table}[!htbp]
\small
\centering
\caption{Per-decoding-step kernel time on GPT-OSS-20B, same protocol as
Table~\ref{tab:kernel-profile-qwen}.}
\label{tab:kernel-profile-gptoss}
\setlength{\tabcolsep}{3pt}
\renewcommand{\arraystretch}{0.9}
\begin{tabular}{lrccccc}
\toprule
\textbf{Method} & \textbf{B} & \textbf{GEMM} & \textbf{Attn} & \textbf{Quant} & \textbf{Other} & \textbf{Total} \\
\midrule
\multirow{6}{*}{BF16}
 & 1   & 2.6 (65.5) & 0.7 (17.8) & --        & 0.7 (16.7) & 4.0 \\
 & 8   & 4.5 (73.2) & 0.9 (14.8) & --        & 0.7 (12.0) & 6.2 \\
 & 16  & 4.9 (70.1) & 1.3 (19.1) & --        & 0.8 (10.8) & 7.1 \\
 & 32  & 5.8 (63.7) & 2.5 (27.4) & --        & 0.8 (8.9)  & 9.1 \\
 & 64  & 6.5 (53.8) & 4.8 (39.2) & --        & 0.8 (7.0)  & 12.1 \\
 & 128 & \multicolumn{5}{c}{\emph{capacity limited}} \\
\midrule
\multirow{6}{*}{OSCAR}
 & 1   & 2.8 (66.0) & 0.5 (12.4) & 0.0 (0.9) & 0.9 (20.8) & 4.2 \\
 & 8   & 4.2 (68.5) & 0.9 (14.8) & 0.0 (0.7) & 1.0 (15.9) & 6.1 \\
 & 16  & 4.6 (67.4) & 1.2 (17.6) & 0.0 (0.6) & 1.0 (14.4) & 6.9 \\
 & 32  & 5.3 (60.5) & 2.4 (27.1) & 0.0 (0.6) & 1.0 (11.8) & 8.7 \\
 & 64  & 6.2 (52.7) & 4.4 (37.6) & 0.1 (0.6) & 1.1 (9.1)  & 11.8 \\
 & 128 & 7.6 (43.5) & 8.6 (49.2) & 0.1 (0.4) & 1.2 (6.9)  & 17.5 \\
\midrule
\multirow{6}{*}{\shortstack[l]{\method{}\\(Triton)}}
 & 1   & 2.7 (63.4) & 0.5 (12.7) & 0.1 (2.4) & 0.9 (21.5) & 4.3 \\
 & 8   & 3.3 (56.8) & 1.3 (22.5) & 0.2 (2.6) & 1.1 (18.1) & 5.9 \\
 & 16  & 3.6 (54.0) & 1.7 (25.9) & 0.2 (3.4) & 1.1 (16.7) & 6.6 \\
 & 32  & 4.2 (48.6) & 3.0 (34.8) & 0.3 (3.6) & 1.1 (13.0) & 8.7 \\
 & 64  & 5.2 (40.9) & 5.8 (45.2) & 0.6 (4.4) & 1.2 (9.5)  & 12.8 \\
 & 128 & 6.6 (33.7) & 10.5 (53.7) & 1.0 (5.1) & 1.5 (7.5) & 19.6 \\
\midrule
\multirow{6}{*}{\shortstack[l]{\method{}\\(CUDA)}}
 & 1   & 2.7 (62.2) & 0.6 (14.3) & 0.1 (2.4) & 0.9 (21.0) & 4.4 \\
 & 8   & 3.3 (59.8) & 1.0 (18.1) & 0.2 (2.8) & 1.1 (19.4) & 5.6 \\
 & 16  & 3.5 (54.0) & 1.7 (25.6) & 0.2 (3.3) & 1.1 (17.1) & 6.5 \\
 & 32  & 4.2 (48.8) & 2.9 (34.1) & 0.3 (3.8) & 1.2 (13.4) & 8.6 \\
 & 64  & 5.4 (42.1) & 5.6 (44.0) & 0.5 (4.2) & 1.2 (9.7)  & 12.7 \\
 & 128 & 6.5 (32.6) & 11.0 (55.2) & 1.0 (4.8) & 1.5 (7.3) & 20.0 \\
\bottomrule
\end{tabular}
\end{table}

We adopt the per-step profiling protocol of \citet{zhou2026oscar}: one
server per arm on a single H100, 8192-token contexts, batches 1 to 128,
and GPU kernel time per decoding step split by kernel name into
\emph{GEMM} (projections, MLP, and MoE experts), \emph{Attn} (the fused
attention read, including dequantization and the split combine),
\emph{Quant} (the write path: codebook encoder, per-step query
preparation, OSCAR pack), and \emph{Other} (norms, activations, rotary,
sampling, BF16-band writes). Profiling starts in steady state, after every
warmed request is decoding, and stops after exactly 64 forward steps;
per-step figures divide device time by the step count recovered from the
trace. Tables~\ref{tab:kernel-profile-qwen} and~\ref{tab:kernel-profile-gptoss}
report the breakdown. The kernel
sum matches the wall clock per step within $1\%$ in every cell, and
unassigned kernels are below $0.2\%$. New tokens enter the
BF16 recent band and are quantized in blocks of eight steps in every
quantized arm, so the scalar arm's Quant column is smaller than a per-step
fused quantizer would show. The BF16 and OSCAR columns reproduce the
profile reported by \citet{zhou2026oscar} within about $2\%$ at every
batch size, which puts the \method{} columns on a verified scale.

\paragraph{Qwen3-8B (Table~\ref{tab:kernel-profile-qwen}).} GEMM is flat
at about $6$\,ms for every arm: parameter load is unaffected by
quantization, so the step is a fixed GEMM cost plus an attention read that
grows with batch (the regime of main-text Fig.~\ref{fig:memory}). From batch 16 on, every
2-bit arm is faster wherever BF16 runs (at batch 8 the Triton arm's
write-path cost still outweighs the small read saving), and only the 2-bit
arms serve batches 64 and 128. At batch 128, \method{} (CUDA) pays
$1.15\times$ OSCAR's attention read ($36.4$ against $31.6$\,ms) plus
$5.3$\,ms of write-path cost, for a $1.26\times$ step ($49.7$ against
$39.3$\,ms); the write-path term grows from $0.3$\,ms at batch 1 and is
the larger contribution beyond batch 64. This is the kernel-level form of
the main-text statement that decoding is on par with the scalar baseline:
the difference is confined to the read and the write path.

\paragraph{GPT-OSS-20B (Table~\ref{tab:kernel-profile-gptoss}).} At this
context no 2-bit arm beats BF16 per step: only 12 of 24 layers are
quantized, the global-layer cache at 8192 tokens is small, and at $d{=}64$
the dequantizing read is bound by instruction issue rather than bandwidth,
so the write-path cost is not repaid. The MoE GEMMs dominate the step and
vary across arms by more than the totals differ (routing depends on the
token streams). What quantization buys here is the batch itself: the 2-bit
arms serve batch 128, which the BF16 pool cannot hold. The main-text gains
for this model come from longer contexts and larger admitted batches, not
from per-step speed at 8192 tokens.

\subsection{Sampling and statistical testing}
\label{sec:stats}
Generation is stochastic. For each prompt, we draw five samples on the
generation benchmarks and three samples on RULER NIAH. The compression
parameters (transforms, groupings, codebooks, the $k$-means
initialization, and the random comparison partition of Sec.~\ref{sec:vq}) are fitted
once, with fixed seeds, and held fixed across these samples; the
repetitions therefore characterize generation and evaluation variability,
not variability from refitting the compression method. Confidence
intervals are computed by bootstrap over samples ($10000$ resamples).

McNemar's test compares two methods evaluated on the same items when
each item's outcome is binary (correct/incorrect). For each item, the
pair of outcomes falls into one of four cells: both methods correct,
both incorrect, only method A correct, or only method B correct. Items
where the two methods agree carry no information about their
difference, so the test uses only the discordant items: with $n_{01}$
items that only A solves and $n_{10}$ that only B solves, it tests the
null hypothesis that a discordant item is equally likely to fall either
way, i.e., that the two methods have the same per-item error rate. We
use the exact binomial form: under the null, $n_{01} \sim
\mathrm{Binomial}(n_{01}{+}n_{10},\, 1/2)$.

We form binary outcomes by majority voting over the
samples drawn for a given item (5 on the generation benchmarks, 3 on
RULER NIAH), and run the test over the items that both compared
methods evaluate on. For the generation benchmarks the five task sets
are combined into a single test of $992$ paired items per model; on
RULER NIAH the test is run at $128$K on Qwen3-8B over $200$ paired
items. Tests are two-sided and reported without any further correction.

\subsection{Datasets and licenses}

We evaluate on six public benchmarks. \emph{RULER NIAH} \cite{hsieh2024ruler} is a
synthetic needle-in-a-haystack retrieval suite (Apache-2.0). \emph{GPQA-Diamond}
\cite{rein2023gpqa} is graduate-level multiple-choice science (CC BY 4.0).
\emph{HumanEval} \cite{chen2021evaluating} is Python code synthesis graded by
execution (MIT). \emph{LiveCodeBench v6} \cite{jain2025livecodebench} is contamination-controlled competitive programming (benchmark released under MIT; problem statements originate from competition platforms). \emph{AIME25} consists of the 2025 AIME competition problems, copyright © Mathematical Association of America, used here for research evaluation
\cite{maa2025aime}.
\emph{MATH500} \cite{lightman2024let, hendrycks2021measuring} is a $500$-problem
subset of the MATH dataset (MIT). The models are Llama-3.1-8B (Llama 3.1 Community
License) \cite{grattafiori2024llama}, Qwen3-8B / Qwen3-4B-Thinking-2507
(Apache-2.0) \cite{yang2025qwen3}, and GPT-OSS-20B (Apache-2.0) \cite{agarwal2025gpt}. All datasets and models are used for research
evaluation consistent with their licenses.

\subsection{Throughput protocol and tuning}
\label{sec:throughput-tuning}

The decode-throughput comparison (main-text Fig.~\ref{fig:throughput_bars}) runs one server per
arm on a single H100. Prompts are distinct random token streams with no
shared prefix. Before the measured pass, a warm-up pass prefills every
request's own prompt into the cache; the measured pass then reuses these
prefixes, so all requests enter decoding together and throughput is
computed over the decode window alone. Every arm runs
its best configuration: OSCAR takes the best of eight decode-kernel tile
configurations swept per cell, \method{} (the CUDA kernel,
Sec.~\ref{sec:kernel-profiling}) takes the best of its compiled variants
per cell, and BF16 runs the
engine's dense read with an uncapped memory-fraction pool. All arms share
the same chunk size, split-KV budget, and fp32 accumulation in the
attention read.

\subsection{Configuration summary}

Tables~\ref{tab:hparams} and~\ref{tab:datasets} collect the infrastructure,
serving, method, and per-benchmark settings described above.

\begin{table}[!htbp]
\small
\centering
\caption{Infrastructure, serving, and hyperparameters.}
\label{tab:hparams}
\setlength{\tabcolsep}{4pt}
\begin{tabular}{@{}l@{\hspace{6pt}}p{0.52\columnwidth}@{}}
\toprule
Setting & Value \\
\midrule
\multicolumn{2}{l}{\emph{Software \& hardware}}\\
SGLang & research fork based on v0.5.10 \\
GPUs (accuracy) & $12\times$A100 (40GB) \\
\quad Triton / PyTorch / CUDA (A100) & 3.5.1 / 2.9.1 / 12.8 \\
GPUs (throughput, profiling) & $1\times$H100 (80GB) \\
GPUs (decode-time breakdown) & $8\times$H100 (80GB) \\
\quad Triton / PyTorch / CUDA (H100) & 3.6.0 / 2.11.0 / 13.0 \\
\midrule
\multicolumn{2}{l}{\emph{Serving}}\\
Decode attention backend & Triton \\
Chunked-prefill size & 4096 tokens (8192 on GPT-OSS) \\
Sink / recent band $n_{\mathrm{sink}},n_{\mathrm{rec}}$ & 64 / 256 (BF16) \\
KV cache dtype & INT2, quant group 128 \\
 & \quad(GPT-OSS: per-row scale, $d{=}64$) \\
\midrule
\multicolumn{2}{l}{\emph{Transform \& quantizer}}\\
Head dim $d$ / group $g$ & 128 / 4 \\
Codebook size / rate & 256 / 2 bits per coord \\
Groups $L$ / grouping & 32 / volume-equalizing \\
K / V quantizer & VQ / per-token affine INT2 \\
$k$-means init / iters & random-distinct / 25 \\
Train.\ vectors / (layer, head) & 131072 ($8{\times}16384$) \\
Transform / codebook storage & float64${\to}$bf16 / fp16${\to}$fp8 \\
\method{} regularizer & eig.\ floor $10^{-30}$ (no ridge) \\
\midrule
\multicolumn{2}{l}{\emph{Calibration \& decoding}}\\
Calibration & 198 GPQA-Diamond, prefill-only \\
Statistics dtype & float64 \\
\multicolumn{2}{l}{\quad\emph{Sampling (SGLang serving; each model's released defaults)}}\\
\quad Qwen3-8B, Qwen3-4B-Think. & $T{=}0.6$, top-$p$ $0.95$, top-$k$ $20$ \\
\quad Llama-3.1-8B-Instruct & $T{=}0.6$, top-$p$ $0.9$, top-$k$ disabled \\
\quad GPT-OSS & $T{=}1.0$, top-$p$ $1.0$, top-$k$ disabled \\
Samples per prompt & 5 (3 for NIAH) \\
Confidence intervals & paired bootstrap 95\% ($10000$ resamples) \\
\bottomrule
\end{tabular}
\end{table}
\begin{table}[!htbp]
\small
\centering
\caption{Per-benchmark evaluation settings. Samples = samples per prompt; Gen =
max generated tokens.}
\label{tab:datasets}
\setlength{\tabcolsep}{3pt}
\begin{tabular}{lrrrl}
\toprule
Dataset & \# items & Smp. & Gen & Metric \\
\midrule
RULER NIAH & 800 / 200 & 3 & 128 & substring \\
\midrule
GPQA-Diamond & 198 & 5 & 32768 & letter (A--D) \\
HumanEval & 164 & 5 & 32768 & pass@1 exec.\ \\
LiveCodeBench v6 & 100 & 5 & 32768 & pass@1 exec.\ \\
AIME25 & 30 & 5 & 32768 & math-verify \\
MATH500 & 500 & 5 & 32768 & math-verify \\
\bottomrule
\end{tabular}
\end{table}

\section{Calibration ablation}
\label{app:calib}
We vary the calibration set and refit the transform, grouping, and codebooks (Table~\ref{tab:training_data}). Reducing from $198$ to $32$ GPQA-Diamond prompts ($40$K to $8.9$K tokens) changes RULER NIAH by $1.1$
points at $64$K and $1.1$ at $128$K, both upward. Calibrating on MMLU instead, a disjoint domain with $116$K tokens, gives $74.4$ and $75.4$, within $2.0$ and $0.0$ points of the default. GPQA accuracy under MMLU
calibration is $59.3$, against $57.2$ for BF16 (Table~\ref{tab:main}).

Table~\ref{tab:calib-accuracy} extends this comparison to the five reasoning and
coding benchmarks. The effect of the calibration domain is small and model-dependent. Relative to the deployed GPQA-Diamond calibration, MMLU raises the
Qwen3-8B mean by $2.0$ points but
lowers the Qwen3-4B-Thinking mean by $0.7$. Both stay within two
points of the deployed calibration and within $2.3$ points of BF16. For consistency with \cite{zhou2026oscar}, we use GPQA-Diamond in our main evaluations.

\begin{table}[!htbp]
\small
\centering
\caption{Calibration data ablation, Qwen3-8B. We vary the calibration set and token count; GPQA accuracy does not
benefit from calibrating on GPQA. Mean accuracy (\%) $\pm$ 95\% CI over
3 rollouts.}\label{tab:training_data}
\setlength{\tabcolsep}{6pt}
\renewcommand{\arraystretch}{0.8}
\begin{tabular}{llccc}
\toprule
\textbf{Calibration}  & \textbf{Tokens} & \textbf{GPQA} & \textbf{64K} & \textbf{128K} \\
\midrule
GPQA-198 & 40000 & \ci{55.1}{1.5} &\ci{76.4}{0.4} & \ci{75.4}{0.1}\\
GPQA-32 & 8945 & \ci{56.2}{1.4} &\ci{77.5}{0.6} & \ci{76.5}{0.1}\\
MMLU & 116340 & \ci{59.3}{0.4} & \ci{74.4}{0.3} & \ci{75.4}{0.6}\\
\bottomrule
\end{tabular}
\end{table}

\begin{table}[!htbp]
\small
\centering
\caption{Calibration domain ablation on \method{}, for
Qwen3-4B-Thinking-2507 and Qwen3-8B. Compressed rows are \method{} at
$2.22$ BPE; \emph{GPQA-198} is the deployed calibration, and \emph{MMLU} a disjoint domain ($116$K tokens). Entries
are mean\,$\pm$\,95\% CI with $5$ samples per prompt; ``Drop'' is the gap in the Mean column to the BF16
reference.}
\label{tab:calib-accuracy}
\setlength{\tabcolsep}{3pt}
\renewcommand{\arraystretch}{0.8}
\begin{tabular}{llccccccrr}
\toprule
\textbf{Model} & \textbf{Calibration} & \textbf{BPE} & \textbf{GPQA} & \textbf{HumanE} & \textbf{LCB v6} & \textbf{AIME25} & \textbf{MATH500} & \textbf{Mean} & \textbf{Drop} \\
\midrule
\multirow{3}{*}{\shortstack[l]{Qwen3-4B\\-Thinking-2507}}
 & BF16 & 16.00 & \ci{64.9}{0.8} & \ci{84.8}{0.5} & \ci{59.2}{0.9} & \ci{70.0}{1.5} & \ci{97.5}{0.2} & 75.3 & -- \\
 & GPQA-198 & 2.22 & \ci{62.9}{1.3} & \ci{84.3}{0.4} & \ci{56.6}{0.6} & \ci{67.3}{1.2} & \ci{97.4}{0.2} & 73.7 & $-$1.6 \\
 & MMLU & 2.22 & \ci{64.7}{1.2} & \ci{83.7}{1.6} & \ci{53.0}{3.8} & \ci{66.7}{8.3} & \ci{97.0}{0.5} & 73.0 & $-$2.3 \\
\midrule
\multirow{3}{*}{Qwen3-8B}
 & BF16 & 16.00 & \ci{57.2}{0.9} & \ci{91.7}{0.5} & \ci{57.6}{0.7} & \ci{70.0}{1.7} & \ci{97.0}{0.1} & 74.7 & -- \\
 & GPQA-198 & 2.22 & \ci{55.1}{1.5} & \ci{92.1}{0.2} & \ci{55.6}{0.6} & \ci{64.6}{2.5} & \ci{96.9}{0.2} & 72.9 & $-$1.8 \\
 & MMLU & 2.22 & \ci{57.5}{3.3} & \ci{91.5}{1.4} & \ci{60.0}{2.3} & \ci{68.7}{6.3} & \ci{96.7}{0.5} & 74.9 & $+$0.2 \\
\bottomrule
\end{tabular}
\end{table}

\section{Partition statistics}
\begin{figure}[!htbp]
\centering
\includegraphics[width=\textwidth]{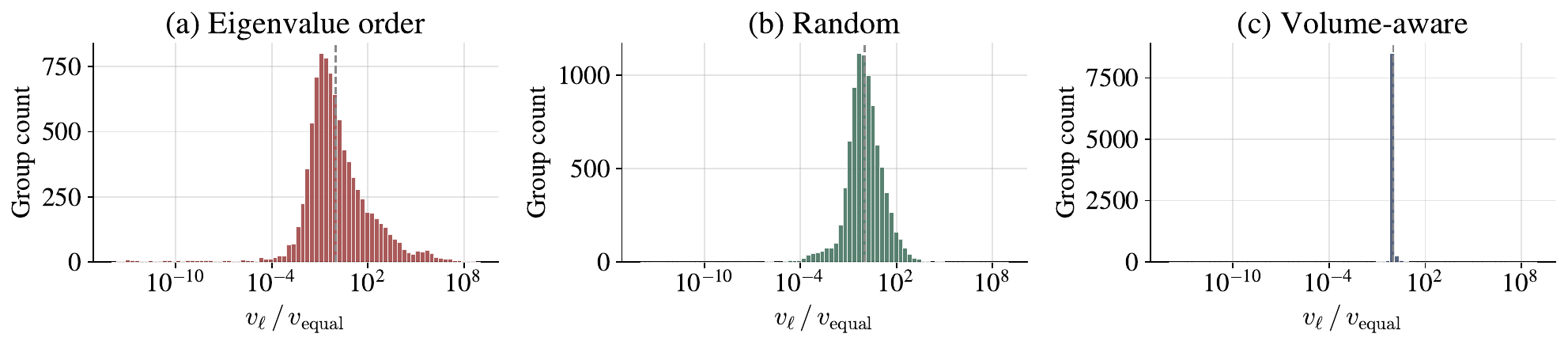}
\caption{Distribution of group volumes under the three partitions, averaged
over all (layer, KV head) pairs of Qwen3-8B at $g=4$. Each count is one group;
volumes are normalized by $v_{\mathrm{equal}}$, the common value they would take under a perfectly
volume-equalizing partition (dashed line). The horizontal axis is
logarithmic.}
\label{fig:volumes}
\end{figure}

The fixed-rate case of Theorem~2 bounds the distortion of a partition $\pi$ by
$C_g\,2^{-2b}\sum_\ell v_\ell(\pi)^{1/g} \ge D^*(b)$, with equality iff all group
volumes coincide. In theory, the gap is a property of the volume distribution
(cf.\ Sec.~\ref{app:hires}). Fig.~\ref{fig:volumes}
shows that distribution for the three partitions compared in Sec.~\ref{sec:experiments}. Volumes
are normalized by:
\begin{equation}
v_{\mathrm{equal}} = \left(\prod_{i=1}^{d}\sigma_i^2\right)^{g/d},
\end{equation}
the group volume under an exactly equalizing partition; hence, a partition
that equalizes exactly places all mass on $1$.

The three panels differ by orders of magnitude in spread. Grouping
consecutive entries in eigenvalue order (a) produces volumes spanning
roughly 20 decades, since the leading group collects the largest
$\sigma_i^2$ and the trailing group the smallest. A random partition (b) is
narrower, as each group draws entries from across the spectrum, but the
spread remains close to 3 decades. The volume-aware partition (c)
concentrates almost all groups with a small tail attributable to the greedy dealing
rule.

The ordering of the three spreads matches the ordering of the corresponding accuracies reported in Sec.~\ref{sec:experiments}, which is consistent with Theorem~2. We note that the
theorem is derived under the high-resolution model with independent Gaussian
coefficients, so the agreement is
evidence that the predicted ordering survives outside the regime of the derivation rather than a verification of the bound (cf.~Sec.~\ref{app:hires}).

\section{Analytical tools}
\subsection{Companding for non-difference distortions}
\label{sec:tc-companding}
This subsection states the high-resolution theory for distortion
measures $\rho(\cdot, \cdot)$ other than the MSE, following the companding framework of
\citet{linder1999high}, which covers locally quadratic measures whose
weight may be input-dependent. The criterion
in~\eqref{eq:qmse}, that is,
\begin{equation}
\label{eq:supp_kq}
\bigl\| \mathbf{Q}\mathbf{K}^\top -
      \mathbf{Q}\widehat{\mathbf{K}}^{\top} \bigr\|_F^2
    = \sum_{j=1}^{M} \, (\mathbf{k}_j - \widehat{\mathbf{k}}_j)\,
      \mathbf{M}_q\, (\mathbf{k}_j - \widehat{\mathbf{k}}_j)^\top,
\end{equation}
is a special case of this theory, as we will show next.

\paragraph{Locally quadratic distortions.}
Suppose $\rho(\cdot, \cdot)$ is three times continuously differentiable in its second
argument, vanishes exactly when its second argument equals the first argument, and has
positive definite second derivative at that point. Define the \emph{sensitivity
matrix}
\begin{equation}
  \bigl[\mathbf{M}(\mathbf{x})\bigr]_{ij}
  = \frac{1}{2}\,
    \frac{\partial^2 \rho(\mathbf{x},\hat{\mathbf{x}})}
         {\partial \hat{x}_i\,\partial \hat{x}_j}
    \Big|_{\hat{\mathbf{x}}=\mathbf{x}} .
  \label{eq:tc-sensitivity}
\end{equation}
Because the gradient vanishes at $\hat{\mathbf{x}} = \mathbf{x}$, a
second-order expansion gives
\begin{equation}
  \rho(\mathbf{x},\hat{\mathbf{x}})
  = (\mathbf{x}-\hat{\mathbf{x}})\,\mathbf{M}(\mathbf{x})\,
    (\mathbf{x}-\hat{\mathbf{x}})^\top
  + O\bigl(\|\mathbf{x}-\hat{\mathbf{x}}\|^3\bigr),
  \label{eq:tc-locally-quadratic}
\end{equation}
so at high resolution the distortion is a quadratic form with a possibly
input-dependent weight. An input-weighted quadratic measure
\begin{equation}
\rho(\mathbf{x},\hat{\mathbf{x}}) =
\|(\mathbf{x}-\hat{\mathbf{x}})\mathbf{W}(\mathbf{x})\|_2^2,
\end{equation}
is of this type,
with $\mathbf{M}(\mathbf{x}) =
\mathbf{W}(\mathbf{x})\mathbf{W}(\mathbf{x})^\top$. In our criterion~\eqref{eq:qmse}, the weight does not depend on the input,
$\mathbf{M}(\mathbf{x}) \equiv \mathbf{M}_q$, and the expansion
\eqref{eq:tc-locally-quadratic} is exact.

\paragraph{Companding quantizers.}
A compander applies an invertible map $h$ to the source, quantizes with a
lattice quantizer $\mathcal{Q}_{\mathcal{L}}$, and inverts:
\begin{equation}
  \mathbf{x} \;\longrightarrow\; h(\mathbf{x})
  \;\longrightarrow\; \mathcal{Q}_{\mathcal{L}}\bigl(h(\mathbf{x})\bigr)
  \;\longrightarrow\; h^{-1}
  \;\longrightarrow\; \hat{\mathbf{x}} .
  \label{eq:tc-compander}
\end{equation}
Write $\mathbf{J}(\mathbf{x})$ for the Jacobian of $h$, so that
$\mathrm{d}h = \mathrm{d}\mathbf{x}\,\mathbf{J}(\mathbf{x})$. The compander is a structured vector quantizer: all adaptation to the
source and to the distortion is carried by $h$, while the cells in the
companded domain are congruent, i.e., translates of the basic lattice
cell, identical in shape and volume. Mapped back through $h^{-1}$,
these identical cells become the non-uniform cells of the effective
quantizer in the source domain. A linear $h(\mathbf{x}) = \mathbf{x}\mathbf{R}$, with
$\mathbf{J} \equiv \mathbf{R}$ a typical transform coder. In this case, $\mathrm{d}h = \mathrm{d}\mathbf x \, \mathbf R$.

\paragraph{Asymptotic rate.}
Let $Z_g$ denote the normalized second moment of the basic cell of the
lattice. \citet{linder1999high} show that, for a source with finite differential entropy $H(\mathbf{x})$, the rate of the
compander, defined as the entropy $\mathcal{H}(D)$ of the quantizer
indices when the lattice is scaled to operate at distortion $D$,
satisfies
\begin{equation}
  \lim_{D\to 0}\Bigl( \mathcal{H}(D) + {g}/{2}\log_2 D \Bigr)
  = H(\mathbf{x})
  + \mathbb{E}\bigl[\log_2 |\det \mathbf{J}(\mathbf{x})|\bigr]
  + {g}/{2}\log_2\Bigl( g\,Z_g\,
    \mathbb{E}\bigl[\tr \boldsymbol{\Gamma}(\mathbf{x})\bigr]\Bigr),
  \label{eq:tc-lzz-thm1}
\end{equation}
where
\begin{equation}
  \boldsymbol{\Gamma}(\mathbf{x})
  = \mathbf{J}(\mathbf{x})^{-1}\,\mathbf{M}(\mathbf{x})\,
    \mathbf{J}(\mathbf{x})^{-\top},
  \label{eq:tc-gamma}
\end{equation}
measures the sensitivity as seen in the companded domain. 

\paragraph{The optimal compressor.}
Minimizing \eqref{eq:tc-lzz-thm1} over $h$ uses two inequalities. Since
$\boldsymbol{\Gamma}(\mathbf{x}) \succ 0$, the arithmetic--geometric mean
inequality gives $\tr\boldsymbol{\Gamma} \ge
g\,(\det\boldsymbol{\Gamma})^{1/g}$, with equality iff the eigenvalues of
$\boldsymbol{\Gamma}$ are all equal; Jensen's inequality then moves the
expectation inside the logarithm. Together
\cite[Thm.~2]{linder1999high}:
\begin{equation}
  \lim_{D\to 0}\Bigl( \mathcal{H}(D) + {g}/{2}\log_2 D \Bigr)
  \ge H(\mathbf{x}) + {g}/{2}\log_2\bigl(g\,Z_g\bigr)
  + {1}/{2}\,\mathbb{E}\bigl[\log_2 \det \mathbf{M}(\mathbf{x})\bigr],
  \label{eq:tc-lzz-thm2}
\end{equation}
with equality if and only if
\begin{equation}
  \mathbf{J}(\mathbf{x})\,\mathbf{J}(\mathbf{x})^\top
  = c\,\mathbf{M}(\mathbf{x})
  \quad\text{a.e., for some } c > 0 .
  \label{eq:tc-optimality}
\end{equation}
Three consequences:
\begin{enumerate}[label=(C\arabic*)]
\item \emph{The optimal compressor does not depend on the source
  distribution}, only on the distortion measure. This is the analogue, for
  locally quadratic distortions, of the fact that the optimal entropy-coded
  quantizer for MSE is uniform regardless of the source.
  \label{cons:source-free}
\item \emph{Condition \eqref{eq:tc-optimality} fixes $h$ only up to an
  orthogonal factor on the right}, since replacing $\mathbf{J}$ by
  $\mathbf{J}\mathbf{O}$ with $\mathbf{O}$ orthogonal leaves
  $\mathbf{J}\mathbf{J}^\top$ unchanged. The condition therefore constrains
  the stretch but not the rotation.
  \label{cons:orth-factor}
\item \emph{An orthogonal $h$ satisfies \eqref{eq:tc-optimality} only when
  $\mathbf{M} \propto \mathbf{I}$}, since $\mathbf{J}\mathbf{J}^\top =
  \mathbf{I}$ then forces $\mathbf{M} = c^{-1}\mathbf{I}$.
  \label{cons:orth-subopt}
\end{enumerate}
\paragraph{For the key transform.}
Take $\mathbf{M}(\mathbf{x}) \equiv \mathbf{M}_q$, constant. The compressor
is then linear, $h(\mathbf{x}) = \mathbf{x}\mathbf{R}$, and
\eqref{eq:tc-optimality} reads $\mathbf{R}\mathbf{R}^\top = c\,\mathbf{M}_q$. By
\ref{cons:orth-subopt}, orthogonal transforms
are suboptimal for the attention-aware criterion unless the queries are
isotropic; by \ref{cons:orth-factor}, Theorem~1 must resolve the
remaining orthogonal factor, which we do by minimizing the rank-$p$
reconstruction error.

\subsection{Softmax Perturbation Bound}
\label{sec:softmax-bound}

We prove that the logit error controls the attention weights (beginning of Sec.~\ref{sec:transform}).

\begin{propositionS}
\label{prop:lipschitz}
Let $\mathbf{Z}, \widehat{\mathbf{Z}} \in \mathbb{R}^{T \times T}$, $\mathbf{S} = \softmax_{\mathrm{row}}(\mathbf{Z}/\sqrt{d})$, and $\widehat{\mathbf{S}} = \softmax_{\mathrm{row}}(\widehat{\mathbf{Z}}/\sqrt{d})$. Then, $\|\mathbf{S} - \widehat{\mathbf{S}}\|_F
  \le {1}/({2\sqrt{d}})\, \|\mathbf{Z} - \widehat{\mathbf{Z}}\|_F$.
\end{propositionS}

\begin{proof}
Consider one row. The Jacobian of $\softmax$ at $\mathbf{z}$ is $\mathbf{J}(\mathbf{z}) = \operatorname{diag}(\mathbf{s}) - \mathbf{s}^\top \mathbf{s}$, $\mathbf{s} = \softmax(\mathbf{z})$, symmetric positive semidefinite with $\|\mathbf{J}(\mathbf{z})\|_2 \le 1/2$ \cite{gao2017properties}. Along the segment $\mathbf{z}(t) = (1-t)\hat{\mathbf{z}} + t\mathbf{z}$,
\begin{equation}
  \|\softmax(\mathbf{z}) - \softmax(\hat{\mathbf{z}})\|_2
   \le  \sup_{t\in[0,1]} \|\mathbf{J}(\mathbf{z}(t))\|_2\, \|\mathbf{z} - \hat{\mathbf{z}}\|_2
  \le 1/2 \, \|\mathbf{z} - \hat{\mathbf{z}}\|_2 .
\end{equation}
Applying this to each row of $\mathbf{Z}/\sqrt d$, $\widehat{\mathbf{Z}}/\sqrt d$ and summing squares over rows gives the claim.
\end{proof}

With $\mathbf{Z} = \mathbf{Q}\mathbf{K}^\top$ and $\widehat{\mathbf{Z}} = \mathbf{Q}\widehat{\mathbf{K}}^\top$, the attention-weight error is bounded by $1/(2\sqrt d) \, \|\mathbf Q\mathbf K^\top - \mathbf Q\widehat{\mathbf K}^\top\|_F$.

\subsection{Decoupling of the Output Error}
\label{sec:decoupling}

We justify the decomposition of Sec.~\ref{sec:transform}. Let $\widehat{\mathbf{S}} = \softmax_{\mathrm{row}}(\mathbf{Q}\widehat{\mathbf{K}}^\top/\sqrt d)$ be the scores computed from reconstructed keys and $\widehat{\mathbf{O}} = \widehat{\mathbf{S}}\widehat{\mathbf{V}}$ the output under joint key and value quantization.

\begin{propositionS}
\label{prop:decoupling}
\begin{equation}
  \|\mathbf{O} - \widehat{\mathbf{O}}\|_F
  \le {\|\mathbf{V}\|_2}/({2\sqrt d})\,
      \bigl\|\mathbf{Q}\mathbf{K}^\top - \mathbf{Q}\widehat{\mathbf{K}}^\top\bigr\|_F
  + \bigl\|\mathbf{S}(\mathbf{V} - \widehat{\mathbf{V}})\bigr\|_F
  + \bigl\|(\mathbf{S} - \widehat{\mathbf{S}})(\mathbf{V} - \widehat{\mathbf{V}})\bigr\|_F .
\end{equation}
\end{propositionS}

\begin{proof}
Adding and subtracting $\mathbf{S}\widehat{\mathbf{V}}$ gives the exact identity
\begin{equation}
  \mathbf{O} - \widehat{\mathbf{O}}
  = (\mathbf{S} - \widehat{\mathbf{S}})\mathbf{V}
  + \mathbf{S}(\mathbf{V} - \widehat{\mathbf{V}})
  - (\mathbf{S} - \widehat{\mathbf{S}})(\mathbf{V} - \widehat{\mathbf{V}}) .
\end{equation}
The triangle inequality, $\|(\mathbf{S}-\widehat{\mathbf{S}})\mathbf{V}\|_F \le \|\mathbf{S}-\widehat{\mathbf{S}}\|_F\|\mathbf{V}\|_2$, and Proposition~\ref{prop:lipschitz} give the claim.
\end{proof}

The first term is controlled by the key objective \eqref{eq:qmse}; the second is the value objective of Corollary~1 (Sec.~\ref{sec:proof-cor1}); the third is a product of the two quantization errors, hence second order. We do not claim joint optimality of the two-step design for the combined objective. Consistently with this decomposition, the value objective below treats the scores as computed from full-precision keys; the discrepancy from using $\widehat{\mathbf{S}}$ instead is absorbed by the second-order term.

\section{Optimal transforms}
\subsection{Lemma S1: Softmax Offset Invariance}
\label{sec:lemma-offset}
We prove in this section that adding a constant to all keys leaves the attention scores unaltered.

\begin{lemmaS}
For any $\mathbf{c} \in \mathbb{R}^{1\times d}$ and any query $\mathbf{q}$,
$\softmax_j\bigl(\mathbf{q}(\mathbf{k}_j-\mathbf{c})^\top\bigr)
= \softmax_j\bigl(\mathbf{q}\mathbf{k}_j^\top\bigr)$.
\end{lemmaS}

\begin{proof}
Replacing each key by $\mathbf{k}_j - \mathbf{c}$ shifts every logit of the query $\mathbf{q}$ by the same amount $\mathbf{q}\mathbf{c}^\top$, which cancels:
\begin{equation}
  \frac{\exp\bigl(\mathbf{q}\,(\mathbf{k}_j-\mathbf{c})^\top\bigr)}
       {\sum_{j'}\exp\bigl(\mathbf{q}\,(\mathbf{k}_{j'}-\mathbf{c})^\top\bigr)} = \frac{\exp\bigl(\mathbf{q}\,\mathbf{k}_j^\top\bigr)}
         {\sum_{j'}\exp\bigl(\mathbf{q}\,\mathbf{k}_{j'}^\top\bigr)} .
\end{equation}
\end{proof}

By Lemma~S1, replacing every key $\mathbf{k}_j$ by
$\mathbf{k}_j - \mathbf{c}$ leaves the attention unchanged, so we may
quantize the shifted keys instead and minimize
\begin{equation}
  \min_{\mathbf{c},\,\mathbf{D},\,\{\mathbf{r}_j\}}
  \sum_{j=1}^{M}
  \bigl\| \mathbf{k}_j - \mathbf{c} - \mathbf{r}_j \mathbf{D}
  \bigr\|_{\mathbf{M}_q}^{2},
  \label{eq:offset-objective}
\end{equation}
which coincides with the original objective at $\mathbf{c} = \mathbf{0}$;
optimizing over $\mathbf{c}$ therefore never increases the error, and we
choose $\mathbf{c}$ to make the keys cheaper to compress.

\subsection{Theorem 1: Key Transform}
\label{sec:proof-thm1}

We find the transform that minimizes the low-rank reconstruction
$\mathbf M_q$-MSE given training examples. We write
$\mathbf{R}^{\dagger} = (\mathbf{R}^\top\mathbf{M}_q^{-1}\mathbf{R})^{-1}
\mathbf{R}^\top\mathbf{M}_q^{-1}$ for the $\mathbf{M}_q$-weighted pseudoinverse,
so that $\mathbf{R}\mathbf{R}^{\dagger}$ is the projector onto the column space of
$\mathbf{R}$ that is orthogonal in the $\mathbf{M}_q$ inner product.

\begin{thmIrestated}
Let $\widetilde{\mathbf{k}}_j = \mathbf{k}_j - \bar{\mathbf{k}}$,
$\widetilde{\mathbf{S}}_k = \sum_j \widetilde{\mathbf{k}}_j^\top\widetilde{\mathbf{k}}_j$,
and let
\[
  \mathbf{R}^\star_{K,p}
  = \argmin_{\mathbf{R}\,\in\,\mathbb{R}^{d\times p}}\,
  \sum_{j=1}^{M}\,
  \bigl\|\,\widetilde{\mathbf{k}}_j
  - \widetilde{\mathbf{k}}_j\,\mathbf{R}\mathbf{R}^{\dagger}\bigr\|_{\mathbf{M}_q}^{2}.
\]
Then $\mathbf{R}^\star_{K,p} = \mathbf{M}_q^{1/2}\,\mathbf{E}_{1:p}$, where
$\mathbf{E}\boldsymbol{\Lambda}\mathbf{E}^\top$ is the eigendecomposition of
$\mathbf{M}_q^{1/2}\,\widetilde{\mathbf{S}}_k\,\mathbf{M}_q^{1/2}$.
\end{thmIrestated}

We prove a slightly more general statement. Rather than fixing the reconstruction
to the $\mathbf{M}_q$-orthogonal projection $\widetilde{\mathbf{k}}_j
\mathbf{R}\mathbf{R}^{\dagger}$, we leave the codes, the synthesis matrix, and a
shared offset free, and minimize the logit error~\eqref{eq:qmse} directly. In particular, we parameterize a rank-$p$ linear reconstruction with codes $\mathbf{r}_j \in \mathbb{R}^{1\times p}$, a shared synthesis matrix $\mathbf{D} \in \mathbb{R}^{p\times d}$ of full row rank, and a shared offset $\mathbf{c} \in \mathbb{R}^{1\times d}$:
\begin{equation}
  \widehat{\mathbf{k}}_j = \mathbf{r}_j\mathbf{D} + \mathbf{c},
\end{equation}
and minimize
\begin{equation}
  L(\mathbf{D}, \mathbf{c}, \{\mathbf{r}_j\})
  = \sum_{i=1}^{N}\sum_{j=1}^{M}
    \bigl( \mathbf{q}_i\mathbf{k}_j^\top - \mathbf{q}_i\widehat{\mathbf{k}}_j^\top \bigr)^2 .
  \label{eq:supp-key-objective}
\end{equation}
We show our result in 5 steps.

\paragraph{Step 1: reorder to expose the query second moment.}
With $\boldsymbol{\delta}_j = \mathbf{k}_j - \mathbf{c} - \mathbf{r}_j\mathbf{D}$,
\begin{equation}
\sum_{i,j} \bigl(\mathbf{q}_i\boldsymbol{\delta}_j^\top\bigr)^2
    = \sum_{j=1}^{M} \boldsymbol{\delta}_j
      \Bigl( \sum_{i=1}^{N} \mathbf{q}_i^\top\mathbf{q}_i \Bigr)
      \boldsymbol{\delta}_j^\top
    = \sum_{j=1}^{M} \|\boldsymbol{\delta}_j\|_{\mathbf{M}_q}^2 ,
  \label{eq:supp-key-weighted}
\end{equation}
with $\mathbf{M}_q = \mathbf{Q}^\top\mathbf{Q}$, which is \eqref{eq:qmse}. Invertibility of $\mathbf{M}_q$ requires the calibration queries to span $\mathbb{R}^d$ (in particular $N \ge d$), which holds for typical calibration sets; in practice we enforce positive-definiteness as in Appendix~\ref{app:calib-stats}.

\paragraph{Step 2: optimal codes for fixed $\mathbf{D}, \mathbf{c}$.}
The loss decouples across $j$. Setting the gradient of $\|\boldsymbol{\delta}_j\|_{\mathbf{M}_q}^2$ in $\mathbf{r}_j$ to zero,
\begin{equation}
  \mathbf{r}_j\,\mathbf{D}\mathbf{M}_q\mathbf{D}^\top
  = (\mathbf{k}_j - \mathbf{c})\,\mathbf{M}_q\mathbf{D}^\top
  \;\Longrightarrow\;
  \mathbf{r}_j^* = (\mathbf{k}_j - \mathbf{c})\,\mathbf{M}_q\mathbf{D}^\top
                   \bigl(\mathbf{D}\mathbf{M}_q\mathbf{D}^\top\bigr)^{-1},
  \label{eq:supp-key-codes}
\end{equation}
a generalized least-squares projection onto the row space of $\mathbf{D}$, measured in the $\mathbf{M}_q$ inner product.

\paragraph{Step 3: optimal offset.}
For fixed $\mathbf{D}$ and codes, $\nabla_{\mathbf{c}} L = -2\sum_j \boldsymbol{\delta}_j \mathbf{M}_q = \mathbf{0}$. Since $\mathbf{M}_q \succ 0$, the condition reduces to $\sum_j \boldsymbol{\delta}_j = \mathbf{0}$, i.e., $\mathbf{c}^* = \bar{\mathbf{k}} - \bar{\mathbf{r}}\mathbf{D}$ with $\bar{\mathbf{r}} = 1/M\sum_j \mathbf{r}_j$. Any choice satisfying this equation is optimal; we pick the one that centers the codes, $\bar{\mathbf{r}} = \mathbf{0}$, so that no rate is spent encoding the code mean $\mathbf{c}^* = \bar{\mathbf{k}}$. We write $\tilde{\mathbf{k}}_j = \mathbf{k}_j - \bar{\mathbf{k}}$.

\paragraph{Step 4: optimal synthesis.}
Substituting \eqref{eq:supp-key-codes},
$\mathbf{r}_j^*\mathbf{D} = \tilde{\mathbf{k}}_j \mathbf{P}$ with
$\mathbf{P} = \mathbf{M}_q\mathbf{D}^\top(\mathbf{D}\mathbf{M}_q\mathbf{D}^\top)^{-1}\mathbf{D}$,
the oblique projector onto the row space of $\mathbf{D}$ that is orthogonal in the $\mathbf{M}_q$ inner product ($\mathbf{P}^2 = \mathbf{P}$, and $\mathbf{P}\mathbf{M}_q$ is symmetric). Using these two identities,
$(\mathbf{I}-\mathbf{P})\mathbf{M}_q(\mathbf{I}-\mathbf{P})^\top = \mathbf{M}_q - \mathbf{P}\mathbf{M}_q$, and with $\widetilde{\mathbf{S}}_k = \sum_j \tilde{\mathbf{k}}_j^\top\tilde{\mathbf{k}}_j$,
\begin{equation}
  L^*(\mathbf{D})
  = \tr\bigl[\mathbf{M}_q\widetilde{\mathbf{S}}_k\bigr]
  - \tr\bigl[\mathbf{P}\mathbf{M}_q\widetilde{\mathbf{S}}_k\bigr].
\end{equation}
The first term is constant, so we maximize
$T(\mathbf{D}) = \tr\bigl[\mathbf{M}_q\mathbf{D}^\top(\mathbf{D}\mathbf{M}_q\mathbf{D}^\top)^{-1}\mathbf{D}\,\mathbf{M}_q\widetilde{\mathbf{S}}_k\bigr]$.

\paragraph{Step 5: change of variables and Ky Fan.}
Let $\mathbf{B} = \mathbf{D}\mathbf{M}_q^{1/2} \in \mathbb{R}^{p\times d}$, so $\mathbf{D}\mathbf{M}_q\mathbf{D}^\top = \mathbf{B}\mathbf{B}^\top$ and $\mathbf{M}_q\mathbf{D}^\top = \mathbf{M}_q^{1/2}\mathbf{B}^\top$. Then  $T = \tr\bigl[\boldsymbol{\Pi}_{\mathbf{B}}\,\mathbf{A}\bigr],
  \quad
  \mathbf{A} = \mathbf{M}_q^{1/2}\widetilde{\mathbf{S}}_k\mathbf{M}_q^{1/2},
  \quad
  \boldsymbol{\Pi}_{\mathbf{B}} = \mathbf{B}^\top(\mathbf{B}\mathbf{B}^\top)^{-1}\mathbf{B}$, where $\boldsymbol{\Pi}_{\mathbf{B}}$ is a rank-$p$ orthogonal projector and $\mathbf{A} \succeq 0$. Maximizing $\tr[\boldsymbol{\Pi}\mathbf{A}]$ over rank-$p$ orthogonal projectors is the classical PCA problem, solved by the projector onto the top-$p$ eigenvectors of $\mathbf{A}$ with maximum $\sum_{i\le p}\lambda_i(\mathbf{A})$ \cite{fan1949theorem}. With $\mathbf{A} = \mathbf{E}\boldsymbol{\Lambda}\mathbf{E}^\top$ and $\mathbf{E}_{1:p}$ the top-$p$ eigenvectors, take $\mathbf{B}^* = \mathbf{E}_{1:p}^\top$, i.e.,
\begin{equation}
  \mathbf{D}^* = \mathbf{E}_{1:p}^\top\mathbf{M}_q^{-1/2} = \mathbf{R}^\dagger .
\end{equation}
Then $\mathbf{D}^*\mathbf{M}_q\mathbf{D}^{*\top} = \mathbf{I}_p$, and \eqref{eq:supp-key-codes} simplifies to
\begin{equation}
  \mathbf{r}_j^* = \tilde{\mathbf{k}}_j\,\mathbf{M}_q\mathbf{D}^{*\top}
                 = \tilde{\mathbf{k}}_j\,\mathbf{M}_q^{1/2}\mathbf{E}_{1:p}
                 = \tilde{\mathbf{k}}_j\,\mathbf{R},
\end{equation}
with $\mathbf{R} = \mathbf{M}_q^{1/2}\mathbf{E}_{1:p} = \mathbf{R}^\star_{K,p}$. The
reconstruction is therefore $\widehat{\mathbf{k}}_j = \widetilde{\mathbf{k}}_j
\mathbf{R}\mathbf{R}^{\dagger} + \bar{\mathbf{k}}$, the projected form of the
theorem, and the optimal cost is $\sum_{i>p}\lambda_i(\mathbf{A})$. For $p = d$,
$\mathbf{E}$ is orthogonal and $\mathbf{R}\mathbf{R}^\dagger = \mathbf{I}$. \hfill$\blacksquare$

\begin{remark}
If $\mathbf{M}_q \propto \mathbf{I}$, the procedure reduces to ordinary mean-centered PCA on the keys.
\end{remark}

\begin{remark}
For $p = d$, the columns of $\mathbf{R}$ are not orthonormal in general; they satisfy $\mathbf{R}\mathbf{R}^\top = \mathbf{M}_q \not=\mathbf I$ whenever $\mathbf M_q \not \propto \mathbf I$. For $p < d$, $\mathbf{R}\mathbf{R}^\top = \mathbf{M}_q^{1/2}\mathbf{E}_{1:p}\mathbf{E}_{1:p}^\top\mathbf{M}_q^{1/2} \neq \mathbf{M}_q$.
\end{remark}

\subsection{Proposition 1: Generalized Parseval Relation}

Throughout, $\mathbf{R}_K = \mathbf{R}^\star_{K,d}$ denotes the full-rank ($p = d$)
transform of Theorem~1, which is invertible.

\begin{propIrestated}
Let $\mathbf r$ and $\hat{\mathbf{r}}$ be any two vectors, and $\mathbf k =
\mathbf{r}\mathbf{R}_K^{-1}$ and $\hat{\mathbf k} = \hat{\mathbf{r}}\mathbf{R}_K^{-1}$.
Then, $\bigl\| \mathbf{r} - \hat{\mathbf{r}} \bigr\|_2^2
  = \bigl\| \mathbf k - \hat{\mathbf{k}} \bigr\|_{\mathbf M_q}^2$.
\end{propIrestated}

\begin{proof}
With $\mathbf{r} - \hat{\mathbf{r}} = (\mathbf{k} - \hat{\mathbf{k}})\mathbf{R}$,
\begin{equation}
  \|\mathbf{r} - \hat{\mathbf{r}}\|_2^2
  = (\mathbf{k} - \hat{\mathbf{k}})\,\mathbf{R}\mathbf{R}^\top(\mathbf{k} - \hat{\mathbf{k}})^\top ,
\end{equation}
and $\mathbf{R}\mathbf{R}^\top = \mathbf{M}_q^{1/2}\mathbf{E}\mathbf{E}^\top\mathbf{M}_q^{1/2} = \mathbf{M}_q$, since $\mathbf{E}$ is orthogonal for $p = d$.
\end{proof}

\subsection{Corollary 1: Value Transform}
\label{sec:proof-cor1}

\begin{corIrestated}
Let $\mathbf{o}_i = \mathbf{s}_i \mathbf{V}$ denote the attention
outputs on calibration data, and
$\mathbf{E}\boldsymbol{\Lambda}\mathbf{E}^\top$ the eigendecomposition
of $\mathbf{M}_o = \mathbf{V}^\top \mathbf{M}_s \mathbf{V}$. The minimizer of
$\| \mathbf{S}\mathbf{V} - \mathbf{S}\widehat{\mathbf{V}} \|_F^2$ over low-rank
approximations is $\mathbf{R}_V = \mathbf{E}_{1:p}$, with codes
$\mathbf{s}_j = \mathbf{v}_j \mathbf{R}_V$ and reconstruction
$\widehat{\mathbf{v}}_j = \mathbf{s}_j \mathbf{R}_V^\top$.
\end{corIrestated}

Let $\boldsymbol{\alpha}_i \in \mathbb{R}^{1\times M}$ denote the $i$th row of $\mathbf{S}$, so the attention output for calibration query $i$ is $\mathbf{o}_i = \boldsymbol{\alpha}_i\mathbf{V} \in \mathbb{R}^{1\times d}$. We parameterize a rank-$p$ linear reconstruction with codes $\mathbf{s}_j \in \mathbb{R}^{1\times p}$ and synthesis $\mathbf{D}_V \in \mathbb{R}^{p\times d}$ of full row rank, $\widehat{\mathbf{v}}_j = \mathbf{s}_j\mathbf{D}_V$, and minimize the output distortion
\begin{equation}
  L(\mathbf{D}_V, \{\mathbf{s}_j\})
  = \sum_{i=1}^{N} \bigl\| \boldsymbol{\alpha}_i\mathbf{V} - \boldsymbol{\alpha}_i\widehat{\mathbf{V}} \bigr\|_2^2 .
  \label{eq:supp-value-objective}
\end{equation}

\paragraph{Step 1: reorder to expose the score second moment.}
With $\boldsymbol{\delta}_j = \mathbf{v}_j - \mathbf{s}_j\mathbf{D}_V$ and $\boldsymbol{\Delta} \in \mathbb{R}^{M\times d}$ stacking the $\boldsymbol{\delta}_j$ as rows,
\begin{equation}
\sum_{i} \Bigl\| \sum_{j} \alpha_{ij}\,\boldsymbol{\delta}_j \Bigr\|_2^2
    = \sum_{j,j'} (\mathbf{M}_s)_{jj'}\,\boldsymbol{\delta}_j\boldsymbol{\delta}_{j'}^\top
    = \tr\bigl[\boldsymbol{\Delta}^\top\mathbf{M}_s\boldsymbol{\Delta}\bigr],
\end{equation}
with $\mathbf{M}_s = \mathbf{S}^\top\mathbf{S} = \sum_i \boldsymbol{\alpha}_i^\top\boldsymbol{\alpha}_i$. Attention reads values through score-weighted sums (mixing tokens).

\paragraph{Step 2: optimal codes for fixed $\mathbf{D}_V$.}
Writing $\boldsymbol{\Delta} = \mathbf{V} - \mathbf{S}_c\mathbf{D}_V$ with $\mathbf{S}_c \in \mathbb{R}^{M\times p}$ stacking the codes gives
\begin{equation}
  \mathbf{M}_s\,\mathbf{S}_c\,\mathbf{D}_V\mathbf{D}_V^\top
  = \mathbf{M}_s\,\mathbf{V}\mathbf{D}_V^\top .
\end{equation}
A solution independent of $\mathbf{M}_s$ is
\begin{equation}
  \mathbf{S}_c^* = \mathbf{V}\mathbf{D}_V^\top\bigl(\mathbf{D}_V\mathbf{D}_V^\top\bigr)^{-1},
  \label{eq:supp-value-codes}
\end{equation}
the ordinary Euclidean projection of each $\mathbf{v}_j$ onto the row space of $\mathbf{D}_V$: the weight lives on the token axis, orthogonal to the head-space axis along which the projection acts, so the projection geometry is standard.\footnote{If $\mathbf{M}_s$ is invertible, \eqref{eq:supp-value-codes} is the unique solution. Invertibility is not required for what follows: the optimal basis depends on $\mathbf{M}_s$ only through the $d\times d$ matrix $\mathbf{M}_o = \mathbf{V}^\top\mathbf{M}_s\mathbf{V}$, which has rank $\min(d, N)$ regardless of the cache length $M$.}

\paragraph{Step 3: optimal synthesis.}
Substituting, $\mathbf{S}_c^*\mathbf{D}_V = \mathbf{V}\mathbf{P}$ with $\mathbf{P} = \mathbf{D}_V^\top(\mathbf{D}_V\mathbf{D}_V^\top)^{-1}\mathbf{D}_V$, a standard rank-$p$ orthogonal projector. Then
\begin{equation}
\tr\bigl[(\mathbf{I}-\mathbf{P})\,\mathbf{V}^\top\mathbf{M}_s\mathbf{V}\,(\mathbf{I}-\mathbf{P})\bigr]
  = \tr[\mathbf{M}_o] - \tr[\mathbf{P}\mathbf{M}_o],
\end{equation}
with $\mathbf{M}_o = \mathbf{V}^\top\mathbf{M}_s\mathbf{V}$. Maximizing $\tr[\mathbf{P}\mathbf{M}_o]$ over rank-$p$ orthogonal projectors is again the PCA problem of Sec.~\ref{sec:proof-thm1}, Step 5: with $\mathbf{M}_o = \mathbf{E}\boldsymbol{\Lambda}\mathbf{E}^\top$, the maximizer is $\mathbf{P}^* = \mathbf{E}_{1:p}\mathbf{E}_{1:p}^\top$, i.e., $\mathbf{D}_V^* = \mathbf{E}_{1:p}^\top$ and, by \eqref{eq:supp-value-codes}, $\mathbf{s}_j^* = \mathbf{v}_j\mathbf{E}_{1:p}$. This is Corollary~1 with $\mathbf{R}_V = \mathbf{E}_{1:p}$; since $\mathbf{E}_{1:p}$ has orthonormal columns, the synthesis is the transpose and no normalization is needed.

\paragraph{Second moment of the attention outputs.}
Although $\mathbf{M}_o$ was defined through the $M\times M$ matrix $\mathbf{M}_s$, it is the empirical second moment of the attention outputs:
\begin{equation}
\mathbf{V}^\top\Bigl(\sum_i \boldsymbol{\alpha}_i^\top\boldsymbol{\alpha}_i\Bigr)\mathbf{V}
  = \sum_{i=1}^{N} (\boldsymbol{\alpha}_i\mathbf{V})^\top(\boldsymbol{\alpha}_i\mathbf{V})
  = \sum_{i=1}^{N} \mathbf{o}_i^\top\mathbf{o}_i .
\end{equation}
We never form $\mathbf{M}_s$ on the token axis, which grows with context length; we accumulate outer products of attention-output vectors in head space, a $d\times d$ matrix. \hfill$\blacksquare$

\begin{remark}
No offset appears in Corollary~1, but centering the values is also free: since $\mathbf{S}\mathbf{1} = \mathbf{1}$, a shared value offset satisfies $\mathbf{S}(\mathbf{V} - \mathbf{1}\mathbf{c}_V) = \mathbf{S}\mathbf{V} - \mathbf{1}\mathbf{c}_V$, so it can be subtracted before quantization and added back to the output exactly.
\end{remark}

\subsection{Query Statistics under Grouped-Query Attention}
\label{sec:gqa}

In grouped-query attention (GQA) \cite{ainslie2023gqa}, one KV head is shared by $H$ query heads with queries $\mathbf{Q}^{(1)},\dots,\mathbf{Q}^{(H)}$. The logit error for the shared keys is
\begin{equation}
  \sum_{h=1}^{H}\bigl\|\mathbf{Q}^{(h)}\mathbf{K}^\top - \mathbf{Q}^{(h)}\widehat{\mathbf{K}}^\top\bigr\|_F^2
  =  \sum_{j=1}^{M} (\mathbf{k}_j - \widehat{\mathbf{k}}_j)
    \Bigl(\sum_{h=1}^{H}\mathbf{M}_q^{(h)}\Bigr)
    (\mathbf{k}_j - \widehat{\mathbf{k}}_j)^\top ,
\end{equation}
with $\mathbf{M}_q^{(h)} = \mathbf{Q}^{(h)\top}\mathbf{Q}^{(h)}$. This has the same form as \eqref{eq:qmse} with $\mathbf{M}_q = \sum_h \mathbf{M}_q^{(h)}$, so we accumulate the query second moment over the heads sharing each KV head.

\section{Quantizer design: grouping and rate allocation}
\subsection{Theorem 2: Grouping}
\label{sec:proof-thm2}

\begin{thmIIrestated}
Under the model~\eqref{eq:zador}, for any partition $\pi$ the allocation
minimizing $D(\pi,\{b_\ell\})$ subject to $\sum_\ell b_\ell = Lb$ is
\[
  b^{*}_\ell(\pi) \;=\; b \;+\; \frac{1}{2g}\,
  \log_2\Big({v_\ell(\pi)}\big/{\prod_m v_m(\pi)^{1/L}}\Big),
\]
which spends more bits on groups of larger volume. The resulting optimal
distortion, $D^{*}(b) = C_gL\,2^{-2b}\bigl(\prod_{i=1}^d\sigma_i^2\bigr)^{1/d}$,
is the same for every $\pi$, assuming that rates can be the arbitrary real values
from~\eqref{eq:alloc}.
\end{thmIIrestated}

We prove this statement together with the fixed-rate claim that follows
Theorem~2 in the main text: under $b_\ell \equiv b$, the distortion of a partition
$\pi$ is $C_g\,2^{-2b}\sum_\ell v_\ell(\pi)^{1/g} \ge D^{*}(b)$, with equality iff
$v_1(\pi) = \cdots = v_L(\pi)$.

By the model~\eqref{eq:zador} the total distortion is $D(\pi,\{b_\ell\}) =
C_g\sum_\ell 2^{-2b_\ell} v_\ell(\pi)^{1/g}$; the constant $C_g>0$ is common to all
groups, so it leaves the minimizing $\pi$ and $\{b_\ell\}$ unchanged and only
rescales the optimal value. We therefore solve
\begin{equation}
  \min_{\pi}\;\min_{\{b_\ell\}:\,\sum_\ell b_\ell = Lb}\;
  \sum_{\ell=1}^{L} 2^{-2b_\ell}\, v_\ell(\pi)^{1/g} ,
  \label{eq:supp-grouping}
\end{equation}
and restore $C_g$ in the optimal distortion at the end.

\paragraph{Optimal allocation and distortion.}
For fixed $\pi$, each term $2^{-2b_\ell}v_\ell^{1/g}$ is convex in $b_\ell$ and the constraint is affine, so first-order conditions are sufficient. Stationarity of the Lagrangian
$\mathcal{L} = \sum_\ell 2^{-2b_\ell}v_\ell^{1/g} + \mu(\sum_\ell b_\ell - Lb)$
gives $2\ln 2 \cdot 2^{-2b_\ell}v_\ell^{1/g} = \mu$ for all $\ell$: the optimal allocation equalizes per-group distortions. Solving and enforcing the constraint,
\begin{equation}
  b_\ell^* = b + \frac{1}{2g}\log_2 \bigg(
  {v_\ell(\pi)}\big/\prod_m v^{1/L}_m(\pi)\bigg),
\end{equation}
which is \eqref{eq:alloc}. Substituting back, every term equals $2^{-2b}\bigl(\prod_m v_m(\pi)\bigr)^{1/(Lg)}$, so the inner minimum is
\begin{equation}
  L\,2^{-2b}\Bigl(\prod_\ell v_\ell(\pi)\Bigr)^{1/(Lg)}
  = L\,2^{-2b}\Bigl(\prod_{i=1}^{d}\sigma_i^2\Bigr)^{1/d},
\end{equation}
where the last equality uses $\prod_\ell v_\ell(\pi) = \prod_{i=1}^{d}\sigma_i^2$ for every partition and $Lg = d$. The value is independent of $\pi$; restoring the common factor, the optimal distortion is $D^* = C_g\,L\,2^{-2b}\bigl(\prod_{i=1}^{d}\sigma_i^2\bigr)^{1/d}$, proving the theorem.

\paragraph{Fixed rates.}
Under $b_\ell \equiv b$, the distortion is $C_g\,2^{-2b}\sum_\ell v_\ell(\pi)^{1/g}$. By the arithmetic--geometric mean inequality applied to the nonnegative numbers $v_\ell^{1/g}$,
\begin{equation}
  \sum_{\ell=1}^{L} v_\ell^{1/g}
  \ge L\, \Bigl(\prod_\ell v_\ell^{1/g}\Bigr)^{1/L}
  = L\, \Bigl(\prod_{i=1}^{d}\sigma_i^2\Bigr)^{1/d},
\end{equation}
with equality iff all $v_\ell$ coincide, so the fixed-rate distortion is at least $C_g\,L\,2^{-2b}(\prod_i\sigma_i^2)^{1/d} = D^*$. The equal-rate distortion attains the unconstrained optimum on volume-equalizing partitions. \hfill$\blacksquare$

\subsection{Low-Rate Version of the Grouping Bound}
\label{sec:low-rate}

The remark in Sec.~\ref{sec:vq} removes the high-resolution assumption for independent Gaussian groups. Let $\mathbf{r}_{G_\ell} \in \mathbb{R}^{1\times g}$ have independent Gaussian entries with variances $\{\sigma_i^2\}_{i\in G_\ell}$, and let $\mathcal{Q}_\ell$ be \emph{any} quantizer with $gb_\ell$ bits. The rate of a fixed-rate quantizer upper-bounds the mutual information between source and reconstruction, so the distortion is bounded below by the distortion-rate function of the source. For an independent Gaussian vector, reverse water-filling \cite{berger1971rate} gives, when the entry-level distortion $D/g$ is below $\min_{i\in G_\ell}\sigma_i^2$,
\begin{equation}
  R(D) = \sum_{i\in G_\ell} 1/2 \, \log_2\frac{g\, \sigma_i^2}{D}
       = \frac{g}{2}\log_2\frac{g\, v_\ell^{1/g}}{D} .
\end{equation}
Setting $R(D) \le gb_\ell$ and solving for $D$,
\begin{equation}
  \mathbb{E}\bigl\|\mathbf{r}_{G_\ell} - \mathcal{Q}_\ell(\mathbf{r}_{G_\ell})\bigr\|_2^2
  \ge g\, v_\ell^{1/g}\, 2^{-2b_\ell} .
\end{equation}
The dependence on the partition is only through $v_\ell^{1/g}$, so the fixed-rate argument of Sec.~\ref{sec:proof-thm2} applies to this lower bound.

\section{High-Resolution Model: Empirical Test}
\label{app:hires}

High-resolution quantization theory is often used as a design principle for compressors deployed at finite rates \cite{sullivan1998rate}. The accuracy of the approximation depends on whether the quantization cells are small relative to variations in the source density and distortion measure, which can hold even at modest nominal rates. We emphasize, in any case, that our results are asymptotic rather than finite-rate guarantees. We test, at and around the deployed rate, the specific predictions that determine our design.

We focus on Theorem~2, which rests on the high-resolution distortion model~\eqref{eq:zador},
\begin{equation}
D_\ell(b)
=
C_g \,2^{-2b}\,v_\ell(\pi)^{1/g},
\label{eq:hires-model}
\end{equation}
while the method operates at $b=2$, where the asymptotics are not
guaranteed to hold. For a group $G_\ell$ of $g$ transform entries, a partition $\pi$,
and a rate $b$, we define the empirical distortion as
\begin{equation}
\widehat{D}_\ell(b)
= 
\mathbb{E}_{\mathrm{test}}
\left[
  \left\|
    \mathbf{r}_{G_\ell}
    -
    \mathcal{Q}_\ell(\mathbf{r}_{G_\ell})
  \right\|_2^2
\right].
\label{eq:hires-empirical-distortion}
\end{equation}
We estimate this quantity using a $2^{gb}$-entry $k$-means codebook
trained on training-split coefficients and evaluated on held-out
coefficients. We vary $b$ unless otherwise stated and set $g=4$ (matching the deployed method). Here
$\mathbb{E}_{\mathrm{test}}$ denotes the empirical mean over held-out
coefficients, and $\widehat{D}_{\ell,\mathrm{train}}$ denotes the same
estimate on the training split.

We evaluate~\eqref{eq:hires-model} using parameters estimated from the
training data. The group volume is
\begin{equation}
v_\ell(\pi)
=
\prod_{i\in G_\ell}\sigma_i^2,
\label{eq:hires-group-volume}
\end{equation}
where the coordinate variances $\sigma_i^2$ are estimated on the
training split. Let $b_0=2.5$. We estimate the common constant once as
\begin{equation}
\widehat{C}_g
=
2^{2b_0}
\operatorname*{median}_{\ell}
\left({
  \widehat{D}_{\ell,\mathrm{train}}(b_0)
}/{
  v_\ell(\pi)^{1/g}
}\right),
\label{eq:hires-constant}
\end{equation}
where the median is taken over the groups of the equalizing partition
and all tested (layer, KV head) pairs. We hold $\widehat{C}_g$ fixed across all held-out rates, groups, and
partitions. In the empirical comparison, we therefore evaluate
\begin{equation}
D_\ell(b)
=
\widehat{C}_g\,2^{-2b}\,v_\ell(\pi)^{1/g}.
\label{eq:hires-fitted-model}
\end{equation}

We also report the volume-normalized empirical distortion ${\widehat{D}_\ell(b)}/{v_\ell(\pi)^{1/g}}$, which the model predicts to be approximately common across groups at a
given rate:
\begin{equation}
{D_\ell(b)}/{v_\ell(\pi)^{1/g}}
=
\widehat{C}_g\,2^{-2b}.
\label{eq:hires-normalized-model}
\end{equation}
Here the hat denotes an empirical estimate and is unrelated to the
reconstruction notation of Appendix~\ref{app:notation}.

We test three predictions. First, at each rate, the volume-normalized
distortions ${\widehat{D}_\ell(b)}/{v_\ell(\pi)^{1/g}}$ should approximately agree across groups. Second, across
rates, the volume-normalized distortion should follow the scaling
$2^{-2b}$. Third, at the deployed rate $b=2$, the measured total
distortion $\sum_{\ell=1}^{L}\widehat{D}_\ell(b)$ and the theoretical
partition criterion $\sum_{\ell=1}^{L}v_\ell(\pi)^{1/g}$ should order partitions in the same way.

\subsection{Setup}
We used Qwen3-8B, Qwen3-4B-Thinking-2507, and Llama-3.1-8B post-RoPE keys at three (layer, KV head) pairs across the model depth (layer 1, head 0; layer 17, head 3; layer 35, head 7 for the Qwen3 models, and layer 1, head 0; layer 15, head 3; layer 31, head 7 for the 32-layer Llama-3.1-8B), giving $3 \times L = 96$ groups per partition per rate at $d=128$, $g=4$. Coefficients are $\mathbf{r} = (\mathbf{k}-\bar{\mathbf{k}})\mathbf{R}_K$ with the deployed transform. Per-token RMS normalization was off for this study; it isolates the unnormalized transform coefficients analyzed in Theorem~2. Throughout, $b$ denotes the rate for the key branch in bits per entry. For Qwen3 we pooled three model captures (GPQA-198, MMLU, GPQA-32), each 8 sequences of 16384 tokens; examples 0-5 of each capture formed the training pool, from which we retained $2^{18}-1 = 262143$ coefficient vectors, and examples 6-7 formed the held-out set (98304 tokens); the training and held-out sets contain disjoint sequences from each capture. For Llama-3.1-8B, we used a single LongBench capture of 80 variable-length prompts split 0-23 / 24-47 (262144 training, 98304 held-out vectors). Codebooks were initialized by sampling $2^{gb}$ codewords uniformly without replacement from the training vectors, followed by 25 Lloyd iterations; empty cells were reseeded to the training points worst served by the current codebook. We consider rates $b \in \{1.0, 1.5, 2.0, 2.5\}$, i.e.\ $2^{gb} \in \{16, 64, 256, 1024\}$ codewords. In testing the partitions, we consider the three methods in Appendix~\ref{app:notation}: equalizing (deployed), random, and variance-sorted.

\subsection{Agreement across groups}
At each rate, dividing each group's distortion by $v_\ell^{1/g}$ removes
most of the variation across the 96 groups  (Fig.~\ref{fig:hires}a--b). The volume term
accounts for the bulk of the per-group differences, which is the component of the model that the partition argument of Theorem~2 relies on.

\subsection{Rate dependence}
We fit the generalized rate law $2^{-\alpha b}$ to the volume-normalized held-out distortions at $b\in\{1,1.5,2,2.5\}$ and obtain $\alpha=1.816$ for Qwen3-8B, $\alpha=1.828$ for Qwen3-4B-Thinking-2507, and $\alpha=1.878$ for Llama-3.1-8B, compared with the high-resolution value $\alpha=2$. If we compute the exponent between adjacent rates, which is more relevant for an asymptotic law, we obtain $1.796$, $1.813$, and $1.840$ on held-out data for the intervals $1{\to}1.5$, $1.5{\to}2$, and $2{\to}2.5$ on Qwen3-8B, increasing monotonically over the tested range, a trend consistent with approaching the asymptotic value $2$. The fitted constants agree closely across the three models, $\widehat{C}_g=1.518$, $1.507$, and $1.489$ respectively, as expected for a quantity that should depend only on $g$ and the source shape. As we observe in Fig.~\ref{fig:hires}a,~\ref{fig:hires}d, and~\ref{fig:hires}g, with $\widehat{C}_g$ estimated from training data at the top rate ($b=2.5$), the measured and predicted curves are closest at that rate and separate toward lower rates.

Because these distortions are measured on held-out coefficients, the
reported deviation includes both finite-rate effects and any residual
suboptimality of the trained vector quantizers; we do not attempt to
separate these effects.

\subsection{Partition comparison}
At $b=2$, Table~\ref{tab:hires-grouping} shows the measured and predicted per-group distortion over the 96 groups for each partition. The predicted criterion ranks the three partitions in the same order as the held-out distortion, with the deployed equalizing partition performing best. The model is quantitatively accurate for the equalizing and random partitions, but overestimates the distortion of the variance-sorted partition by $81.5\%$.

\begin{table}[!htbp]
\small
\centering
\caption{Partition comparison at $b=2$ for Qwen3-8B. Measured and predicted mean
per-group distortion over 96 groups. Predictions use~\eqref{eq:zador} with
$\widehat{C}_g=1.518$.}
\label{tab:hires-grouping}
\begin{tabular}{lccc}
\toprule
Partition & Measured & Predicted & Pred./meas. \\
\midrule
Equalizing     & 0.739 & 0.803 & 1.087 \\
Random         & 0.941 & 1.060 & 1.125 \\
Variance-sorted & 2.610 & 4.740 & 1.815 \\
\bottomrule
\end{tabular}
\end{table}

\paragraph{Residual dependence.}
We observe that our model provides worse predictions for the variance-sorted partition. Beyond existing finite-rate approximation errors, we argue that this effect is consistent with a stronger violation of the independent-Gaussian assumption. The transform decorrelates the entry at second order but
does not make them independent; the sorted partition places adjacent entries together, and those groups exhibit
substantially greater dependence as measured by Shannon total correlation
$\mathrm{TC} = \sum_j h_j - h_{\text{joint}}$
(Table~\ref{tab:hires-tc}, Kozachenko--Leonenko $k$-NN joint-entropy
estimates and $m$-spacing marginal estimates on training coefficients). We remark that total correlation indicates a
violation of the independence assumption but does not quantify its effect on distortion. Finite-sample entropy estimates may yield slightly negative estimated total correlations, explaining why a reported mean can be below the corresponding median.

\begin{table}[!htbp]
\small
\centering
\caption{Total correlation per 4-coordinate group for Qwen3-8B on training
coefficients.}
\label{tab:hires-tc}
\begin{tabular}{lccc}
\toprule
Grouping & Mean & Median & q90 \\
\midrule
Equalizing & 0.0290 & 0.0647 & 0.1077 \\
Random     & 0.0419 & 0.0643 & 0.1488 \\
Sorted     & 0.2284 & 0.0906 & 0.4488\\
\bottomrule
\end{tabular}
\end{table}

\begin{figure}[!htbp]
\centering
\includegraphics[width=\textwidth]{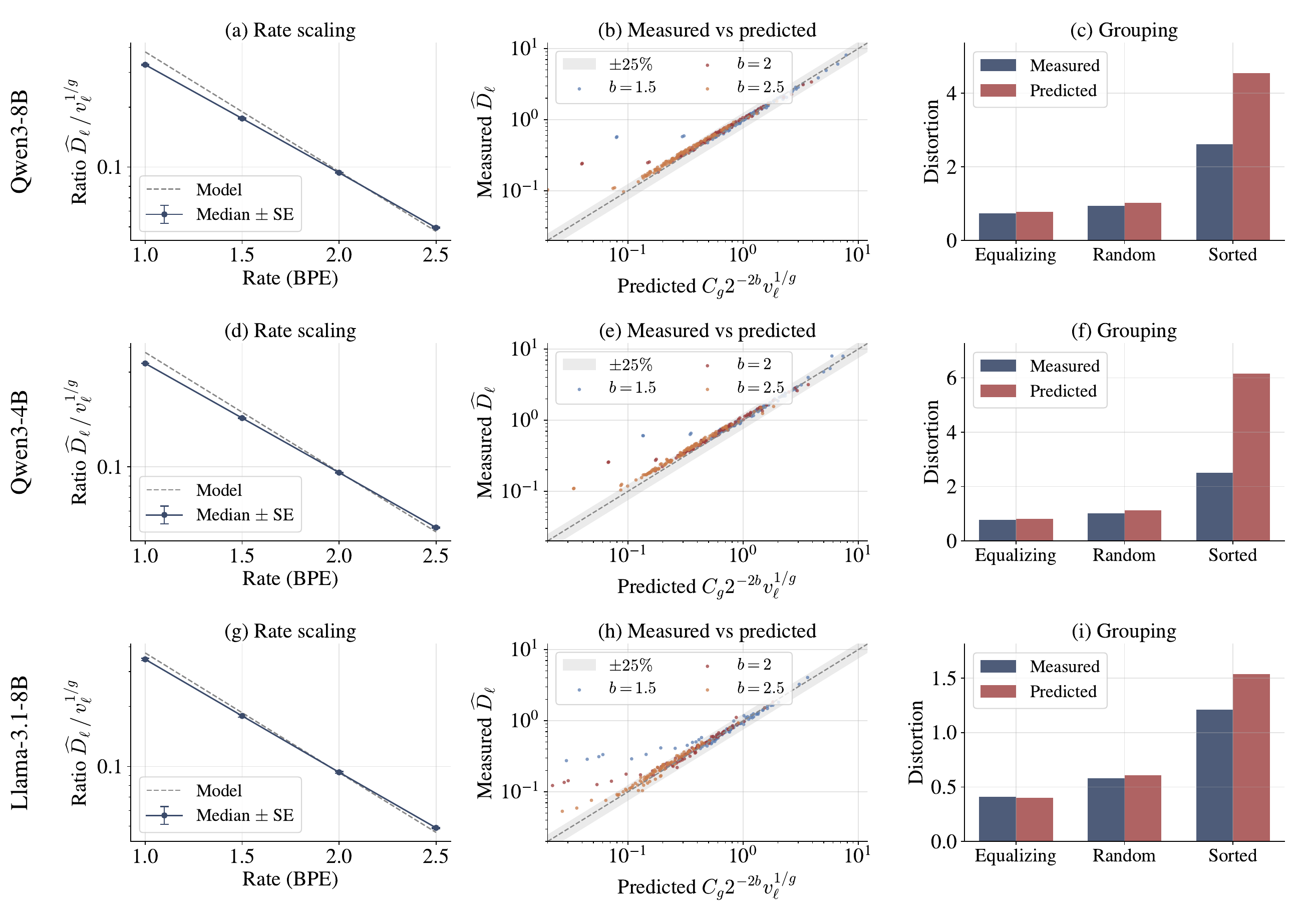}
\caption{High-resolution model test, 96 groups
pooled over three (layer, KV head) pairs of Qwen3-8B (a-c), Qwen3-4B-Thinking-2507 (d-f), and Llama-3.1-8B (g-i). (a, d)~Median
volume-normalized distortion $\widehat{D}_\ell / v_\ell^{1/g}$ against rate, with
the model curve $\widehat{C}_g 2^{-2b}$ tied to the median at $b=2.5$; the curves
are closest at the anchor and separate at low rate, where the local exponent is
below 2. (b, e)~Per-group measured distortion $\widehat{D}_\ell$ against the model $D_\ell$ at
$b \in \{1.5, 2, 2.5\}$; the band is $\pm 25\%$ around the identity.
(c, f)~Measured and predicted per-group distortion for the three
partitions at $b=2$; the ordering is preserved.}
\label{fig:hires}
\end{figure}

\section{Further evaluations}
RULER NIAH measures exact-match retrieval. To test whether the
comparison transfers beyond that setting, we evaluate on two suites: LongBench
(Sec.~\ref{sec:supp_longbench}), a multitask suite of QA,
summarization, and code completion, and LooGLE
(Sec.~\ref{sec:supp_loogle}), long-document QA and summarization with
longer inputs and a larger item count. 
\subsection{LongBench}
\label{sec:supp_longbench}
We evaluate on LongBench \cite{bai2023longbench} tasks (Qasper,
QMSum, MultiNews, TREC, TriviaQA, SAMSum, LCC, RepoBench-P) using the full test sets (200 examples per task; 500 for LCC
and RepoBench-P) and greedy decoding (default for this benchmark),
so runs are deterministic. Prompt KV entries are quantized during
prefill at 2.28 BPE for OSCAR and 2.22 BPE for \method{}, while KV entries generated during decoding remain in BF16;
layer~0 is kept in full precision for every arm. OSCAR and \method{}
differ only on the key quantizer: both share the same SQ value path. \method{} transforms and codebooks are calibrated on the same GPQA-Diamond subset we used in the main paper.

\paragraph{Qwen3-8B} We let both OSCAR and NOVA-KV use the sink-plus-recent band ($n_{\mathrm{sink}}{=}64$,
$n_{\mathrm{rec}}{=}256$). We also ablate the effect of this band: the
half-band rows reduce it to $n_{\mathrm{sink}}{=}32$,
$n_{\mathrm{rec}}{=}128$ for both methods. Since decoding is
deterministic, we assess separability with a paired bootstrap over
items (10000 resamples, percentile 95\% CIs); CIs on the mean row are
computed by resampling items within each task and averaging the eight
per-task means with equal weight, so each task counts 1/8 regardless
of test-set size. Per-task intervals are uncorrected for multiple
comparisons and are provided as descriptive analyses; the equal-weight
average is the primary comparison. 

\begin{table}[!htbp]
\small
\centering
\caption{LongBench scores per task for Qwen3-8B. Full band is
$(n_{\mathrm{sink}}, n_{\text{rec}}) = (64, 256)$; half band is $(32, 128)$.
BF16 is band-independent. Best compressed method per task and band
setting in \textbf{bold}.}
\label{tab:longbench_scores}
\begin{tabular}{l c c >{\columncolor{oursred}}c c >{\columncolor{oursred}}c}
\toprule
 & & \multicolumn{2}{c}{Full band} & \multicolumn{2}{c}{Half band} \\
\cmidrule(lr){3-4} \cmidrule(lr){5-6}
Task & BF16 & OSCAR & \cellcolor{oursred}\method{} & OSCAR & \cellcolor{oursred}\method{} \\
\midrule
Qasper      & 44.27 & \textbf{43.72} & 39.65          & \textbf{41.63} & 39.96 \\
QMSum       & 24.67 & 23.28          & \textbf{23.72} & 22.87          & \textbf{23.58} \\
MultiNews   & 24.92 & \textbf{24.40} & 23.70          & \textbf{24.29} & 23.34 \\
TREC        & 41.50 & \textbf{49.50} & 46.00          & 36.50          & \textbf{50.00} \\
TriviaQA    & 90.53 & 87.85          & \textbf{89.74} & 88.35          & \textbf{89.74} \\
SAMSum      & 40.00 & \textbf{39.75} & 37.73          & \textbf{39.78} & 37.66 \\
LCC         & 64.90 & 60.79          & \textbf{63.52} & 60.83          & \textbf{62.85} \\
RepoBench-P & 60.05 & 50.93          & \textbf{61.75} & 49.87          & \textbf{61.38} \\
\midrule
Mean        & 48.86 & 47.53          & \textbf{48.23} & 45.52          & \textbf{48.56} \\
\bottomrule
\end{tabular}
\end{table}
\begin{table}[!htbp]
\small
\centering
\caption{Paired bootstrap comparison, gap between \method{} and OSCAR, per task and
averaged. Entries are $\Delta$ with 95\% CI;
$^{*}$ marks CIs excluding zero.}
\label{tab:longbench_bootstrap}
\begin{tabular}{lcc}
\toprule
Task & Full band & Half band \\
\midrule
Qasper      & $-4.07\;[-7.63, -0.61]^{*}$  & $-1.67\;[-5.73, 2.30]$ \\
QMSum       & $+0.44\;[-0.31, 1.17]$       & $+0.71\;[0.04, 1.40]^{*}$ \\
MultiNews   & $-0.70\;[-1.12, -0.28]^{*}$  & $-0.95\;[-1.39, -0.50]^{*}$ \\
TREC        & $-3.50\;[-10.50, 3.50]$      & $+13.50\;[7.00, 20.50]^{*}$ \\
TriviaQA    & $+1.90\;[-0.77, 4.83]$       & $+1.40\;[-1.90, 4.75]$ \\
SAMSum      & $-2.02\;[-3.36, -0.70]^{*}$  & $-2.12\;[-3.57, -0.71]^{*}$ \\
LCC         & $+2.73\;[1.05, 4.41]^{*}$    & $+2.02\;[0.25, 3.80]^{*}$ \\
RepoBench-P & $+10.82\;[8.51, 13.19]^{*}$  & $+11.51\;[9.13, 13.86]^{*}$ \\
\midrule
Average      & $+0.70\;[-0.40, 1.79]$       & $+3.05\;[1.90, 4.20]^{*}$ \\
\bottomrule
\end{tabular}
\end{table}

\begin{table}[!htbp]
\small
\centering
\caption{Paired bootstrap comparison against the BF16 reference, using equal-weight average over the eight tasks. $^{*}$ marks CIs
excluding zero.}
\label{tab:longbench_bf16}
\begin{tabular}{lcc}
\toprule
Arm & Average $\Delta$ & 95\% CI \\
\midrule
OSCAR, full band   & $-1.33$ & $[-2.33, -0.29]^{*}$ \\
\method{}, full band & $-0.63$ & $[-1.47, 0.26]$ \\
OSCAR, half band   & $-3.34$ & $[-4.40, -2.24]^{*}$ \\
\method{}, half band & $-0.29$ & $[-1.26, 0.65]$ \\
\bottomrule
\end{tabular}
\end{table}

LongBench contexts are shorter than the regime where the methods
separate on RULER NIAH, so the differences in average score are small. First, against BF16
(Table~\ref{tab:longbench_bf16}), we detect no significant difference between \method{} and the uncompressed reference at both band
settings, while OSCAR is significantly below it at both. Second, at
the full band, \method{}'s $+0.70$ advantage over OSCAR in average score is not
separable from zero (Table~\ref{tab:longbench_bootstrap}). Third, at the half band the average-score gap becomes separable
($+3.05$): halving the BF16 band costs OSCAR $2.0$ points on average
but leaves \method{} almost unchanged, indicating that \method{} is less dependent on protecting sink and recent tokens.

\paragraph{GPT-OSS-20B.} To check whether the OSCAR failure on GPT-OSS
(cf.~Sec.~\ref{sec:gptoss-dissect}) is specific to RULER NIAH's
exact-match retrieval, we separately evaluate the same eight LongBench
tasks on GPT-OSS-20B using the same configuration we used for Qwen3-8B (in this case, we use the full protection band, $n_{\mathrm{sink}}=64, n_{\text{rec}} = 256$). 
 Table~\ref{tab:longbench-gptoss}
reports every task. OSCAR collapses on \emph{every} task (question
answering, summarization, classification, and code completion alike) for a mean of $7.51$ against BF16's $36.76$, an $80\%$ relative drop. \method{} stays within $7\%$ of BF16 ($34.24$), consistent
with the Qwen3-8B result in Table~\ref{tab:longbench_scores}.

\begin{table}[!htbp]
\small
\centering
\caption{LongBench \cite{bai2023longbench} scores per task for GPT-OSS-20B. Best
compressed method per task in \textbf{bold}.}
\label{tab:longbench-gptoss}
\begin{tabular}{lcc>{\columncolor{oursred}}c}
\toprule
Task & BF16 & OSCAR & \cellcolor{oursred}\method{} \\
\midrule
Qasper      & 36.92 & 4.84  & \cellcolor{oursred}\textbf{31.58} \\
QMSum       & 19.62 & 6.37  & \cellcolor{oursred}\textbf{19.32} \\
MultiNews   & 20.64 & 10.04 & \cellcolor{oursred}\textbf{18.89} \\
TREC        & 17.00 & 2.50  & \cellcolor{oursred}\textbf{11.00} \\
TriviaQA    & 76.58 & 2.67  & \cellcolor{oursred}\textbf{72.07} \\
SAMSum      & 28.18 & 3.22  & \cellcolor{oursred}\textbf{25.75} \\
LCC         & 50.82 & 17.57 & \cellcolor{oursred}\textbf{51.69} \\
RepoBench-P & 44.29 & 12.84 & \cellcolor{oursred}\textbf{43.63} \\
\midrule
Mean        & 36.76 & 7.51  & \cellcolor{oursred}\textbf{34.24} \\
\bottomrule
\end{tabular}
\end{table}

\subsection{LooGLE}
\label{sec:supp_loogle}

We evaluate on LooGLE \cite{li2023loogle} with Qwen3-8B ($1101$
items), scored by ROUGE-L, under the same serving configuration and
band as the LongBench evaluation (Sec.~\ref{sec:supp_longbench});
decoding is greedy, so runs are deterministic. Separability is
assessed with the same paired bootstrap over items (10000 resamples,
percentile 95\% CIs).

\begin{table}[!htbp]
\small
\centering
\caption{LooGLE ROUGE-L for Qwen3-8B over $n{=}1101$ items, with
paired bootstrap differences (95\% CIs; $^{*}$ marks CIs excluding
zero).}
\label{tab:loogle}
\begin{tabular}{lc}
\toprule
Arm & ROUGE-L \\
\midrule
BF16               & 31.50 \\
OSCAR              & 26.38 \\
\rowcolor{oursred}
\method{}            & 28.81 \\
\midrule
\multicolumn{2}{l}{\emph{Paired differences}} \\
\method{} $-$ OSCAR  & $+2.43\;[1.07, 3.78]^{*}$ \\
\method{} $-$ BF16   & $-2.69\;[-3.79, -1.62]^{*}$ \\
OSCAR $-$ BF16     & $-5.12\;[-6.60, -3.69]^{*}$ \\
\bottomrule
\end{tabular}
\end{table}

\paragraph{Discussion.}
\method{} scores $2.43$ ROUGE-L above OSCAR, separable from zero, and
$2.69$ below the BF16 reference, roughly half of OSCAR's $5.12$-point
degradation. Unlike on LongBench, where \method{} was statistically
indistinguishable from BF16, the larger item count in this dataset resolves all pairs. LooGLE complements RULER NIAH by measuring graded long-document QA and
summarization rather than exact-match retrieval.

\section{GPT-OSS-20B: Quantizing a Hybrid-Attention Cache}
\label{sec:gptoss-hybrid}

As attention computation and KV-cache size become the dominant cost of
long-context generation, recent architectures increasingly adopt hybrid
designs in which only a fraction of the layers attend globally over the full
context, and the remaining layers use a cheaper mechanism. The mechanism
varies across models: sliding-window attention
\cite{gemmateam2025gemma3,agarwal2025gpt}, gated delta networks \cite{yang2025gated},
or state-space layers \cite{lieber2024jamba}. The consequence for KV
quantization is the same in each case: only the global layers' cache grows
with context, so the cache that a quantizer sees is smaller, the information
it holds is denser, and there is less redundancy, limiting potential compression gains. GPT-OSS-20B \cite{agarwal2025gpt} is such a model and differs from
Llama-3.1-8B and Qwen3-8B in every property relevant to KV
quantization (Table~\ref{tab:gptoss-arch}): only 12 of its 24 layers attend
globally (the other 12 see a 128-token window), heads are half as wide, keys
are not bounded by QK-norm, every head carries a learned sink logit that
competes with tokens in the softmax, and the feed-forward blocks route each
token to 4 of 32 experts. Its growing per-token cache is $5$--$6\times$
smaller than the dense models'. We quantize only the 12 global-attention
layers; the sliding-window layers stay BF16, since their cache is bounded by
the window and does not grow with context. In this setting, \method{} is the
only 2-bit method (among the methods we tested) that remains effective (Table~\ref{tab:niah}). This
section describes the serving changes the architecture requires, analyzes
why the scalar baselines fail, and measures the rate at which scalar
quantization recovers.

\begin{table}[!htbp]
\small
\centering
\caption{Architecture properties relevant to KV quantization. ``KV
bytes/token'' counts the BF16 cache that grows with context (global-attention
layers).}
\label{tab:gptoss-arch}
\setlength{\tabcolsep}{3pt}
\renewcommand{\arraystretch}{0.9}
\begin{tabular}{lccc}
\toprule
 & \textbf{Llama-3.1} & \textbf{Qwen3-8B} & \textbf{GPT-OSS} \\
\midrule
layers (global / total) & 32 / 32 & 36 / 36 & 12 / 24 \\
sliding window & -- & -- & 128 (12 layers) \\
head dim.\ & 128 & 128 & 64 \\
heads (Q / KV) & 32 / 8 & 32 / 8 & 64 / 8 \\
QK-norm & no & yes & no \\
attention sink & no & no & learned \\
feed-forward & dense & dense & MoE (4 of 32) \\
KV bytes/token & 131072 & 147456 & 24576 \\
\bottomrule
\end{tabular}
\end{table}

\subsection{Attention sinks in the quantized serving path}
\label{sec:gptoss-sinks}

The learned sink enters the softmax as an extra logit per head that
contributes to the normalizer but has no value vector. Three changes are
required in the quantized read path.

\paragraph{Split-KV normalization.} In the two-stage decode kernel
(Sec.~\ref{sec:decode-kernel}), each KV split reduces independently. The sink
term $\exp(s_h - m)$ is added to the normalizer exactly once, after the
cross-split combine in stage~2, not per split.

\paragraph{Mean centering and offset invariance.} \method{} stores
centered keys and drops the constant $\mathbf{q}^\top\bar{\mathbf{k}}$
from the logits, which is harmless under a softmax
(Lemma~S1, Sec.~\ref{sec:lemma-offset}). With a sink the invariance fails:
token logits carry the offset but the learned sink logit does not, so the
sink's softmax share would be rescaled by
$\exp(\mathbf{q}^\top\bar{\mathbf{k}})$. The kernel therefore shifts the sink
logit by the same per-query constant (one $d$-dimensional dot product per KV
head per step), which restores exact equivalence.

\paragraph{Chunked prefill.} New chunks are scored against the dequantized
prefix by a flash-attention kernel \cite{dao2022flashattention} with no sink support; its output and
log-sum-exp are rescaled afterwards to fold the sink into the normalizer.

\subsection{On the OSCAR and QuaRot failures}
\label{sec:gptoss-dissect}

On GPT-OSS-20B, OSCAR collapses on retrieval (Table~\ref{tab:niah}) and
QuaRot collapses on every task, while \method{} stays within
a few points of BF16 at the same rate. We test whether this is an
integration artifact or a property of the model.

\paragraph{Serving stack.} We run the model under plain HuggingFace and overwrite
the cache with the same codec the pool applies. On a 10-needle greedy retrieval probe at 1K context,
GPT-OSS drops from 4 needles (its BF16 control on this harness) to 1 with
keys and values quantized, and to 0 with only keys quantized; Llama-3.1-8B retrieves
10 of 10 both in BF16 and with OSCAR under the identical procedure. Since the
value-side OSCAR codec is shared with \method{}, which is unaffected, the
failure is specific to SQ of the keys.

\paragraph{Reconstruction fidelity.} We measure the
OSCAR codec on real captured keys and queries in four bases
(Table~\ref{tab:gptoss-fidelity}). GPT-OSS keys \emph{reconstruct better}
than Llama's in the served bases ($9.1$ vs.\ $8.2$\,dB in the
Hadamard basis, $8.8$ vs.\ $8.1$\,dB in the calibrated one), and the rows
are not outlier-dominated there: the mean excess kurtosis of a key row is
$-0.3$ on GPT-OSS and $-0.1$ on Llama in the calibrated basis (versus
$20.8$ and $7.2$ in the identity basis, where both models are heavy-tailed
and both SNRs collapse). The same key error nevertheless produces
$3.5$--$3.9\times$ the relative attention-logit error on GPT-OSS, and the
ratio persists across all four bases.
The architecture concentrates all long-range information in 12 narrow-head
layers, so each cached element carries more task information and the logits
tolerate less noise: one scale per (token, head) at 2 bits
sits above GPT-OSS's noise tolerance and below Llama's. The MoE
feed-forward adds an amplification path: the perturbed attention output
enters the router, whose top-$k$ selection is a discrete function of the
hidden state. In a teacher-forced probe (24 decode steps at 1K context,
production codec), KV quantization changed the top-4 expert set at $70\%$
of (token, layer) router decisions, an error mode a dense feed-forward does
not exhibit.

\begin{table}[!htbp]
\small
\centering
\caption{OSCAR key quantization on GPT-OSS-20B (G) and Llama-3.1-8B (L):
reconstruction SNR of the keys and relative attention-logit error
$\|\mathbf{q}^\top(\mathbf{K}-\hat{\mathbf{K}})\|/\|\mathbf{q}^\top\mathbf{K}\|$,
per basis. Post-RoPE keys and queries captured from GPQA prompts (4 prompts
for the SNR statistics, 8 for the logit error), averaged over all quantized
layers and KV heads; per-(token, head) scale with the production percentile
clip ($0.96$). ``Calibrated'' is each model's served rotation.}
\label{tab:gptoss-fidelity}
\setlength{\tabcolsep}{4pt}
\renewcommand{\arraystretch}{0.9}
\begin{tabular}{lcccc}
\toprule
 & \multicolumn{2}{c}{\textbf{Key SNR (dB)}} & \multicolumn{2}{c}{\textbf{Logit error}} \\
\cmidrule(lr){2-3}\cmidrule(lr){4-5}
\textbf{Basis} & G & L & G & L \\
\midrule
calibrated & 8.8 & 8.1 & 0.64 & 0.17 \\
Hadamard   & 9.1 & 8.2 & 0.59 & 0.16 \\
random     & --  & --  & 0.69 & 0.17 \\
identity   & 3.0 & 5.8 & 0.89 & 0.44 \\
\bottomrule
\end{tabular}
\end{table}

\paragraph{QuaRot.} The QuaRot rows in
Table~\ref{tab:niah} already include the protective BF16
sink-plus-recent band and still sit at $0.0$. Without the band, i.e.\
unclipped min--max QuaRot-INT2 on every token from position~0, GPT-OSS produces
degenerate output at every prompt length we probed, from $90$ to
$8000$ tokens (greedy completions; served and reproduced in the
HuggingFace simulation), while Llama only degrades under the identical
procedure. On this model the unclipped scalar codec has no usable
operating point, with or without the band.

\subsection{Scalar quantization recovers at higher rates}
\label{sec:gptoss-recovery}

If information density is the correct explanation, a modest amount of
additional rate should restore the scalar codec. We simulate the OSCAR codec
bit-exactly on the write path of a BF16 server (values are dequantized
before storage; no memory is saved, and every token is quantized) and vary
the rate (Table~\ref{tab:gptoss-bits}). One additional bit takes the codec
from collapse to functional, and 4 bits matches BF16 on retrieval. Finer
scale granularity at 2 bits does not substitute, despite a higher rate than
INT3: the deficit is resolution per coordinate, not outlier range. \method{}
reaches the required fidelity at $2.22$ BPE through vector quantization
(NIAH-8K $89.6$, Table~\ref{tab:niah}), consistent with the main-text
result that, among the techniques we tested, it is the only 2-bit method that remains effective on this
model.

\begin{table}[!htbp]
\small
\centering
\caption{Scalar (OSCAR-codec) rate ladder on GPT-OSS-20B, simulated
bit-exactly in serving. ``Bits'' is the payload width of every stored key
and value coordinate; ``groups'' is the number of min--max scale/zero-point
pairs per (token, head) row of $d{=}64$ coordinates (1 is the served
configuration). Single greedy rollout; NIAH-8K on the 200-prompt
balanced subset (25 per subtask), MATH-500 on 100 prompts. BPE charges BF16
scale and zero-point per group ($b + 0.5g$).}
\label{tab:gptoss-bits}
\setlength{\tabcolsep}{4pt}
\renewcommand{\arraystretch}{0.9}
\begin{tabular}{ccccc}
\toprule
\textbf{Bits} & \textbf{Groups} & \textbf{BPE} & \textbf{NIAH-8K} & \textbf{MATH-500} \\
\midrule
2 & 1 & 2.5 & 0.8 & 4 \\
2 & 4 & 4.0 & 27.0 & 15 \\
3 & 1 & 3.5 & 82.8 & 73 \\
4 & 1 & 4.5 & 98.3 & 85 \\
\multicolumn{2}{c}{BF16} & 16 & 99.0 & 92 \\
\bottomrule
\end{tabular}
\end{table}

\section{A Chunked-Prefill Dequantization Leak}
\label{sec:gptoss-chunked-prefill}

Long prompts are prefilled in chunks rather than in one pass, bounding
peak prefill memory and letting the scheduler interleave a long prompt's
prefill with other requests' decoding. In a mixed-precision KV pool, each
chunk's keys and values are quantized as they are written, so a later
chunk of the same prompt attends over earlier chunks' dequantized rows
instead of their original values -- a leak that single-pass prefill does
not have.

The leak can be closed by keeping an exact copy of a request's own
in-flight rows until its final chunk completes, after which cached
content matches the quantized-on-write baseline exactly. This is not
free: while a request's own prefill is in progress, its in-flight rows
occupy close to BF16-sized cache rather than the compressed footprint,
trading part of the memory saving for correctness during that window.

Table~\ref{tab:chunked-prefill} reports RULER NIAH with the leak present
and closed. It is real and grows with context (Llama OSCAR: $+4.5$ at
$8$K, $+16.6$ at $128$K), and closing it recovers most of the gap for
\method{} and for Llama's OSCAR. GPT-OSS's OSCAR does not recover:
$2$--$11$ across all five lengths versus a floor of $0$ with the leak
present. The leak is real but does not explain the collapse
(Sec.~\ref{sec:gptoss-dissect}).

\begin{table}[!htbp]
\small
\centering
\caption{RULER NIAH accuracy with the chunked-prefill leak present
$\to$ closed. One seed, $25$ prompts per NIAH subtask ($200$ total).
OSCAR and \method{} on GPT-OSS-20B, and, as a control, Llama-3.1-8B.}
\label{tab:chunked-prefill}
\setlength{\tabcolsep}{4pt}
\renewcommand{\arraystretch}{0.9}
\begin{tabular}{lcccc}
\toprule
 & \multicolumn{2}{c}{\textbf{GPT-OSS-20B}} & \multicolumn{2}{c}{\textbf{Llama-3.1-8B}} \\
\cmidrule(lr){2-3}\cmidrule(lr){4-5}
Length & OSCAR & \method{} & OSCAR & \method{} \\
\midrule
8K   & $0.50 \to 11.38$ & $89.75 \to 92.75$ & $82.25 \to 86.75$ & $92.88 \to 95.00$ \\
16K  & $0.00 \to 4.50$  & $79.62 \to 92.75$ & $78.62 \to 86.12$ & $91.75 \to 94.50$ \\
32K  & $0.00 \to 2.38$  & $80.00 \to 90.62$ & $80.50 \to 86.38$ & $88.62 \to 93.25$ \\
64K  & $0.00 \to 4.25$  & $69.25 \to 85.00$ & $74.38 \to 85.12$ & $87.25 \to 93.38$ \\
128K & $0.00 \to 2.00$  & $54.50 \to 67.88$ & $31.12 \to 47.75$ & $54.25 \to 65.75$ \\
\bottomrule
\end{tabular}
\end{table}




\bibliographystyle{plainnat}
\bibliography{refs}

\end{document}